\documentclass{article}

\usepackage{microtype}
\usepackage{graphicx}
\usepackage{subfigure}
\usepackage{booktabs} 

\usepackage{hyperref}

\usepackage[accepted]{icml2024}

\usepackage{amsmath}
\usepackage{amssymb}
\usepackage{mathtools}
\usepackage{amsthm}
\usepackage{enumitem}

\def\*#1{\mathbf{#1}}
\newcommand{\R}{\mathbb{R}}
\newcommand{\E}{\mathbb{E}}

\usepackage{xcolor}

\usepackage[capitalize,noabbrev]{cleveref}

\theoremstyle{plain}
\newtheorem{theorem}{Theorem}[section]
\newtheorem*{theorem*}{Theorem}
\newtheorem{proposition}[theorem]{Proposition}
\newtheorem*{proposition*}{Proposition}
\newtheorem{lemma}[]{Lemma}
\newtheorem{corollary}[theorem]{Corollary}
\theoremstyle{definition}
\newtheorem{definition}[theorem]{Definition}

\theoremstyle{remark}
\newtheorem{remark}[theorem]{Remark}

\usepackage[textsize=tiny]{todonotes}

\icmltitlerunning{From Biased Selective Labels to Pseudo-Labels}

\begin{document}

\twocolumn[
\icmltitle{From Biased Selective Labels to Pseudo-Labels: \\An Expectation-Maximization Framework for Learning from Biased Decisions}




\begin{icmlauthorlist}
\icmlauthor{Trenton Chang}{yyy}
\icmlauthor{Jenna Wiens}{yyy}
\end{icmlauthorlist}

\icmlaffiliation{yyy}{Division of Computer Science \& Engineering, University of Michigan, Ann Arbor, MI, USA}

\icmlcorrespondingauthor{Trenton Chang}{ctrenton@umich.edu}

\icmlkeywords{disparate censorship, selective labels, machine learning in healthcare, semi-supervised learning, healthcare}

\vskip 0.3in
]


\printAffiliationsAndNotice{}  

\begin{abstract}
Selective labels occur when label observations are subject to a decision-making process; \emph{e.g.}, diagnoses that depend on the administration of laboratory tests. We study a clinically-inspired selective label problem called disparate censorship, where labeling biases vary across subgroups and unlabeled individuals are imputed as “negative'' (\textit{i.e.}, no diagnostic test = no illness). Machine learning models na{\"i}vely trained on such labels could amplify labeling bias. Inspired by causal models of selective labels, we propose Disparate Censorship Expectation-Maximization (DCEM), an algorithm for learning in the presence of disparate censorship. We theoretically analyze how DCEM mitigates the effects of disparate censorship on model performance. We validate DCEM on synthetic data, showing that it improves bias mitigation (area between ROC curves) without sacrificing discriminative performance (AUC) compared to baselines. We achieve similar results in a sepsis classification task using clinical data.
\end{abstract}

\begin{figure}[t!]
    \centering
    \includegraphics[width=\linewidth]{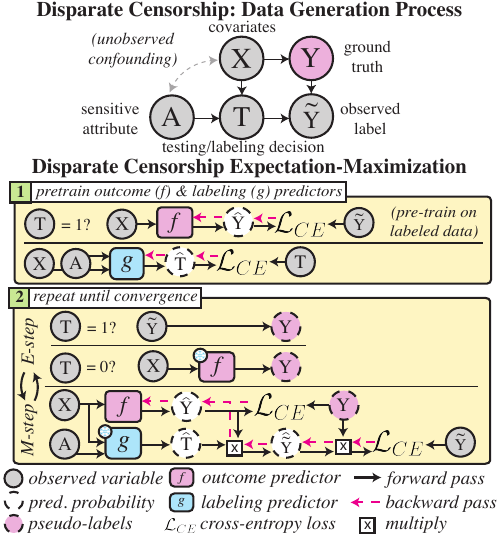}
    \vspace{-7mm}
    \caption{\textbf{Top:} Causal model of disparate censorship ($\*x$: covariates, $y$: ground truth, $\tilde{y}$: observed label, $t$: testing/labeling indicator, $a$: sensitive attribute). Shaded variables are fully observed. \textbf{Bottom:} Disparate Censorship Expectation-Maximization (DCEM). Dashed nodes are probabilistic estimates.}\vspace{-8mm}
    \label{fig:main}
\end{figure}

\section{Introduction}

Selective labels occur when a decision-making process determines access to ground truth~\cite{lakkaraju2017selective}.
We study a practical case of selective labels: disparate censorship~\cite{chang2022disparate}. Disparate censorship introduces two challenges: different labeling biases across subgroups and the assumption that unlabeled individuals have a \textit{negative} label.
For example, in healthcare, labels may depend on laboratory test results only available in some patients. 

Past work has trained ML models to predict outcomes based on laboratory test results (e.g., sepsis~\cite{seymour2016assessment, rhee2020sepsis}). In this setting, patients with no test result are defined as negative~\citep{hartvigsen2018early,  teeple2020clinical,  jehi2020individualizing, mcdonald2021derivation, adams2022prospective, MCURES}. However, laboratory testing decisions may be biased. For example, women are undertested and underdiagnosed for cardiovascular disease~\citep{beery1995gender, schulman1999effect}. ML models trained on such data may recommend women less often for diagnostic testing than men, reinforcing inequity.

To address this bias, one option is to train only on tested individuals. Such an approach may discard a large subset of the data and may not generalize to untested patients. Another option is semi-supervised approaches that do not assume untested patients are negative, such as label propagation~\citep{zhu2002learning, lee2013pseudo} or filtering~\citep{li2020dividemix, nguyen2020self}, or noisy-label learning methods~\citep{blum2020recovering, wang2021gpl, zhu2021clusterability}.
However, such methods do not leverage causal models of label bias, a potential source of additional information. 
We aim to develop an approach that leverages all available signal while accounting for labeling biases.

Inspired by causal models of selective labeling~\citep{laine2020evaluating, chang2022disparate, guerdan2023counterfactual}, we propose a simple method for mitigating bias when training models under disparate censorship: Disparate Censorship Expectation-Maximization (DCEM; Fig.~\ref{fig:main}). 
First, we show that DCEM regularizes model estimates to counterbalance disparate censorship.
We validate DCEM in a simulation study and a sepsis classification task on clinical data. We find that our method mitigates bias (area between ROC curves) while maintaining competitive discriminative performance (AUC), and is generally more robust than baselines to changes in the data generation process.

\section{Preliminaries: Disparate Censorship}
\label{sec:background}

We consider a dataset $\{\*x^{(i)}, \tilde{y}^{(i)}, t^{(i)}, a^{(i)}\}_{i=1}^N$, with covariates $\*x^{(i)}\in\mathbb{R}^d$, labeling/testing decision $t^{(i)}\in\{0,1\}$, sensitive attribute $a^{(i)}$, and observed label $\tilde{y}^{(i)}\in\{0,1\}$, a proxy for ground truth $y^{(i)} \in \{0, 1\}$.  
The proxy label $\tilde{y}^{(i)} = y^{(i)}$ when $t^{(i)}=1$, and $\tilde{y}^{(i)} = 0$ otherwise (\emph{i.e.}, $\tilde{y}^{(i)} = y^{(i)}t^{(i)}$).

\paragraph{What is disparate censorship?} Disparate censorship models ``double standards'' in label collection decisions (Fig.~\ref{fig:main}, top). It is a variation of selective labeling or outcome measurement error~\cite{lakkaraju2017selective, guerdan2023groundless}. 
Disparate censorship uniquely assumes that untested individuals are imputed as \emph{negative.} 

We consider disparate censorship in the context of binary classification~\cite{chang2022disparate} (Fig.~\ref{fig:main}, top).
We justify the model by example. Consider a patient in an emergency room with characteristics $\*x$ and sensitive attribute $a$. 
This patient may have some condition $y$ (currently unobserved) caused by $\*x$ but \emph{not} $a$.
A clinician may order a diagnostic test (set $t$ to 1) to determine $y$.
The decision is based on $\*x$, but could be swayed by biases in $a$.

To simplify, suppose that tests are perfectly sensitive.\footnote{If not, we can define $T$ to indicate whether a label is \textit{confirmed} correct. This definition captures differences in test sensitivity across groups (\emph{i.e.}, spectrum bias~\cite{mulherin2002spectrum}).}
Then, we observe ground truth for tested individuals ($t=1\implies \tilde{y} = y$). Otherwise, the patient's label is imputed as negative ($t = 0 \implies \tilde{y} = 0$; \emph{i.e.}, untested patients are presumed healthy).
However, due to biases in testing decisions $t$, $y$ may only be available in a biased subset of the data.
The causal model of disparate censorship (Fig.~\ref{fig:main}, top) encodes this decision-making pipeline. 
Beyond healthcare, disparate censorship may arise whenever potentially biased decisions affect data labeling. 

\paragraph{Learning under disparate censorship.} 
We aim to learn a mapping $f_\theta: \*x \to y$ parameterized by $\theta$ optimized for discriminative performance (\emph{i.e.}, AUC), but only observe proxy labels $\tilde{y}$. 
The default approach for learning under disparate censorship is to assume $y = \tilde{y}$ and proceed using supervised learning.
However, such an $f_\theta$ may encode labeling biases: estimates of $P(\tilde{Y} \mid X)$ may be inflated compared to $P(Y \mid X)$ for those more likely to be labeled. Thus, biased labeling could yield disproportionate impacts on performance across different subgroups of the data.

Note that we can interpret the estimand of interest as the causal effect of testing on the observed label, since, in the language of do-calculus~\citep{pearl2009causality},
\begin{align}
    \mathbb{E}[Y \mid X] &= \mathbb{E}[Y \mid X, T=1] =\mathbb{E}[\tilde{Y} \mid X, T=1] \nonumber\\
    &= \mathbb{E}[\tilde{Y} \mid X, do(T=1)],\label{eq:ident}
\end{align}
which follows from standard causal identifiability derivations given the causal graph of Fig.~\ref{fig:main}~\cite{imbens2015causal}.
Intuitively, testing an individual ($do(T=1)$) reveals their outcome. 
Thus, a model trained only on tested individuals could consistently estimate $P(Y|X)$, but may not correct for labeling bias.
We discuss other approaches in semi-supervised learning in Section~\ref{sec:related}.

\section{Methodology}
\label{sec:method}
We propose Disparate Censorship Expectation-Maximiza-tion (DCEM) as an approach for learning in the presence of disparate censorship. We first build intuition for how one could mitigate disparate censorship based on the causal model (Section~\ref{subsec:desired}). We then derive DCEM (Section~\ref{subsec:dcem}) and show that it mitigates disparate censorship via a form of regularization (Section~\ref{subsec:balance}). We consider alternative designs and their limitations (Section~\ref{subsec:whynot}).
Detailed proofs and definitions are in Appendix~\ref{app:proofs}.

\subsection{Towards mitigating disparate censorship}
\label{subsec:desired}

Recalling the causal model of disparate censorship, suppose that we are naively training a model $f_\theta$ to predict $\tilde{y}$.
Define groups $a$ and $a'$ and $\*x \sim \mathcal{X}$. Consider some $\mathcal{X}' \subseteq \mathcal{X}$ so that
\begin{equation}
    \underset{\*x \in \mathcal{X}'}{\mathbb{P}}[T \mid X, A=a] << \underset{\*x \in \mathcal{X}'}{\mathbb{P}}[T \mid X, A=a']\label{eq:undertesting}
\end{equation}
for all $\*x \in \mathcal{X}'$.
Define $\hat{t} \triangleq P(T \mid X=\*x, A=a)$ (\emph{e.g.}, probability of receiving a laboratory test) and $ \hat{y} \triangleq P(Y \mid X = \*x)$. 
By assumption, $\mathbf{x}$ is sufficient for predicting $y$ (\emph{i.e.}, as in Fig.~\ref{fig:main}, top), such that the optimal $\hat{y}$ should be similar across $a$ (within $\mathcal{X}'$).
However, Eq.~\ref{eq:undertesting} states that group $a$ is \textit{undertested} relative to group $a'$: they have a lower $\hat{t}$ within $\mathcal{X}'$. Equivalently, labeling bias \emph{favors} group $a'$.
Thus, our naive model would underestimate $\hat{y}$ in group $a$ (lower $\hat{t}$ than group $a'$ in $\mathcal{X}'$) relative to group $a'$.

To counterbalance this bias, one could increase $\hat{y}$ (within $\mathcal{X}'$) where group $a$ is more prevalent than group $a'$; \emph{i.e.}, lower-$\hat{t}$ regions. Since we are interested in discriminative performance, this is analogous to decreasing $\hat{y}$ where $\hat{t}$ is higher, from which the proposed method follows. More broadly, variables associated with labeling bias ($A$ causally affects $T$) but not the outcome of interest ($A$ does not causally affect $Y$) may be useful for mitigating labeling bias. 

Given our causal model with latent variable $Y$ (Fig.~\ref{fig:main}, top), we base our approach on expectation-maximization (EM)~\cite{dempster1977maximum}.  
We can write: 
\begin{align}
     &P(\tilde{Y}, Y, T, X, A, U) \nonumber\\
     &= P(\tilde{Y} \mid Y, T) P(Y \mid X) P(T \mid X, A) P(X, A, U). \label{eq:ll}
\end{align}
Since $y$ is not fully observed, Eq.~\ref{eq:ll} cannot be optimized via standard supervised objectives.
Dropping terms that do not involve $Y$, we can write the maximization of Eq.~\ref{eq:ll} as 
\begin{align}
    \max \; P(\tilde{Y} \mid Y, T)P(Y \mid X).\label{eq:raw_max}
\end{align}
Optimizing Eq.~\ref{eq:raw_max} proceeds via EM. We show that the resulting objectives align with reducing $\hat{y}$ in higher-$\hat{t}$ regions and maintain discriminative performance on tested individuals.

\begin{algorithm}[t]
    \caption{Disparate Censorship Expectation-Maximiza-tion. $\mathcal{L}$: binary cross-entropy loss.}\label{alg:dcem}
    \begin{algorithmic}
\STATE {\bfseries Input:} $\{(\*x^{(i)}, \tilde{y}^{(i)}, t^{(i)}, a^{(i)})\}_{i=1}^N$
\STATE {\bfseries Output:} $f_{\theta}: \mathcal{X} \to [0, 1]$\\
$f_\theta \leftarrow \underset{f_\theta}{\arg\min} \frac{1}{|\{i: t^{(i)} = 1\}|}\sum_{i:t^{(i)} = 1} \mathcal{L}(\tilde{y}^{(i)}, f_\theta(\*x^{(i)}))$\\
$g_{\zeta^*} \leftarrow \underset{g_\zeta}{\arg\min} \frac{1}{N}\sum_{i=1}^N \mathcal{L}(t^{(i)}, g_\zeta(\*x^{(i)}, a^{(i)}))$\\
$\hat{t}^{(i)} \leftarrow g_{\zeta^*}(\*x^{(i)}, a^{(i)})$\\
\WHILE{{\bfseries not converged}}
    \STATE $Q(y^{(i)}) \leftarrow t^{(i)} \tilde{y}^{(i)} + (1 - t^{(i)}) f_{\theta}(\*x^{(i)})$  // E-step\\
    \STATE $f_{\theta} \leftarrow \underset{f_{\theta}}{\arg\min} \frac{1}{N}\sum_{i=1}^N \mathcal{L}(Q(y^{(i)}), f_{\theta}(\*x^{(i)}))$\\
    \STATE $\quad + Q(y^{(i)})  \mathcal{L}(\tilde{y}^{(i)}, f_{\theta}(\*x^{(i)}) \cdot \hat{t}^{(i)})$ // M-step 
\ENDWHILE
\STATE {\bfseries return} $f_{\theta}$ 
\end{algorithmic}
\end{algorithm}\vspace{-2mm}%
\subsection{Disparate Censorship Expectation-Maximization}
\label{subsec:dcem}

\paragraph{Informal overview.} EM alternates an expectation step (E-step), which imputes guesses for the latent variable(s) (\emph{i.e.}, $Y$ in Eq.~\ref{eq:raw_max}), and a maximization step that optimizes likelihood given the imputed estimates (M-step, \emph{i.e.},  Eq.~\ref{eq:raw_max}).
Our E-step imputes preliminary estimates of $P(Y\mid X)$ for untested individuals.
Our M-step updates the estimates to counteract labeling biases when $t^{(i)} = 0$, and is equivalent to full supervision when $t^{(i)} = 1$. 
The E- and M-steps alternate until convergence. Fig.~\ref{fig:main} (bottom) shows a schematic of DCEM, with pseudocode in Algorithm~\ref{alg:dcem}.

\paragraph{E-step.}
The posterior of $y^{(i)}$ given the observed data is:
\begin{equation}
    Q(y^{(i)}) \triangleq \E[y^{(i)} \mid \tilde{y}^{(i)}, t^{(i)}, \*x^{(i)}, a^{(i)}].~\label{eq:raw_estep}
\end{equation}
We can rewrite Eq.~\ref{eq:raw_estep} as follows:

\begin{theorem}[E-step]
    The posterior distribution of $y^{(i)}$ given the observed data is equivalent to
    \begin{align}
    Q(y^{(i)}) &= \begin{cases}
        \tilde{y}^{(i)} & t^{(i)} = 1\\
        \mathbb{E}[y^{(i)} = 1 \mid \*x^{(i)}] & \text{otherwise}
    \end{cases}\label{eq:estep}
\end{align}
\end{theorem}
Intuitively, the E-step uses $\tilde{y}$ as the label when we have complete label information (recall $t^{(i)} = 1 \implies \tilde{y}^{(i)} = y^{(i)}$); otherwise, we use the posterior estimate $f_\theta(\*x^{(i)})$ as a smoothed label. 
Equivalently, the E-step imputes \emph{soft pseudo-labels} for unlabeled data, \emph{i.e.}, probabilistic estimates $\hat{y}^{(i)} \in [0, 1]$. 
Motivated by approaches that train a pseudo-labeling model on labeled data~\citep{arazo2020pseudo, rizve2021defense}, we pre-train $f_\theta$ on tested individuals.

\paragraph{M-step.} The M-step maximizes the log-likelihood of Eq.~\ref{eq:raw_max} given E-step estimates $Q(y^{(i)})$ (Eq.~\ref{eq:estep}). There are two terms to model, which is done via an estimator for $y^{(i)}$ trained using $Q(y^{(i)})$ and an estimator for $\tilde{y}^{(i)}$. The latter is obtained by combining an estimate of $t^{(i)}$ with $Q(y^{(i)})$. 
Concretely, let $\hat{y}^{(i)} \triangleq f_\theta(\*x^{(i)})$, and let $h_\phi$ be a model of $P(\tilde{Y} \mid Y, T)$. 
Maximizing the log-likelihood of Eq.~\ref{eq:raw_max} reduces to
\begin{align}
    &\underset{\theta}{\max}\; \sum_{i=1}^N  Q(y^{(i)})  \log \hat{y}^{(i)} 
    + (1 - Q(y^{(i)})) \log (1 - \hat{y}^{(i)})\nonumber \\
    & \quad+ Q(y^{(i)})\left[\tilde{y}^{(i)}\log h_\phi(\hat{y}^{(i)}, \hat{t}^{(i)})\right. \nonumber \\
    & \quad \left.+ (1 - \tilde{y}^{(i)}) \log(1 - h_\phi(\hat{y}^{(i)}, \hat{t}^{(i)}) )\right].\label{eq:mstep_ll}
\end{align}
This leads to the following result:
\begin{theorem}[M-step, informal]
    Maximizing Eq.~\ref{eq:mstep_ll} also maximizes the evidence-based lower bound of Eq.~\ref{eq:ll}.\vspace{-1mm}
\end{theorem}
In practice, we set $h_\phi(\hat{y}^{(i)}, \hat{t}^{(i)})  \triangleq \hat{y}^{(i)}  \hat{t}^{(i)}$, a smoothed version of the assumption $\tilde{y} = yt$.
Defining $\mathcal{L}$ as binary cross-entropy loss, we can rewrite Eq.~\ref{eq:mstep_ll}:
\begin{equation}
    \underset{\theta}{\min}\;\sum_{i=1}^N \mathcal{L}(Q(y^{(i)}), \hat{y}^{(i)})\!+\!Q(y^{(i)}) \mathcal{L}(\tilde{y}^{(i)}, \hat{y}^{(i)}\hat{t}^{(i)}).\label{eq:mstep}
\end{equation}
Eq.~\ref{eq:mstep} can be interpreted as a regularized cross-entropy loss with respect to pseudo-label $Q(y^{(i)})$. The first term pushes $\hat{y}^{(i)}$ towards $Q(y^{(i)})$, while the second ``encourages'' $\hat{y}^{(i)}$ to be consistent with the causal model. 
To obtain $\hat{t}$, we pre-train and freeze a binary classifier for $t$, and take the probabilistic estimates as $\hat{t}$.

\subsection{DCEM counterbalances disparate censorship}
\label{subsec:balance}

We show that DCEM imposes a form of ``causal regularization'' that lowers $\hat{y}$ in untested individuals. 

\paragraph{DCEM is a form of causal regularization.} 
By analogy to regularized risk minimization, consider an objective
\begin{equation}
    \mathcal{L}(\theta) + \lambda r(\theta),\label{eq:basic_rrm}
\end{equation}
for $\lambda \in \mathbb{R}_+$  (regularization strength) and a regularizer $r:\Theta \to \mathbb{R}$, where $\Theta$ is the parameter space of $\theta$. 

Without loss of generality, setting $\lambda = 1$ and matching terms between Eq.~\ref{eq:basic_rrm} and Eq.~\ref{eq:mstep} yields $r(\theta) = Q(y^{(i)}) \mathcal{L}(\tilde{y}^{(i)}, \hat{y}^{(i)}\hat{t}^{(i)})$.
While $\hat{t}^{(i)}$ affects the optimization of Eq.~\ref{eq:mstep}, it is not a multiplier (\emph{e.g.}, $\lambda$ in Eq.~\ref{eq:basic_rrm}). To interpret the effect of $\hat{t}^{(i)}$, we propose a definition of causal regularization strength based how the optimal $\hat{y}^{(i)}$ changes.\footnote{``Causal regularization'' has been defined in the context of causal discovery~\cite{bahadori2017causal, janzing2019causal}. Our usage is unrelated: we use a causal model to regularize an estimator.} 

\begin{definition}[Causal regularization strength, informal]
Let $\hat{y}^{(i)}_{\text{OPT}}(Q(y^{(i)}), \hat{t}^{(i)})$ be the minimizer of Eq.~\ref{eq:mstep}. For $\mathcal{L}$ finite \& convex on $\hat{y}^{(i)}$ in [0, 1], the causal regularization strength is $R(\cdot) \triangleq |Q({y}^{(i)}) - \hat{y}^{(i)}_{\text{OPT}}(Q(y^{(i)}), \hat{t}^{(i)})|$.\label{def:causalreg}
\end{definition}

Definition~\ref{def:causalreg} quantifies the tradeoff between matching $\hat{y}^{(i)}$ to the E-step estimates and optimizing Eq.~\ref{eq:mstep}.
While $\hat{y}^{(i)}$ is not an optimization parameter, analyzing the optimal $\hat{y}^{(i)}$ can clarify the inductive bias of the M-step. 
We proceed by considering how causal regularization impacts untested vs. tested individuals.
 When $t^{(i)} = 0$, the M-step is 
\begin{equation}
    \underset{\theta}{\min}\; \sum_{i=1}^N \mathcal{L}(Q(y^{(i)}), \hat{y}^{(i)})- Q(y^{(i)})\log(1- \hat{y}^{(i)}\hat{t}^{(i)}).\label{eq:mstep_t0}
\end{equation}
Since $-\log(1- \hat{y}^{(i)}\hat{t}^{(i)})$ increases in $\hat{y}^{(i)}$, the regularization term ``encourages'' $\hat{y}^{(i)}$ to decrease when $\hat{t}^{(i)} > 0$. The regularization term is \textit{constant} if $\hat{t}^{(i)} = 0$, such that the M-step would not change the E-step estimate. This matches the intuition that one cannot learn about $y^{(i)}$ from individuals that are very different from labeled individuals (\emph{i.e.}, when the \emph{overlap} assumption in causal inference is violated). 
The regularization strength depends on $\hat{t}^{(i)}$ as follows:

\begin{theorem}[informal]
If $t^{(i)} = 0$, causal regularization strength increases in $\hat{t}^{(i)}$.~\label{prop:counterbalance}
\end{theorem}
The result implies that \textbf{causal regularization counterbalances disparate censorship.} Recall that lowering $\hat{y}^{(i)}$ in regions where $\hat{t}^{(i)}$ is higher can mitigate bias. Equivalently, causal regularization must strengthen as $\hat{t}^{(i)}$ increases, which follows from Theorem~\ref{prop:counterbalance}. 

\paragraph{Causal regularization aligns with full supervision in tested individuals.} When $t^{(i)} = 1$, the M-step is
\begin{align}
    \underset{\theta}{\min}&\; \sum_{i=1}^N \mathcal{L}(y^{(i)}, \hat{y}^{(i)})+ y^{(i)} \mathcal{L}(y^{(i)}, \hat{y}^{(i)}\hat{t}^{(i)}),~\label{eq:mstep_t1}
\end{align}
substituting $y^{(i)}$ for $Q(y^{(i)})$ and $\tilde{y}^{(i)}$. Thus: 
\begin{proposition}
    Eq.~\ref{eq:mstep_t1} is minimized when $\hat{y}^{(i)} = y^{(i)}$.~\label{prop:mstep_t1}\vspace{-1mm}
\end{proposition}
Proposition~\ref{prop:mstep_t1} states that causal regularization does not change the M-step optimum from matching ground truth when $t^{(i)} = 1$ (\emph{i.e.}, regularization strength = 0). Thus, the M-step objective aligns with fully-supervised loss.

Thus, the M-step (Eq.~\ref{eq:mstep}) counterbalances disparate censorship by regularizing $\hat{y}^{(i)}$ towards 0 as $\hat{t}^{(i)}$ increases. For $t^{(i)} = 1$, the M-step optimum stays constant, and DCEM should maintain discriminative performance.

\subsection{Alternative designs and their limitations}
\label{subsec:whynot}

We consider two alternative designs and their limitations: directly using $t^{(i)}$ in DCEM and propensity score adjustment.

\paragraph{Why not use $t^{(i)}$ directly?} 
We substitute $\hat{t}^{(i)} = t^{(i)}$ into Eq.~\ref{eq:mstep} and analyze one summand (without loss of generality):
\begin{align}
    \mathcal{L}(y^{(i)}, \hat{y}^{(i)})+ y^{(i)}\mathcal{L}(y^{(i)}, \hat{y}^{(i)}) & \quad t^{(i)} = 1 \label{eq:hard_t1_mstep}\\
    \mathcal{L}(Q(y^{(i)}), \hat{y}^{(i)}) & \quad t^{(i)} = 0
    \label{eq:hard_t0_mstep}
\end{align}
Both losses use the E-step estimate $Q(y^{(i)})$ as supervision. When $t^{(i)} = 1$ (Eq.~\ref{eq:hard_t1_mstep}), the M-step adds $y^{(i)}\mathcal{L}(y^{(i)}, \hat{y}^{(i)})$, penalizing false negatives 2x as heavily as false positives. This does not affect ranking metrics (\emph{e.g.}, AUC).
When $t^{(i)} = 0$ (Eq.~\ref{eq:hard_t0_mstep}), the M-step drops causal regularization, and thus cannot counterbalance disparate censorship.
Directly using $t^{(i)}$ would only help if counterbalancing disparate censorship is unnecessary for good estimation, \emph{i.e.}, when tested individuals are representative of the population. 

\paragraph{Why not propensity score adjustment/related causal approaches?}
Recall that estimating the effect of $T$ on the observed label yields a consistent estimate of $P(Y \mid X)$ (Eq.~\ref{eq:ident}, Section~\ref{sec:background}).  
Indeed, $\hat{t}^{(i)}$ is an estimate of $P(T \mid X, A)$, \emph{i.e.}, a \textit{propensity score}, motivating the usage of causal effect estimators that leverage $\hat{t}^{(i)}$.
However, propensity score adjustment (\emph{e.g.}, IPW~\cite{rosenbaum1983central} or doubly-robust variations~\cite{robins1994estimation, van2006targeted, hu2022non}) require an ``overlap'' assumption $\eta < \hat{t}^{(i)}< 1-\eta$ for some $\eta = (0, \frac{1}{2})$ and have asymptotic variance of order $O(1 / (\eta \cdot (1 - \eta))$, which is sensitive to extreme $\hat{t}^{(i)}$ (\emph{e.g.}, as in AIPW~\citep{glynn2010introduction}).

However, in finite-sample settings, ``sharp'' testing decisions lead to weak overlap.
Such extreme $\hat{t}^{(i)}$ may arise in threshold-based decisions~\cite{djulbegovic2014physicians, pierson2018fast}.
For example, a patient either exhibits or does not exhibit the requisite symptoms to warrant testing.
This is analogous to inducing covariate shift between tested and untested individuals. In other words, ``holes'' in the training data emerge when using only labeled examples. Thus, systematic testing bias could exacerbate model performance gaps across population subgroups. 
While low overlap still impacts DCEM (since DCEM cannot learn when $\hat{t}^{(i)} = 0$), our method instead leverages an evidence-based lower bound to model $y$ under disparate censorship. We further discuss potential improvements in overlap-robustness of the proposed approach in Appendix~\ref{app:proofs}.

\section{Experimental Setup}

We validate DCEM with \textbf{synthetic data} across different data-generation processes on simulated binary classification tasks (Section~\ref{subsec:synth}) and in a pseudo-synthetic sepsis classification task using real clinical data \textbf{(MIMIC-III)}~\cite{johnson2016mimic}, across potential laboratory testing policies (Section~\ref{subsec:sepsis}).
We then discuss our chosen baselines (Section~\ref{subsec:models}) and evaluation metrics (Section~\ref{subsec:eval}).

\subsection{Synthetic Datasets}
\label{subsec:synth}
By definition, $y$ is not fully observed under disparate censorship. Thus, we design a simulation study in order to evaluate various methods with respect to ground truth. The data generation process follows from the assumed causal model of disparate censorship (Fig.~\ref{fig:main}, top):
\begin{align*}
&\quad a^{(i)} \sim Ber(0.5),   
\*x^{(i)} \sim \mathcal{N}( \mu_a \cdot \*1_2, 0.03^2 \*I_{2 \times 2})\\
&\quad t^{(i)} \sim Ber(\sigma(30 \cdot s_T(\*x^{(i)}, a^{(i)}))) \\
&\quad  y^{(i)} \sim Ber(\sigma(10 \cdot s_Y(\*x^{(i)}) - c_y)),\; \tilde{y}^{(i)} = y^{(i)}t^{(i)} 
\end{align*}
where $\*I_{2 \times 2}$ is the identity matrix, and $s_T: \*x, a \to \R$, $s_Y: \*x \to \R$, and $\mu_a \in \R, c_y \in \R$ are simulation parameters. We set $P(A = 0) = 0.5$ and induce confounding between $\*x^{(i)}$ and $a^{(i)}$ by setting $u^{(i)} = a^{(i)}$. 
We draw $\*x^{(i)} \in \mathbb{R}^2$ from group-specific Gaussians, and  assume Bernoulli-distributed $t^{(i)}$ and $y^{(i)}$ with parameters defined via $s_T: \*x, a \to \mathbb{R}$ and $s_Y: \*x \to \mathbb{R}$, respectively. Intuitively, $s_T$ ($s_Y$) is a soft ``decision boundary'' for $T$ ($Y$). Inspired by  observations that clinician testing is can be represented by simpler functions than observed outcomes~\cite{mullainathan2022diagnosing}, we choose a non-linear $s_Y$ and a linear $s_T$.

We simulate $N=20,000$ individuals for training, validation, and testing each (\emph{i.e.}, 60,000 total). 
We define settings in terms of testing disparity $q_t =  \frac{P(T \mid A = 0)}{P(T \mid A = 1)}$, prevalence disparity $q_y = \frac{P(Y \mid A = 0)}{P(Y \mid A = 1)}$, and testing multiple $k =\frac{P(T = 1)}{P(Y = 1)}$. Intuitively, $q_t$ controls labeling biases, $q_y$ controls differences between groups, and $k$ controls testing rate. We consider $q_t \in \{1/4, 1/3, 1/2, 1, 2, 3, 4\}$, $q_y \in \{1/4, 1/3, 1/2, 1\}$, and $k \in \{1/4, 1/3, 1/2, 1, 2, 3, 4\}$, 
and set simulation parameters to yield the desired $q_t, q_y, k$.\footnote{We skip settings where $q_t, q_y, k$ yield infeasible testing rates.}

Since $s_Y$ is unknown in practice, we replicate the main experiments across various $s_Y$ as a robustness check.
The simulation makes simplifying assumptions (\emph{e.g.}, low dimensionality and $P(A=0) = 0.5$) but allows full control over $y$ and $t$. Additional simulation details are in Appendix~\ref{app:sim_details}.

\subsection{Clinical data: MIMIC-III}
\label{subsec:sepsis}

Multiple sepsis definitions, such as Sepsis-3~\cite{singer2016third}, are based on laboratory tests (blood culture) such that patients without a test result are \emph{by definition} negative.
Thus, sepsis classification is a potential real-world case of disparate censorship. We curate a sepsis classification task using
 the MIMIC-III Sepsis-3 cohort~\cite{johnson2016mimic, johnson2018comparative}, an electronic health record dataset.

We aim to distinguish patients who never develop sepsis from those who develop sepsis within 8 hours of an initial 3-hour observation period.
If a patient met the Sepsis-3 criteria between 3-11 hours of the first chart measurement, we set $y=1$, and $y=0$ if the patient never develops sepsis during their hospital stay.
We exclude patients with onset times outside this range and include only White and Black patients to simplify the analysis of bias mitigation.
We choose $\*x \in \R^{13}$ following an existing sepsis prediction model~\cite{delahanty2019development}, and exclude patients where all features are missing. This yields $N = 5,301$ patients, from which we create a 60-20-20 train-validation-test split. 
This is a simplified version of a real clinical task, since we exclude patients who develop sepsis later during their hospitalization. Nonetheless, it is helpful for probing the strengths and weaknesses of the proposed approach.

To evaluate model performance, we assume that the observed $y$ reflects ground truth, since $\approx 90\%$ of patients were tested (\emph{i.e.}, received a blood culture) in our cohort.
To generate label proxies $\tilde{y}$, we simulate multiple potential labeling biases via a clinically-inspired testing function $s_T$. We specify a linear $s_T$ based on qSOFA, a score for triaging patients at risk of sepsis~\cite{seymour2016assessment}.
Inspired by observations that clinicians over-weight representative symptoms in diagnostic test decisions~\cite{mullainathan2022diagnosing}, we create different versions of $s_T$ via different weightings of qSOFA features. 
We examine $k \in \{1/4, 1/3, 1/2, 1, 2, 3, 4, 5\}$ and $q_t \in \{1/2, 2/3, 1, 3/2, 2\}$.
Details of the sepsis cohort are in Appendix~\ref{app:sepsis_details}.

\subsection{Models}
\label{subsec:models}

As naive baselines, we test a \textbf{$y$-obs model} (training on $\tilde{y}$) and training on \textbf{group $a$ only}. 
We select similarly-motivated or applicable baselines from related settings:\vspace{-2mm}
\begin{itemize}[leftmargin=*]
\item \textbf{Noisy-label learning}: Group peer loss~\cite{wang2021gpl} (Appendix: Peer loss~\cite{liu2020peer},  truncated $\ell_q$ loss~\cite{zhang2018lq} and generalized Jensen-Shannon loss~\cite{ranzato2021gjs}),\vspace{-1mm}
\item \textbf{Semi-supervised learning}: SELF~\cite{nguyen2020self} (Appendix: DivideMix~\cite{li2020dividemix}),\vspace{-1mm}
    \item \textbf{Causal inference}: \textbf{tested-only} (training on examples where $t = 1$), and DragonNet~\cite{shi2019adapting}, using the treatment effect of the sensitive attribute on testing to correct disparate censorship (\emph{i.e.}, learn a correction for $P(Y\mid X) - P(\tilde{Y} \mid X, A)$),\vspace{-1mm}
    \item \textbf{Positive-unlabeled learning:} Selected-At-Random EM (SAREM)~\cite{bekker2020sarem}.\vspace{-2mm}
\end{itemize}
As an oracle, we compare to training on $y$ (``$y$-model''). 
We use neural networks for all approaches. Training details, such as  hyperparameters, are in Appendix~\ref{app:hyperparam}.

\begin{figure*}[t]
    \centering
    \includegraphics[width=\linewidth]{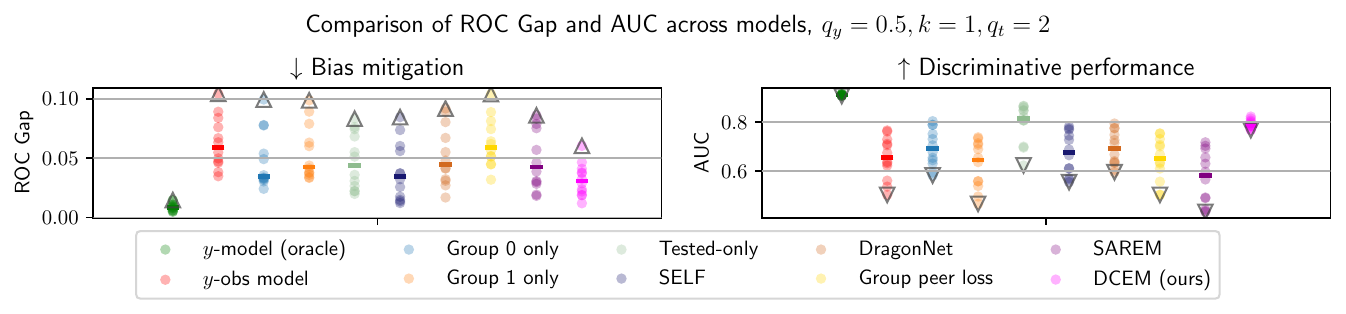}
    \vspace{-9mm}
    \caption{Comparison of ROC gap (left) and AUC (right) of selected models at $q_y = 0.5,  k=1, q_t  = 2$. Each point represents a different $s_Y$. Our method (DCEM, magenta) mitigates bias while maintaining competitive AUC compared to baselines, with a tighter range and improved empirical worst-case for both metrics. ``-'': median, ``$\bigtriangleup$'': worst-case ROC gap, ``$\bigtriangledown$'': worst-case AUC.}\vspace{-3mm}
    \label{fig:main-result}
\end{figure*}

\subsection{Evaluation metrics}
\label{subsec:eval}

We consider bias mitigation and discriminative performance metrics with respect to $y$, and measure the robustness of both metrics to changes in the data-generation process.  

\paragraph{Discriminative performance.} We use the area under the receiver operating characteristic curve (AUC), a standard discriminative performance metric. 

\paragraph{Mitigating bias.} We use the ROC gap (also called ROC fairness~\cite{vogel2021learning} or ABROCA~\cite{gardner2019evaluating}), the absolute area between the ROC curves for each group $a$. 
The ROC gap is in [0, 1]. Lower values indicate better bias mitigation.
Intuitively, the ROC gap is zero when a classifier with some fixed false positive rate in each group obtains equal true positive rates across groups. 
Under disparate censorship, a zero ROC gap is achievable if a model \textit{perfectly} predicts $y$ from $\*x$.

\paragraph{Robustness.} We consider the median AUC and ROC gap over all $s_Y$ (synthetic data setting) or $s_T$ (sepsis classification) and the empirical \textbf{worst-case} (AUC:  min.; ROC gap: max.) and \textbf{range}.

\begin{figure*}[t]
    \centering
    \includegraphics[width=\linewidth]{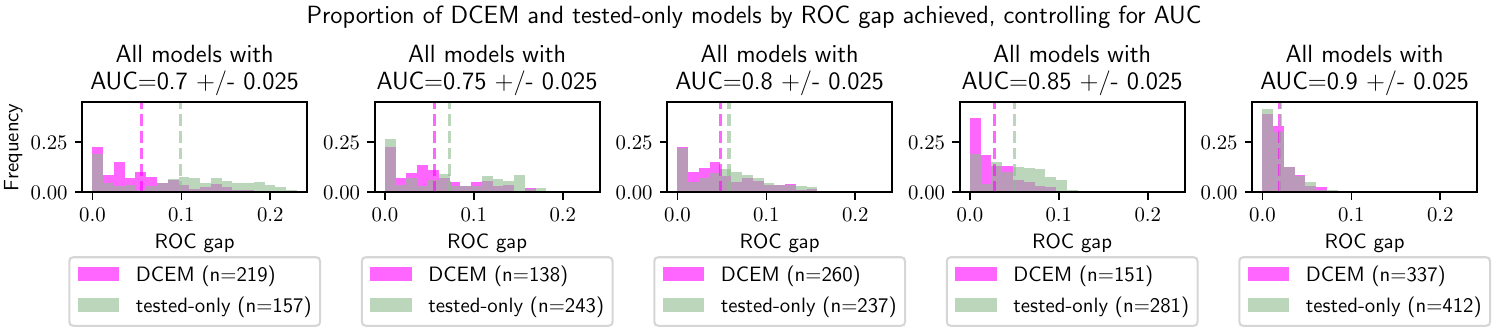}    \vspace{-8mm}
    \caption{Relative frequencies of ROC gaps for DCEM vs. tested-only models at similar AUC (increasing to the right), pooling models across all $k, q_y, q_t$ tested. Dashed lines = mean ROC gap by model. DCEM improves bias mitigation among models with similar AUC.}\vspace{-4mm}
    \label{fig:no_tradeoff}
\end{figure*}

\section{Experiments \& Discussion}

Our experiments aim to substantiate our main claims:
\begin{itemize}[leftmargin=*]
    \item In synthetic data, DCEM mitigates bias, maintains competitive discriminative performance and improves robustness, while achieving better tradeoffs between performance and bias mitigation compared to baselines (Section~\ref{subsec:sim_results}).\vspace{-1mm}
    \item On a sepsis classification task, DCEM improves discriminative performance while maintaining good tradeoffs with bias mitigation, and is more robust compared to baselines (Section~\ref{sec:pseudo}).\vspace{-2mm}
\end{itemize}
We also report full results (Appendix~\ref{app:full_results}) and an ablation study of DCEM (Appendix~\ref{app:ablation}).
We also benchmark  causal effect estimators (\emph{i.e.}, as alternatives to the tested-only model) and their overlap robustness compared to DCEM (Appendix~\ref{app:overlap}). Further sensitivity analyses can be found in Appendix~\ref{app:sensitivity_t} (smoothed $\hat{t}^{(i)}$) and \ref{app:sensitivity_init} (E-step initialization).

\subsection{Results on simulated disparate censorship}
\label{subsec:sim_results}

Fig.~\ref{fig:main-result} shows ROC gaps (left) and AUCs (right) of the baselines most competitive with our approach (DCEM, magenta) at $q_y = 0.5, k=1, q_t=2$. In this setting, 25\% (\emph{i.e.}, $k \cdot P(Y = 1)$) of individuals are tested, and the base rate of $Y$ in group $a=0$ is $1/2$ that of group $a=1$, but group $a=0$ is twice as likely to be tested.
Each point is an ROC gap/AUC value achieved under one decision boundary $s_Y$. Results for the remaining baselines are in Appendix~\ref{app:full_results}. The takeaways align with the main results.

\paragraph{DCEM mitigates bias more effectively than baselines.} 
DCEM achieves a median ROC gap of 0.030 (2nd-best, SELF: 0.034), suggesting that it mitigates bias more effectively than baselines (Fig.~\ref{fig:main-result}, left). 
We show similar trends for $q_t \geq 1$, $q_y$, and $k \in [1/3, 2]$ (Appendix~\ref{app:full_results}). At low testing rates, all models mitigate bias poorly. At high testing rates, the tested-only model is sufficient.

For $q_t < 1$ (Appendix~\ref{app:full_results}), DCEM mitigates bias compared to the default approach ($y$-obs model) but no longer dominates the baselines.
We hypothesize that DCEM has similar bias mitigation capabilities as baselines, since there is less bias to mitigate.
Recalling that $q_y < 1$, since $q_t < 1$, testing probability, $P(Y \mid X)$ \emph{and} $P(\tilde{Y} \mid X)$ are correlated. Learning to predict $\tilde{Y}$ would preserve ordering in $P(Y \mid X)$, reducing impacts on ranking metrics (\emph{e.g.}, ROC gap).\footnote{Such settings are related to boundary-consistent noise; see Proposition 1 of~\cite{menon2018learning}.} 

\begin{figure*}[t]
    \centering
    \includegraphics[width=\linewidth]{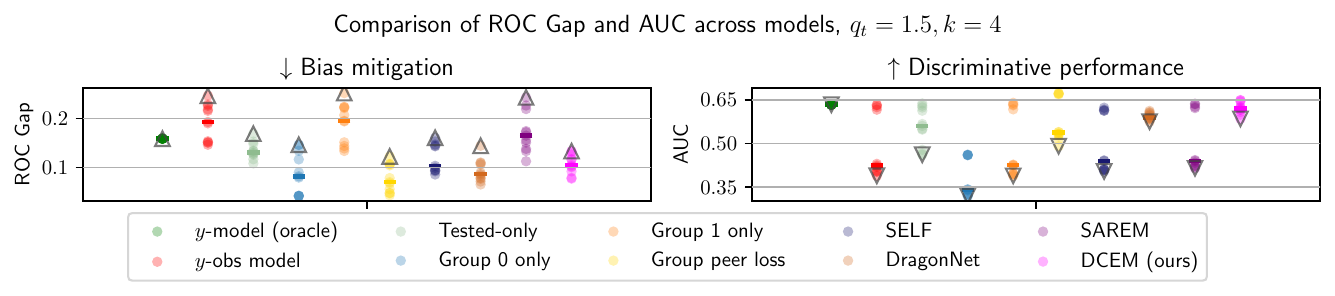}   \vspace{-8mm}
    \caption{ROC gaps (left) and AUC (right) of baselines and DCEM on sepsis classification task at $q_t=1.5, k=4$. Each dot represents a different $s_T$. Our method (DCEM, magenta) maintains competitive or better bias mitigation and discriminative performance compared to baselines. ``-'': median, ``$\bigtriangleup$'': worst-case ROC gap, ``$\bigtriangledown$'': worst-case AUC.}\vspace{-3mm}
    \label{fig:sepsis}
\end{figure*}

\paragraph{DCEM is more robust than baselines to changes in the data-generating process.} 
Fig.~\ref{fig:main-result} (left) shows that the maximum ROC gap is lower for DCEM compared to baselines (ours: 0.060 vs. 2nd-best, tested-only: 0.083). We report similar results for the minimum AUC (Fig.~\ref{fig:main-result}, right; ours: 0.768 vs. 2nd-best, tested-only: 0.623). 
DCEM also achieves a tighter ROC gap and AUC range.
Fig.~\ref{fig:main-result} also shows that our method has the tightest ROC gap range (left, DCEM: 0.048 vs. tested-only, 2nd-tightest: 0.063) and AUC range (right, DCEM: 0.055 vs. DragonNet: 0.199).

The results suggest that DCEM maintains robust bias mitigation and discriminative performance across different data-generation processes ($s_Y$).
This is expected, as DCEM optimizes likelihood under the disparate censorship data-generation process by design.
In contrast, the baselines may experience selection bias or misspecification, since they discard data or assume certain noise structures/variable dependencies that disparate censorship violates.

\paragraph{DCEM maintains competitive discriminative performance.} 
Fig.~\ref{fig:main-result} (right) shows that DCEM outperforms all baselines except for the tested-only model, which our approach lags by 0.028 AUC (DCEM: 0.787 vs. tested-only: 0.815). Other causal estimators achieve similar discriminative performance to the tested-only approach (Appendix~\ref{app:overlap}). However, our method improves on the ``default'' $y$-obs model, increasing the median AUC by 0.130 ($y$-obs: 0.657). SELF, which has a similar median ROC gap to DCEM, underperforms DCEM by 0.110 AUC (SELF: 0.677 vs. DCEM: 0.787).
Other baselines also underperform. 
This is expected, since some methods ignore label bias: training on $\tilde{Y}$ alone is misspecified for $E(Y \mid X)$, since it incorrectly assumes that if $T=0$, then $Y=0$. The same argument applies to Group 0/1 only approaches.

Some baselines account for label noise/bias, but are  mis-specified under disparate censorship since they make different independence assumptions. Group peer loss assumes $T \perp\!\!\!\perp X \mid (Y, A)$, and SELF assumes $T \perp\!\!\!\perp X \mid Y$, ignoring the dependence of biased selective labeling on $X$. DragonNet accounts for $X$ by adding $P(T \mid X, A = 1) - P(T \mid X, A = 0)$ to the default model’s estimates (i.e., $P(\tilde{Y} \mid X)$) as a ``correction factor.'' However, the correction factor may be biased for true negatives: the oracle is zero, because $P(\tilde{Y} \mid X) = P(Y \mid X)$, but $P(T \mid X, A = 1) - P(T \mid X, A = 0) \neq 0$ in general under systematic labeling bias.

SAR-EM, which is most similar to the proposed approach, models missingness at random (i.e., $T \not\perp\!\!\!\perp X, Y$), but discards reliable negatives. In contrast, the proposed approach incorporates reliable negatives in alignment with our assumptions about labeling bias, allowing it to counterbalance biased selective labeling.
Trends are similar for other $q_t, q_y$ and $k \in [1/2, 2]$.
Since the tested-only model is a strong baseline, we now compare it directly to DCEM.

\paragraph{DCEM achieves better tradeoffs between discriminative performance and bias mitigation.} Among models with similar AUC where $\text{AUC} < 0.875$, DCEM reduces ROC gaps compared to the tested-only model (Fig.~\ref{fig:no_tradeoff}).
For example, for $\text{AUC} \in (0.825, 0.875)$ (Fig.~\ref{fig:no_tradeoff}, 2nd from right), DCEM improves the average ROC gap by 0.022 (0.028 vs. 0.050), with similar trends at lower AUCs.
Among the best-performing models ($\text{AUC} > 0.875$; Fig.~\ref{fig:no_tradeoff}, 1st from right), both methods have similar ROC gaps. 

The results suggest that DCEM is not trading discriminative performance for bias mitigation. 
At a given AUC, DCEM more often yields models with a lower ROC gap  than the tested-only model.
Since the tested-only approach does not account for label bias, it can achieve relatively high AUC without mitigating bias. 
In contrast, DCEM explicitly counteracts disparate censorship.
A similar comparison to SELF shows that, at low ROC gaps, DCEM likewise finds higher-AUC solutions than SELF (Appendix~\ref{app:self_tradeoff}).

\subsection{Results on sepsis classification in MIMIC-III}
\label{sec:pseudo}

\paragraph{DCEM has better discriminative performance than baselines.}
Fig.~\ref{fig:sepsis} compares the ROC gap and AUC of DCEM to selected baselines at testing disparity $q_t = 1.5$, and testing rate multiplier $k=4$. Each dot corresponds to one variation of $s_T$ (laboratory testing policy).
Our method has the highest median AUC among baselines (ours: 0.620 vs. DragonNet: 0.593), nearing the oracle ($y$-model, 0.633).
Note that DCEM has better discriminative performance than the tested-only approach, suggesting that extrapolation from tested to untested individuals is more difficult on the sepsis classification task than the fully synthetic tasks.

\paragraph{DCEM achieves good tradeoffs with bias mitigation.} 
DCEM achieves a smaller median ROC gap compared to five of eight baselines tested. 
Group peer loss, DragonNet and the Group 0 approach achieve lower median ROC gaps of 0.070, 0.088 and 0.082, respectively (DCEM: 0.105). 
However, the Group 0 approach catastrophically fails (median AUC: 0.342). 
Models may perform arbitrarily poorly under disparate censorship if labeling biases sufficiently ``conceal'' the true decision boundary. 
Group peer loss (among many other baselines) also exhibits a much wider AUC range than the proposed approach (Group peer loss: 0.182 vs. DCEM: 0.065), suggesting that its discriminative utility may be limited. 
DragonNet appears competitive (0.027 AUC lower than DCEM), but 
would only perform well when the effect of race on testing is close to  $|P(Y\mid X) - P(\tilde{Y} \mid X, A)|$, which is violated if labeling biases (large effect of race on testing) are present in negative patients ($P(Y \mid X) \approx P(\tilde{Y} \mid X, A)$). 

Many approaches, including DCEM, obtain a lower ROC gap than training on $y$. Although the oracle obtains the highest median AUC, optimizing discriminative performance on $y$ is not always guaranteed to mitigate bias.
DCEM uses labeling probabilities to mitigate bias via causal regularization, while DragonNet directly uses an estimate of the labeling bias as a correction factor.
Thus, the results validate that the labeling bias can provide signal for bias mitigation.

\paragraph{DCEM is more robust than most baselines to changes in $s_T$.}
DCEM maintains robust bias mitigation capabilities across $s_T$; \emph{i.e.}, differences in how labelers weigh features in their decisions. Fig.~\ref{fig:sepsis} shows that DCEM attains a maximum ROC gap of 0.133 (left; DragonNet: 0.144), and a minimum AUC of 0.584 (right; DragonNet: 0.574).
Fig.~\ref{fig:sepsis} also shows that DCEM achieves the tightest ROC gap range (left; DCEM: 0.094 vs. 2nd-best: 0.102) and 2nd-tightest AUC range (right; DCEM: 0.065 vs. DragonNet: 0.018). 
Many baselines also exhibit a \textit{bimodal} empirical AUC
distribution and only perform well under specific labeling behaviors. 
We examine the sensitivity of baselines to $s_T$ by plotting AUC and ROC gaps against coefficients of $s_T$ (Appendix~\ref{app:sepsis_robust}). 

While DragonNet is competitive on this dataset, its robustness and performance capabilities may not generalize (\emph{e.g.}, simulation results, Fig.~\ref{fig:main-result}). 
DCEM is the only approach tested that achieved competitive discriminative performance and bias mitigation on both datasets. Trends in performance and robustness are similar for other $k, q_t$ (Appendix~\ref{app:full_results}).

\paragraph{Overall takeaways.} 
In a simulation study of disparate censorship, DCEM mitigates bias while achieving similar or better discriminative performance compared to baselines. The proposed approach is empirically more robust than baselines to changes in the data-generating process.
On a sepsis classification task, DCEM mitigates bias while improving or maintaining discriminative performance compared to baselines across different labeling behaviors.
Thus, DCEM can potentially mitigate bias with less impact on discriminative performance than existing methods. 

\section{Related Work} 
\label{sec:related}

\paragraph{Selective labeling/disparate censorship.} Disparate censorship is a variation of selective labeling~\cite{lakkaraju2017selective, kleinberg2018human} and outcome measurement error~\cite{guerdan2023groundless}. 
Selective labeling problems have been studied in clinical settings~\cite{farahani2020explanatory, shanmugam2021quantifying, chang2022disparate, mullainathan2022diagnosing, balachandar2023domain}, social/public policy~\cite{saxena2020childwelfare, kontokosta2021bias, laufer2022end, liu2022equity, kiani2023counterfactual}, and finance~\cite{bjorkegren2020behavior, henderson2023integrating}, among other domains. For an extended literature review of selective labeling problems, see Appendix~\ref{app:lit_review}.

Past work has trained ML models under disparate censorship, directly encoding untested individuals as negative~\cite{ henry2015targeted, jehi2020individualizing, mcdonald2021derivation, adams2022prospective, MCURES}.
Previous approaches for learning under selective labels leverage heterogeneity in human decisions to recover outcome estimates~\cite{lakkaraju2017selective, kleinberg2018human, chen2023learning}, or use domain-specific adjustments~\cite{gholami2018adversary, balachandar2023domain}. 
We propose DCEM, a complementary approach for mitigating bias under disparate censorship without such restrictions.\vspace{-1mm}

\paragraph{Semi-supervised learning.} Semi-supervised approaches do not assume labels for untested individuals. However, many causally-motivated methods diverge from the causal model of disparate censorship~\cite{madras2019fairness, yao2021causalnl, garg2023instancegm, guerdan2023counterfactual, gong2021instancedep, kato2023automatic, sportisse2023labels} via different independence/causal relationships between variables. 
Filtering methods~\cite{ han2018coteaching,  li2020dividemix, nguyen2020self, Chen2020BeyondCA, prog_noise_iclr2021, zhao2022tscsiidn} assume specific model behavior on noisy examples (\emph{e.g.,}, noise is learned late in training~\cite{arpit2017memorization}) or labeling bias (randomness/class-dependence), which disparate censorship violates.
We also highlight historical expectation-maximization approaches for learning with missing data~\citep{ghahramani1993supervised, ghahramani1996algorithm, ambroise2000algorithm}, which place parametric assumptions on the data-generation process. We use neural networks to target the estimands of interest to circumvent parametric assumptions.

Other alternatives include positive-unlabeled learning approaches that assume labeling depends on covariates (\emph{e.g.}, missing not at random)~\cite{bekker2020sarem, furmanczyk2022joint, gerych2022recovering, wang2024pue}. However, these methods do not leverage correctly-labeled negatives, and naively-incorporating negative examples without causal assumptions may potentially harm model performance or bias mitigation. 
Other methods for noisy-label learning make assumptions incompatible with our setting, \emph{e.g.} uniform noise within subgroups~\cite{wang2021gpl}, almost-surely clean \& noisy examples~\cite{liu2015classification, patrini2016making, tjandra2023alignment}, different variable independence/directionality relationships~\cite{wu2022fairidn}, 
that noisy (\emph{i.e.}, out of distribution) examples are rare~\cite{wald2023birds}, 
or other noise constraints~\cite{li2021provably, zhu2021clusterability}.
Our approach complements existing work by jointly modeling selective and biased labeling via causal assumptions tailored to a biased decision-making pipeline. 

\section{Conclusion}
When biased human decisions affect observations of ground truth, applying standard supervised learning techniques to data exhibiting disparate censorship can amplify the harm of ML models to marginalized groups.
We propose Disparate Censorship Expectation-Maximization (DCEM), a novel approach to classification, to mitigate such harm. In a simulation study and a sepsis classification task, DCEM mitigates bias and maintains competitive discriminative performance compared to baselines.
Limitations of DCEM include potential slow convergence, since EM is iterative.
Model evaluation under disparate censorship is also inherently difficult due to the difficulty of obtaining ground truth, motivating future work in dataset curation.
Furthermore, DCEM does not learn a full generative model with all variables. While such a model could target a wider range of estimands, it would also increase the number of terms that need to be modeled. 
Ultimately, DCEM is a step towards mitigating the disproportionate impacts of disparate censorship. Our work aims to raise awareness of disparate censorship and motivate the study of bias mitigation methods.

\section*{Impact Statement}
This paper addresses disparate censorship, a realistic source of  label bias in ML, and proposes a method that mitigates its harms. Since the goal of the paper is aligned with reducing inequity in decision-making, practical use cases of DCEM are inherently high-stakes settings.
Thus, we believe that the ethical usage of DCEM (or any bias mitigation approach) in the real-world requires prospective model evaluation in the context of use (\emph{e.g.}, shadowing human decision-makers) to assess unforeseen negative impacts. Our work provides a general choice of bias mitigation (area between ROC curves) and discriminative performance metrics (AUC), which are motivated by clinical tasks where equitably ranking individuals in terms of resource needs is important.
Practitioners should ensure their evaluation metrics align with domain-specific criteria for bias mitigation and performance.

\section*{Acknowledgements}
We thank (in alphabetical order) Donna Tjandra, Divya Shanmugan, Fahad Kamran, Jung Min Lee, Maggie Makar, Meera Krishnamoorthy, Michael Ito, Sarah Jabbour, Shengpu Tang, Stephanie Shepard, and Winston Chen for helpful conversations and proofreading, and the anonymous reviewers for their constructive feedback. 
T.C. and J.W. are supported by the U.S. National Heart, Lung, and Blood Insitute of the National Institutes of Health (Grant No. 5R01HL158626-03).
The views and conclusions in this document are those of the authors and should not be interpreted as necessarily representing the official policies, either expressed or implied
of the U.S. National Institutes of Health.
\bibliography{main}

\begin{thebibliography}{103}
\providecommand{\natexlab}[1]{#1}
\providecommand{\url}[1]{\texttt{#1}}
\expandafter\ifx\csname urlstyle\endcsname\relax
  \providecommand{\doi}[1]{doi: #1}\else
  \providecommand{\doi}{doi: \begingroup \urlstyle{rm}\Url}\fi

\bibitem[Adams et~al.(2022)Adams, Henry, Sridharan, Soleimani, Zhan, Rawat,
  Johnson, Hager, Cosgrove, Markowski, et~al.]{adams2022prospective}
Adams, R., Henry, K.~E., Sridharan, A., Soleimani, H., Zhan, A., Rawat, N.,
  Johnson, L., Hager, D.~N., Cosgrove, S.~E., Markowski, A., et~al.
\newblock Prospective, multi-site study of patient outcomes after
  implementation of the trews machine learning-based early warning system for
  sepsis.
\newblock \emph{Nature Medicine}, pp.\  1--6, 2022.

\bibitem[Ambroise \& Govaert(2000)Ambroise and Govaert]{ambroise2000algorithm}
Ambroise, C. and Govaert, G.
\newblock {EM algorithm for partially known labels}.
\newblock In \emph{{Data Analysis, Classification, and Related Methods}}, pp.\
  161--166. Springer, 2000.

\bibitem[Arazo et~al.(2020)Arazo, Ortego, Albert, O’Connor, and
  McGuinness]{arazo2020pseudo}
Arazo, E., Ortego, D., Albert, P., O’Connor, N.~E., and McGuinness, K.
\newblock Pseudo-labeling and confirmation bias in deep semi-supervised
  learning.
\newblock In \emph{2020 International Joint Conference on Neural Networks
  (IJCNN)}, pp.\  1--8. IEEE, 2020.

\bibitem[Arpit et~al.(2017)Arpit, Jastrzebski, Ballas, Krueger, Bengio, Kanwal,
  Maharaj, Fischer, Courville, Bengio, and
  Lacoste-Julien]{arpit2017memorization}
Arpit, D., Jastrzebski, S., Ballas, N., Krueger, D., Bengio, E., Kanwal, M.~S.,
  Maharaj, T., Fischer, A., Courville, A., Bengio, Y., and Lacoste-Julien, S.
\newblock A closer look at memorization in deep networks.
\newblock In \emph{Proceedings of the 34th International Conference on Machine
  Learning}, volume~70 of \emph{Proceedings of Machine Learning Research}, pp.\
   233--242, 2017.

\bibitem[Bahadori et~al.(2017)Bahadori, Chalupka, Choi, Chen, Stewart, and
  Sun]{bahadori2017causal}
Bahadori, M.~T., Chalupka, K., Choi, E., Chen, R., Stewart, W.~F., and Sun, J.
\newblock Causal regularization.
\newblock \emph{arXiv preprint arXiv:1702.02604}, 2017.

\bibitem[Balachandar et~al.(2024)Balachandar, Garg, and
  Pierson]{balachandar2023domain}
Balachandar, S., Garg, N., and Pierson, E.
\newblock Domain constraints improve risk prediction when outcome data is
  missing.
\newblock In \emph{12th International Conference on Learning Representations},
  2024.

\bibitem[Beery(1995)]{beery1995gender}
Beery, T.~A.
\newblock Gender bias in the diagnosis and treatment of coronary artery
  disease.
\newblock \emph{Heart \& Lung}, 24\penalty0 (6):\penalty0 427--435, 1995.

\bibitem[Bekker et~al.(2020)Bekker, Robberechts, and Davis]{bekker2020sarem}
Bekker, J., Robberechts, P., and Davis, J.
\newblock Beyond the selected completely at random assumption for learning from
  positive and unlabeled data.
\newblock In \emph{Machine Learning and Knowledge Discovery in Databases}, pp.\
   71--85, Cham, 2020.

\bibitem[Bergman et~al.(2021)Bergman, Kopko, and Rodriguez]{bergman2021using}
Bergman, P., Kopko, E., and Rodriguez, J.~E.
\newblock A seven-college experiment using algorithms to track students:
  Impacts and implications for equity and fairness.
\newblock Technical report, National Bureau of Economic Research, 2021.

\bibitem[Binns et~al.(2017)Binns, Veale, Van~Kleek, and
  Shadbolt]{binns2017like}
Binns, R., Veale, M., Van~Kleek, M., and Shadbolt, N.
\newblock {Like trainer, like bot? Inheritance of bias in algorithmic content
  moderation}.
\newblock In \emph{Social Informatics: 9th International Conference, SocInfo
  2017, Oxford, UK, September 13-15, 2017, Proceedings, Part II 9}, pp.\
  405--415, 2017.

\bibitem[Bj{\"o}rkegren \& Grissen(2020)Bj{\"o}rkegren and
  Grissen]{bjorkegren2020behavior}
Bj{\"o}rkegren, D. and Grissen, D.
\newblock Behavior revealed in mobile phone usage predicts credit repayment.
\newblock \emph{The World Bank Economic Review}, 34\penalty0 (3):\penalty0
  618--634, 2020.

\bibitem[Blum \& Stangl(2020)Blum and Stangl]{blum2020recovering}
Blum, A. and Stangl, K.
\newblock Recovering from biased data: Can fairness constraints improve
  accuracy?
\newblock In \emph{1st Symposium on Foundations of Responsible Computing},
  2020.

\bibitem[Chang et~al.(2022)Chang, Sjoding, and Wiens]{chang2022disparate}
Chang, T., Sjoding, M.~W., and Wiens, J.
\newblock Disparate censorship \& undertesting: A source of label bias in
  clinical machine learning.
\newblock In \emph{Proceedings of the 7th Machine Learning for Healthcare
  Conference}, volume 182 of \emph{Proceedings of Machine Learning Research},
  pp.\  343--390, Aug 2022.

\bibitem[Chen et~al.(2023)Chen, Li, and Mao]{chen2023learning}
Chen, J., Li, Z., and Mao, X.
\newblock Learning under selective labels with heterogeneous decision-makers:
  An instrumental variable approach.
\newblock \emph{arXiv preprint arXiv:2306.07566}, 2023.

\bibitem[Chen et~al.(2020)Chen, Ye, Chen, Zhao, and Heng]{Chen2020BeyondCA}
Chen, P., Ye, J., Chen, G., Zhao, J., and Heng, P.-A.
\newblock Beyond class-conditional assumption: A primary attempt to combat
  instance-dependent label noise.
\newblock In \emph{Proceedings of the AAAI Conference on Artificial
  Intelligence}, volume~34, 2020.

\bibitem[Delahanty et~al.(2019)Delahanty, Alvarez, Flynn, Sherwin, and
  Jones]{delahanty2019development}
Delahanty, R.~J., Alvarez, J., Flynn, L.~M., Sherwin, R.~L., and Jones, S.~S.
\newblock Development and evaluation of a machine learning model for the early
  identification of patients at risk for sepsis.
\newblock \emph{Annals of Emergency Medicine}, 73\penalty0 (4):\penalty0
  334--344, 2019.

\bibitem[Dempster et~al.(1977)Dempster, Laird, and Rubin]{dempster1977maximum}
Dempster, A.~P., Laird, N.~M., and Rubin, D.~B.
\newblock {Maximum likelihood from incomplete data via the EM algorithm}.
\newblock \emph{Journal of the Royal Statistical Society: Series B
  (methodological)}, 39\penalty0 (1):\penalty0 1--22, 1977.

\bibitem[Djulbegovic et~al.(2014)Djulbegovic, Elqayam, Reljic, Hozo,
  Miladinovic, Tsalatsanis, Kumar, Beckstead, Taylor, and
  Cannon-Bowers]{djulbegovic2014physicians}
Djulbegovic, B., Elqayam, S., Reljic, T., Hozo, I., Miladinovic, B.,
  Tsalatsanis, A., Kumar, A., Beckstead, J., Taylor, S., and Cannon-Bowers, J.
\newblock How do physicians decide to treat: an empirical evaluation of the
  threshold model.
\newblock \emph{BMC Medical Informatics and Decision Making}, 14:\penalty0
  1--10, 2014.

\bibitem[Englesson \& Azizpour(2021)Englesson and Azizpour]{ranzato2021gjs}
Englesson, E. and Azizpour, H.
\newblock Generalized jensen-shannon divergence loss for learning with noisy
  labels.
\newblock In Ranzato, M., Beygelzimer, A., Dauphin, Y., Liang, P., and Vaughan,
  J.~W. (eds.), \emph{Advances in Neural Information Processing Systems},
  volume~34, pp.\  30284--30297, 2021.

\bibitem[Farahani et~al.(2020)Farahani, Sundaram, Enayati, Arunachalam,
  Pasupathy, and Arruda-Olson]{farahani2020explanatory}
Farahani, N.~Z., Sundaram, D. S.~B., Enayati, M., Arunachalam, S.~P.,
  Pasupathy, K., and Arruda-Olson, A.~M.
\newblock {Explanatory analysis of a machine learning model to identify
  hypertrophic cardiomyopathy patients from EHR using diagnostic codes}.
\newblock In \emph{2020 IEEE International Conference on Bioinformatics and
  Biomedicine (BIBM)}, pp.\  1932--1937, 2020.

\bibitem[Furma{\'n}czyk et~al.(2022)Furma{\'n}czyk, Mielniczuk, Rejchel, and
  Teisseyre]{furmanczyk2022joint}
Furma{\'n}czyk, K., Mielniczuk, J., Rejchel, W., and Teisseyre, P.
\newblock Joint estimation of posterior probability and propensity score
  function for positive and unlabelled data.
\newblock \emph{arXiv preprint arXiv:2209.07787}, 2022.

\bibitem[Gardner et~al.(2019)Gardner, Brooks, and Baker]{gardner2019evaluating}
Gardner, J., Brooks, C., and Baker, R.
\newblock Evaluating the fairness of predictive student models through slicing
  analysis.
\newblock In \emph{Proceedings of the 9th International Conference on Learning
  Analytics \& Knowledge}, pp.\  225--234, 2019.

\bibitem[Garg et~al.(2023)Garg, Nguyen, Felix, Do, and
  Carneiro]{garg2023instancegm}
Garg, A., Nguyen, C., Felix, R., Do, T.-T., and Carneiro, G.
\newblock Instance-dependent noisy label learning via graphical modelling.
\newblock In \emph{Proceedings of the IEEE/CVF Winter Conference on
  Applications of Computer Vision (WACV)}, pp.\  2288--2298, 2023.

\bibitem[Gerych et~al.(2022)Gerych, Hartvigsen, Buquicchio, Agu, and
  Rundensteiner]{gerych2022recovering}
Gerych, W., Hartvigsen, T., Buquicchio, L., Agu, E., and Rundensteiner, E.
\newblock Recovering the propensity score from biased positive unlabeled data.
\newblock In \emph{Proceedings of the AAAI Conference on Artificial
  Intelligence}, volume~36, pp.\  6694--6702, 2022.

\bibitem[Ghahramani \& Jordan(1993)Ghahramani and
  Jordan]{ghahramani1993supervised}
Ghahramani, Z. and Jordan, M.
\newblock {Supervised learning from incomplete data via an EM approach}.
\newblock \emph{Advances in Neural Information Processing Systems}, 6, 1993.

\bibitem[Ghahramani et~al.(1996)Ghahramani, Hinton,
  et~al.]{ghahramani1996algorithm}
Ghahramani, Z., Hinton, G.~E., et~al.
\newblock {The EM algorithm for mixtures of factor analyzers}.
\newblock Technical report, Technical Report CRG-TR-96-1, University of
  Toronto, 1996.

\bibitem[Gholami et~al.(2018)Gholami, Mc~Carthy, Dilkina, Plumptre, Tambe,
  Driciru, Wanyama, Rwetsiba, Nsubaga, Mabonga, et~al.]{gholami2018adversary}
Gholami, S., Mc~Carthy, S., Dilkina, B., Plumptre, A., Tambe, M., Driciru, M.,
  Wanyama, F., Rwetsiba, A., Nsubaga, M., Mabonga, J., et~al.
\newblock Adversary models account for imperfect crime data: Forecasting and
  planning against real-world poachers.
\newblock In \emph{International Conference on Autonomous Agents and Multiagent
  Systems}, 2018.

\bibitem[Glynn \& Quinn(2010)Glynn and Quinn]{glynn2010introduction}
Glynn, A.~N. and Quinn, K.~M.
\newblock An introduction to the augmented inverse propensity weighted
  estimator.
\newblock \emph{Political Analysis}, 18\penalty0 (1):\penalty0 36--56, 2010.

\bibitem[Gong et~al.(2021)Gong, Wang, Liu, Han, You, Yang, and
  Tao]{gong2021instancedep}
Gong, C., Wang, Q., Liu, T., Han, B., You, J.~J., Yang, J., and Tao, D.
\newblock Instance-dependent positive and unlabeled learning with labeling bias
  estimation.
\newblock \emph{IEEE Transactions on Pattern Analysis and Machine
  Intelligence}, 44:\penalty0 4163--4177, 2021.

\bibitem[Guerdan et~al.(2023{\natexlab{a}})Guerdan, Coston, Holstein, and
  Wu]{guerdan2023counterfactual}
Guerdan, L., Coston, A., Holstein, K., and Wu, Z.~S.
\newblock Counterfactual prediction under outcome measurement error.
\newblock In \emph{Proceedings of the 2023 ACM Conference on Fairness,
  Accountability, and Transparency}, pp.\  1584–1598, 2023{\natexlab{a}}.

\bibitem[Guerdan et~al.(2023{\natexlab{b}})Guerdan, Coston, Wu, and
  Holstein]{guerdan2023groundless}
Guerdan, L., Coston, A., Wu, Z.~S., and Holstein, K.
\newblock Ground(less) truth: A causal framework for proxy labels in
  human-algorithm decision-making.
\newblock In \emph{Proceedings of the 2023 ACM Conference on Fairness,
  Accountability, and Transparency}, pp.\  688–704, 2023{\natexlab{b}}.

\bibitem[Han et~al.(2018)Han, Yao, Yu, Niu, Xu, Hu, Tsang, and
  Sugiyama]{han2018coteaching}
Han, B., Yao, Q., Yu, X., Niu, G., Xu, M., Hu, W., Tsang, I.~W., and Sugiyama,
  M.
\newblock Co-teaching: Robust training of deep neural networks with extremely
  noisy labels.
\newblock In \emph{Proceedings of the 32nd International Conference on Neural
  Information Processing Systems}, pp.\  8536–8546, 2018.

\bibitem[Harris et~al.(2020)Harris, Millman, van~der Walt, Gommers, Virtanen,
  Cournapeau, Wieser, Taylor, Berg, Smith, Kern, Picus, Hoyer, van Kerkwijk,
  Brett, Haldane, del R{\'{i}}o, Wiebe, Peterson, G{\'{e}}rard-Marchant,
  Sheppard, Reddy, Weckesser, Abbasi, Gohlke, and Oliphant]{harris2020array}
Harris, C.~R., Millman, K.~J., van~der Walt, S.~J., Gommers, R., Virtanen, P.,
  Cournapeau, D., Wieser, E., Taylor, J., Berg, S., Smith, N.~J., Kern, R.,
  Picus, M., Hoyer, S., van Kerkwijk, M.~H., Brett, M., Haldane, A., del
  R{\'{i}}o, J.~F., Wiebe, M., Peterson, P., G{\'{e}}rard-Marchant, P.,
  Sheppard, K., Reddy, T., Weckesser, W., Abbasi, H., Gohlke, C., and Oliphant,
  T.~E.
\newblock Array programming with {NumPy}.
\newblock \emph{Nature}, 585\penalty0 (7825):\penalty0 357--362, September
  2020.

\bibitem[Hartvigsen et~al.(2018)Hartvigsen, Sen, Brownell, Teeple, Kong, and
  Rundensteiner]{hartvigsen2018early}
Hartvigsen, T., Sen, C., Brownell, S., Teeple, E., Kong, X., and Rundensteiner,
  E.~A.
\newblock {Early Prediction of MRSA Infections using Electronic Health
  Records}.
\newblock In \emph{HEALTHINF}, pp.\  156--167, 2018.

\bibitem[Henderson et~al.(2023)Henderson, Chugg, Anderson, Altenburger, Turk,
  Guyton, Goldin, and Ho]{henderson2023integrating}
Henderson, P., Chugg, B., Anderson, B., Altenburger, K., Turk, A., Guyton, J.,
  Goldin, J., and Ho, D.~E.
\newblock Integrating reward maximization and population estimation: Sequential
  decision-making for internal revenue service audit selection.
\newblock In \emph{Proceedings of the AAAI Conference on Artificial
  Intelligence}, volume~37, pp.\  5087--5095, 2023.

\bibitem[Henry et~al.(2015)Henry, Hager, Pronovost, and
  Saria]{henry2015targeted}
Henry, K.~E., Hager, D.~N., Pronovost, P.~J., and Saria, S.
\newblock A targeted real-time early warning score (trewscore) for septic
  shock.
\newblock \emph{Science Translational Medicine}, 7\penalty0 (299):\penalty0
  299ra122--299ra122, 2015.

\bibitem[Hu et~al.(2022)Hu, Niu, Miao, Hua, and Zhang]{hu2022non}
Hu, X., Niu, Y., Miao, C., Hua, X.-S., and Zhang, H.
\newblock On non-random missing labels in semi-supervised learning.
\newblock In \emph{10th International Conference on Learning Representations},
  2022.

\bibitem[Imbens \& Rubin(2015)Imbens and Rubin]{imbens2015causal}
Imbens, G.~W. and Rubin, D.~B.
\newblock \emph{Causal inference in statistics, social, and biomedical
  sciences}.
\newblock Cambridge University Press, 2015.

\bibitem[Janzing(2019)]{janzing2019causal}
Janzing, D.
\newblock Causal regularization.
\newblock \emph{Advances in Neural Information Processing Systems}, 32, 2019.

\bibitem[Jehi et~al.(2020)Jehi, Ji, Milinovich, Erzurum, Rubin, Gordon, Young,
  and Kattan]{jehi2020individualizing}
Jehi, L., Ji, X., Milinovich, A., Erzurum, S., Rubin, B.~P., Gordon, S., Young,
  J.~B., and Kattan, M.~W.
\newblock Individualizing risk prediction for positive coronavirus disease 2019
  testing: results from 11,672 patients.
\newblock \emph{Chest}, 158\penalty0 (4):\penalty0 1364--1375, 2020.

\bibitem[Johnson et~al.(2016)Johnson, Pollard, Shen, Lehman, Feng, Ghassemi,
  Moody, Szolovits, Anthony~Celi, and Mark]{johnson2016mimic}
Johnson, A.~E., Pollard, T.~J., Shen, L., Lehman, L.-w.~H., Feng, M., Ghassemi,
  M., Moody, B., Szolovits, P., Anthony~Celi, L., and Mark, R.~G.
\newblock {MIMIC-III, a freely accessible critical care database}.
\newblock \emph{Scientific Data}, 3\penalty0 (1):\penalty0 1--9, 2016.

\bibitem[Johnson et~al.(2018)Johnson, Aboab, Raffa, Pollard, Deliberato, Celi,
  and Stone]{johnson2018comparative}
Johnson, A.~E., Aboab, J., Raffa, J.~D., Pollard, T.~J., Deliberato, R.~O.,
  Celi, L.~A., and Stone, D.~J.
\newblock A comparative analysis of sepsis identification methods in an
  electronic database.
\newblock \emph{Critical Care Medicine}, 46\penalty0 (4):\penalty0 494--499,
  2018.

\bibitem[Kamran et~al.(2022)Kamran, Tang, Ötleş, McEvoy, Saleh, Gong, Li,
  Dutta, Liu, Medford, Valley, West, Singh, Blumberg, Donnelly, Shenoy,
  Ayanian, Nallamothu, Sjoding, and Wiens]{MCURES}
Kamran, F., Tang, S., Ötleş, E., McEvoy, D.~S., Saleh, S.~N., Gong, J., Li,
  B.~Y., Dutta, S., Liu, X., Medford, R.~J., Valley, T.~S., West, L.~R., Singh,
  K., Blumberg, S., Donnelly, J.~P., Shenoy, E.~S., Ayanian, J.~Z., Nallamothu,
  B.~K., Sjoding, M.~W., and Wiens, J.
\newblock {Early identification of patients admitted to hospital for covid-19
  at risk of clinical deterioration: model development and multisite external
  validation study}.
\newblock \emph{The BMJ}, 376, 2022.

\bibitem[Kato et~al.(2023)Kato, Wu, Kureishi, and Yasui]{kato2023automatic}
Kato, M., Wu, S., Kureishi, K., and Yasui, S.
\newblock Automatic debiased learning from positive, unlabeled, and exposure
  data.
\newblock \emph{arXiv preprint arXiv:2303.04797}, 2023.

\bibitem[Kennedy(2023)]{kennedy2023towards}
Kennedy, E.~H.
\newblock Towards optimal doubly robust estimation of heterogeneous causal
  effects.
\newblock \emph{Electronic Journal of Statistics}, 17\penalty0 (2):\penalty0
  3008--3049, 2023.

\bibitem[Kiani et~al.(2023)Kiani, Barton, Sushinsky, Heimbach, and
  Luo]{kiani2023counterfactual}
Kiani, S., Barton, J., Sushinsky, J., Heimbach, L., and Luo, B.
\newblock Counterfactual prediction under selective confounding.
\newblock \emph{arXiv preprint arXiv:2310.14064}, 2023.

\bibitem[Kingma \& Ba(2015)Kingma and Ba]{kingma2014adam}
Kingma, D.~P. and Ba, J.
\newblock Adam: A method for stochastic optimization.
\newblock In \emph{3rd International Conference on Learning Representations},
  2015.

\bibitem[Kleinberg et~al.(2018)Kleinberg, Lakkaraju, Leskovec, Ludwig, and
  Mullainathan]{kleinberg2018human}
Kleinberg, J., Lakkaraju, H., Leskovec, J., Ludwig, J., and Mullainathan, S.
\newblock Human decisions and machine predictions.
\newblock \emph{The Quarterly Journal of Economics}, 133\penalty0 (1):\penalty0
  237--293, 2018.

\bibitem[Kontokosta \& Hong(2021)Kontokosta and Hong]{kontokosta2021bias}
Kontokosta, C.~E. and Hong, B.
\newblock Bias in smart city governance: How socio-spatial disparities in 311
  complaint behavior impact the fairness of data-driven decisions.
\newblock \emph{Sustainable Cities and Society}, 64:\penalty0 102503, 2021.

\bibitem[Laine et~al.(2020)Laine, Hyttinen, and
  Mathioudakis]{laine2020evaluating}
Laine, R., Hyttinen, A., and Mathioudakis, M.
\newblock Evaluating decision makers over selectively labelled data: A causal
  modelling approach.
\newblock In \emph{Discovery Science: 23rd International Conference, DS 2020,
  Thessaloniki, Greece, October 19--21, 2020, Proceedings 23}, pp.\  3--18,
  2020.

\bibitem[Lakkaraju et~al.(2017)Lakkaraju, Kleinberg, Leskovec, Ludwig, and
  Mullainathan]{lakkaraju2017selective}
Lakkaraju, H., Kleinberg, J., Leskovec, J., Ludwig, J., and Mullainathan, S.
\newblock The selective labels problem: Evaluating algorithmic predictions in
  the presence of unobservables.
\newblock In \emph{Proceedings of the 23rd ACM SIGKDD International Conference
  on Knowledge Discovery and Data Mining}, pp.\  275--284, 2017.

\bibitem[Laufer et~al.(2022)Laufer, Pierson, and Garg]{laufer2022end}
Laufer, B., Pierson, E., and Garg, N.
\newblock End-to-end auditing of decision pipelines.
\newblock In \emph{ICML Workshop on Responsible Decision-Making in Dynamic
  Environments.}, pp.\  1--7, 2022.

\bibitem[Lee(2013)]{lee2013pseudo}
Lee, D.-H.
\newblock Pseudo-label: The simple and efficient semi-supervised learning
  method for deep neural networks.
\newblock In \emph{Workshop on challenges in representation learning, ICML},
  volume 3(2), pp.\  896, 2013.

\bibitem[Li et~al.(2020)Li, Socher, and Hoi]{li2020dividemix}
Li, J., Socher, R., and Hoi, S. C.~H.
\newblock Dividemix: Learning with noisy labels as semi-supervised learning.
\newblock In \emph{8th International Conference on Learning Representations,
  {ICLR}}, 2020.

\bibitem[Li et~al.(2021)Li, Liu, Han, Niu, and Sugiyama]{li2021provably}
Li, X., Liu, T., Han, B., Niu, G., and Sugiyama, M.
\newblock Provably end-to-end label-noise learning without anchor points.
\newblock In \emph{Proceedings of the 38th International Conference on Machine
  Learning}, pp.\  6403--6413, 2021.

\bibitem[Liu \& Tao(2015)Liu and Tao]{liu2015classification}
Liu, T. and Tao, D.
\newblock Classification with noisy labels by importance reweighting.
\newblock \emph{IEEE Transactions on Pattern Analysis and Machine
  Intelligence}, 38\penalty0 (3):\penalty0 447--461, 2015.

\bibitem[Liu \& Guo(2020)Liu and Guo]{liu2020peer}
Liu, Y. and Guo, H.
\newblock Peer loss functions: Learning from noisy labels without knowing noise
  rates.
\newblock In \emph{Proceedings of the 37th International Conference on Machine
  Learning}, 2020.

\bibitem[Liu \& Garg(2022)Liu and Garg]{liu2022equity}
Liu, Z. and Garg, N.
\newblock Equity in resident crowdsourcing: Measuring under-reporting without
  ground truth data.
\newblock In \emph{Proceedings of the 23rd ACM Conference on Economics and
  Computation}, pp.\  1016--1017, 2022.

\bibitem[Madras et~al.(2019)Madras, Creager, Pitassi, and
  Zemel]{madras2019fairness}
Madras, D., Creager, E., Pitassi, T., and Zemel, R.
\newblock Fairness through causal awareness.
\newblock \emph{Proceedings of the Conference on Fairness, Accountability, and
  Transparency}, 2019.

\bibitem[McDonald et~al.(2021)McDonald, Medford, Basit, Diercks, and
  Courtney]{mcdonald2021derivation}
McDonald, S.~A., Medford, R.~J., Basit, M.~A., Diercks, D.~B., and Courtney,
  D.~M.
\newblock Derivation with internal validation of a multivariable predictive
  model to predict covid-19 test results in emergency department patients.
\newblock \emph{Academic Emergency Medicine}, 28\penalty0 (2):\penalty0
  206--214, 2021.

\bibitem[Menon et~al.(2018)Menon, Van~Rooyen, and Natarajan]{menon2018learning}
Menon, A.~K., Van~Rooyen, B., and Natarajan, N.
\newblock Learning from binary labels with instance-dependent noise.
\newblock \emph{Machine Learning}, 107\penalty0 (8):\penalty0 1561--1595, 2018.

\bibitem[Mulherin \& Miller(2002)Mulherin and Miller]{mulherin2002spectrum}
Mulherin, S.~A. and Miller, W.~C.
\newblock Spectrum bias or spectrum effect? subgroup variation in diagnostic
  test evaluation.
\newblock \emph{Annals of Internal Medicine}, 137\penalty0 (7):\penalty0
  598--602, 2002.

\bibitem[Mullainathan \& Obermeyer(2022)Mullainathan and
  Obermeyer]{mullainathan2022diagnosing}
Mullainathan, S. and Obermeyer, Z.
\newblock Diagnosing physician error: A machine learning approach to low-value
  health care.
\newblock \emph{The Quarterly Journal of Economics}, 137\penalty0 (2):\penalty0
  679--727, 2022.

\bibitem[Nguyen et~al.(2020)Nguyen, Mummadi, Ngo, Nguyen, Beggel, and
  Brox]{nguyen2020self}
Nguyen, D.~T., Mummadi, C.~K., Ngo, T., Nguyen, T. H.~P., Beggel, L., and Brox,
  T.
\newblock {{SELF:} Learning to Filter Noisy Labels with Self-Ensembling}.
\newblock In \emph{8th International Conference on Learning Representations},
  2020.

\bibitem[Paszke et~al.(2019)Paszke, Gross, Massa, Lerer, Bradbury, Chanan,
  Killeen, Lin, Gimelshein, Antiga, Desmaison, Köpf, Yang, DeVito, Raison,
  Tejani, Chilamkurthy, Steiner, Fang, Bai, and Chintala]{paszke2019pytorch}
Paszke, A., Gross, S., Massa, F., Lerer, A., Bradbury, J., Chanan, G., Killeen,
  T., Lin, Z., Gimelshein, N., Antiga, L., Desmaison, A., Köpf, A., Yang, E.,
  DeVito, Z., Raison, M., Tejani, A., Chilamkurthy, S., Steiner, B., Fang, L.,
  Bai, J., and Chintala, S.
\newblock Pytorch: An imperative style, high-performance deep learning library,
  2019.

\bibitem[Patrini et~al.(2017)Patrini, Rozza, Krishna~Menon, Nock, and
  Qu]{patrini2016making}
Patrini, G., Rozza, A., Krishna~Menon, A., Nock, R., and Qu, L.
\newblock Making deep neural networks robust to label noise: A loss correction
  approach.
\newblock In \emph{Proceedings of the IEEE Conference on Computer Vision and
  Pattern Recognition}, pp.\  1944--1952, 2017.

\bibitem[Pearl(2009)]{pearl2009causality}
Pearl, J.
\newblock \emph{Causality}.
\newblock Cambridge university press, 2009.

\bibitem[Pedregosa et~al.(2011)Pedregosa, Varoquaux, Gramfort, Michel, Thirion,
  Grisel, Blondel, Prettenhofer, Weiss, Dubourg, Vanderplas, Passos,
  Cournapeau, Brucher, Perrot, and Duchesnay]{scikit-learn}
Pedregosa, F., Varoquaux, G., Gramfort, A., Michel, V., Thirion, B., Grisel,
  O., Blondel, M., Prettenhofer, P., Weiss, R., Dubourg, V., Vanderplas, J.,
  Passos, A., Cournapeau, D., Brucher, M., Perrot, M., and Duchesnay, E.
\newblock Scikit-learn: Machine learning in {P}ython.
\newblock \emph{Journal of Machine Learning Research}, 12:\penalty0 2825--2830,
  2011.

\bibitem[Peng et~al.(2019)Peng, Nushi, K{\i}c{\i}man, Inkpen, Suri, and
  Kamar]{peng2019you}
Peng, A., Nushi, B., K{\i}c{\i}man, E., Inkpen, K., Suri, S., and Kamar, E.
\newblock What you see is what you get? the impact of representation criteria
  on human bias in hiring.
\newblock In \emph{Proceedings of the AAAI Conference on Human Computation and
  Crowdsourcing}, volume~7, pp.\  125--134, 2019.

\bibitem[Pierson et~al.(2018)Pierson, Corbett-Davies, and
  Goel]{pierson2018fast}
Pierson, E., Corbett-Davies, S., and Goel, S.
\newblock Fast threshold tests for detecting discrimination.
\newblock In \emph{International Conference on Artificial Intelligence and
  Statistics}, pp.\  96--105, 2018.

\bibitem[Pierson et~al.(2020)Pierson, Simoiu, Overgoor, Corbett-Davies, Jenson,
  Shoemaker, Ramachandran, Barghouty, Phillips, Shroff,
  et~al.]{pierson2020large}
Pierson, E., Simoiu, C., Overgoor, J., Corbett-Davies, S., Jenson, D.,
  Shoemaker, A., Ramachandran, V., Barghouty, P., Phillips, C., Shroff, R.,
  et~al.
\newblock A large-scale analysis of racial disparities in police stops across
  the united states.
\newblock \emph{Nature Human Behaviour}, 4\penalty0 (7):\penalty0 736--745,
  2020.

\bibitem[Rambachan \& Roth(2020)Rambachan and Roth]{rambachan2019bias}
Rambachan, A. and Roth, J.
\newblock {Bias in, bias out? Evaluating the folk wisdom}.
\newblock In \emph{1st Symposium on Foundations of Responsible Computing},
  2020.

\bibitem[Rhee \& Klompas(2020)Rhee and Klompas]{rhee2020sepsis}
Rhee, C. and Klompas, M.
\newblock Sepsis trends: increasing incidence and decreasing mortality, or
  changing denominator?
\newblock \emph{Journal of Thoracic Disease}, 12\penalty0 (Suppl 1):\penalty0
  S89, 2020.

\bibitem[Rizve et~al.(2021)Rizve, Duarte, Rawat, and Shah]{rizve2021defense}
Rizve, M.~N., Duarte, K., Rawat, Y.~S., and Shah, M.
\newblock In defense of pseudo-labeling: An uncertainty-aware pseudo-label
  selection framework for semi-supervised learning.
\newblock In \emph{9th International Conference on Learning Representations},
  2021.

\bibitem[Robins et~al.(1994)Robins, Rotnitzky, and Zhao]{robins1994estimation}
Robins, J.~M., Rotnitzky, A., and Zhao, L.~P.
\newblock Estimation of regression coefficients when some regressors are not
  always observed.
\newblock \emph{Journal of the American Statistical Association}, pp.\
  846--866, 1994.

\bibitem[Rockafellar(1970)]{rockafellar-1970a}
Rockafellar, R.~T.
\newblock \emph{Convex analysis}.
\newblock Princeton University Press, Princeton, N. J., 1970.

\bibitem[Rosenbaum \& Rubin(1983)Rosenbaum and Rubin]{rosenbaum1983central}
Rosenbaum, P.~R. and Rubin, D.~B.
\newblock The central role of the propensity score in observational studies for
  causal effects.
\newblock \emph{Biometrika}, 70\penalty0 (1):\penalty0 41--55, 1983.

\bibitem[Saxena et~al.(2020)Saxena, Badillo-Urquiola, Wisniewski, and
  Guha]{saxena2020childwelfare}
Saxena, D., Badillo-Urquiola, K., Wisniewski, P.~J., and Guha, S.
\newblock A human-centered review of algorithms used within the u.s. child
  welfare system.
\newblock In \emph{Proceedings of the 2020 CHI Conference on Human Factors in
  Computing Systems}, pp.\  1–15, 2020.

\bibitem[Schulman et~al.(1999)Schulman, Berlin, Harless, Kerner, Sistrunk,
  Gersh, Dube, Taleghani, Burke, Williams, et~al.]{schulman1999effect}
Schulman, K.~A., Berlin, J.~A., Harless, W., Kerner, J.~F., Sistrunk, S.,
  Gersh, B.~J., Dube, R., Taleghani, C.~K., Burke, J.~E., Williams, S., et~al.
\newblock The effect of race and sex on physicians' recommendations for cardiac
  catheterization.
\newblock \emph{New England Journal of Medicine}, 340\penalty0 (8):\penalty0
  618--626, 1999.

\bibitem[Seymour et~al.(2016)Seymour, Liu, Iwashyna, Brunkhorst, Rea, Scherag,
  Rubenfeld, Kahn, Shankar-Hari, Singer, et~al.]{seymour2016assessment}
Seymour, C.~W., Liu, V.~X., Iwashyna, T.~J., Brunkhorst, F.~M., Rea, T.~D.,
  Scherag, A., Rubenfeld, G., Kahn, J.~M., Shankar-Hari, M., Singer, M., et~al.
\newblock Assessment of clinical criteria for sepsis: for the third
  international consensus definitions for sepsis and septic shock (sepsis-3).
\newblock \emph{JAMA}, 315\penalty0 (8):\penalty0 762--774, 2016.

\bibitem[Shalev-Shwartz \& Ben-David(2014)Shalev-Shwartz and
  Ben-David]{shalev2014understanding}
Shalev-Shwartz, S. and Ben-David, S.
\newblock \emph{Understanding machine learning: From theory to algorithms}.
\newblock Cambridge University Press, 2014.

\bibitem[Shanmugam et~al.(2024)Shanmugam, Hou, and
  Pierson]{shanmugam2021quantifying}
Shanmugam, D., Hou, K., and Pierson, E.
\newblock Quantifying disparities in intimate partner violence: a machine
  learning method to correct for underreporting.
\newblock \emph{npj Women’s Health}, 2\penalty0 (1), 2024.

\bibitem[Shi et~al.(2019)Shi, Blei, and Veitch]{shi2019adapting}
Shi, C., Blei, D., and Veitch, V.
\newblock Adapting neural networks for the estimation of treatment effects.
\newblock \emph{Advances in Neural Information Processing Systems}, 32, 2019.

\bibitem[Singer et~al.(2016)Singer, Deutschman, Seymour, Shankar-Hari, Annane,
  Bauer, Bellomo, Bernard, Chiche, Coopersmith, et~al.]{singer2016third}
Singer, M., Deutschman, C.~S., Seymour, C.~W., Shankar-Hari, M., Annane, D.,
  Bauer, M., Bellomo, R., Bernard, G.~R., Chiche, J.-D., Coopersmith, C.~M.,
  et~al.
\newblock {The third international consensus definitions for sepsis and septic
  shock (Sepsis-3)}.
\newblock \emph{Jama}, 315\penalty0 (8):\penalty0 801--810, 2016.

\bibitem[Sportisse et~al.(2023)Sportisse, Schmutz, Humbert, Bouveyron, and
  Mattei]{sportisse2023labels}
Sportisse, A., Schmutz, H., Humbert, O., Bouveyron, C., and Mattei, P.-A.
\newblock Are labels informative in semi-supervised learning? estimating and
  leveraging the missing-data mechanism.
\newblock In \emph{Proceedings of the 40th International Conference on Machine
  Learning}, 2023.

\bibitem[S\"{u}hr et~al.(2021)S\"{u}hr, Hilgard, and
  Lakkaraju]{suhr2021fairranking}
S\"{u}hr, T., Hilgard, S., and Lakkaraju, H.
\newblock Does fair ranking improve minority outcomes? {U}nderstanding the
  interplay of human and algorithmic biases in online hiring.
\newblock In \emph{Proceedings of the 2021 AAAI/ACM Conference on AI, Ethics,
  and Society}, pp.\  989–999, 2021.

\bibitem[Teeple et~al.(2020)Teeple, Hartvigsen, Sen, Claypool, and
  Rundensteiner]{teeple2020clinical}
Teeple, E., Hartvigsen, T., Sen, C., Claypool, K.~T., and Rundensteiner, E.~A.
\newblock Clinical performance evaluation of a machine learning system for
  predicting hospital-acquired clostridium difficile infection.
\newblock In \emph{HEALTHINF}, pp.\  656--663, 2020.

\bibitem[{The pandas development team}(2020)]{reback2020pandas}
{The pandas development team}.
\newblock pandas-dev/pandas: Pandas, 2020.

\bibitem[Tjandra \& Wiens(2023)Tjandra and Wiens]{tjandra2023alignment}
Tjandra, D. and Wiens, J.
\newblock Leveraging an alignment set in tackling instance-dependent label
  noise.
\newblock In \emph{Proceedings of the Conference on Health, Inference, and
  Learning}, 2023.

\bibitem[Van Der~Laan \& Rubin(2006)Van Der~Laan and Rubin]{van2006targeted}
Van Der~Laan, M.~J. and Rubin, D.
\newblock Targeted maximum likelihood learning.
\newblock \emph{The International Journal of Biostatistics}, 2\penalty0 (1),
  2006.

\bibitem[Virtanen et~al.(2020)Virtanen, Gommers, Oliphant, Haberland, Reddy,
  Cournapeau, Burovski, Peterson, Weckesser, Bright, {van der Walt}, Brett,
  Wilson, Millman, Mayorov, Nelson, Jones, Kern, Larson, Carey, Polat, Feng,
  Moore, {VanderPlas}, Laxalde, Perktold, Cimrman, Henriksen, Quintero, Harris,
  Archibald, Ribeiro, Pedregosa, {van Mulbregt}, and {SciPy 1.0
  Contributors}]{2020SciPy-NMeth}
Virtanen, P., Gommers, R., Oliphant, T.~E., Haberland, M., Reddy, T.,
  Cournapeau, D., Burovski, E., Peterson, P., Weckesser, W., Bright, J., {van
  der Walt}, S.~J., Brett, M., Wilson, J., Millman, K.~J., Mayorov, N., Nelson,
  A. R.~J., Jones, E., Kern, R., Larson, E., Carey, C.~J., Polat, {\.I}., Feng,
  Y., Moore, E.~W., {VanderPlas}, J., Laxalde, D., Perktold, J., Cimrman, R.,
  Henriksen, I., Quintero, E.~A., Harris, C.~R., Archibald, A.~M., Ribeiro,
  A.~H., Pedregosa, F., {van Mulbregt}, P., and {SciPy 1.0 Contributors}.
\newblock {{SciPy} 1.0: Fundamental Algorithms for Scientific Computing in
  Python}.
\newblock \emph{Nature Methods}, 17:\penalty0 261--272, 2020.

\bibitem[Vogel et~al.(2021)Vogel, Bellet, and
  Cl{\'e}men{\c{c}}on]{vogel2021learning}
Vogel, R., Bellet, A., and Cl{\'e}men{\c{c}}on, S.
\newblock {Learning fair scoring functions: Bipartite ranking under ROC-based
  fairness constraints}.
\newblock In \emph{International Conference on Artificial Intelligence and
  Statistics}, pp.\  784--792, 2021.

\bibitem[Wager(2020)]{wager2020stats}
Wager, S.
\newblock Stats 361: Causal inference.
\newblock Technical report, Stanford University, 2020.

\bibitem[Wald \& Saria(2023)Wald and Saria]{wald2023birds}
Wald, Y. and Saria, S.
\newblock Birds of an odd feather: {G}uaranteed out-of-distribution ({OOD})
  novel category detection.
\newblock In \emph{Uncertainty in Artificial Intelligence}, pp.\  2179--2191,
  2023.

\bibitem[Wang et~al.(2021)Wang, Liu, and Levy]{wang2021gpl}
Wang, J., Liu, Y., and Levy, C.
\newblock Fair classification with group-dependent label noise.
\newblock In \emph{Proceedings of the 2021 ACM Conference on Fairness,
  Accountability, and Transparency}, pp.\  526–536, 2021.

\bibitem[Wang et~al.(2024)Wang, Chen, Guo, and Wang]{wang2024pue}
Wang, X., Chen, H., Guo, T., and Wang, Y.
\newblock Pue: Biased positive-unlabeled learning enhancement by causal
  inference.
\newblock \emph{Advances in Neural Information Processing Systems}, 36, 2024.

\bibitem[Wu et~al.(2022)Wu, Gong, Han, Liu, and Liu]{wu2022fairidn}
Wu, S., Gong, M., Han, B., Liu, Y., and Liu, T.
\newblock Fair classification with instance-dependent label noise.
\newblock In \emph{Proceedings of the First Conference on Causal Learning and
  Reasoning}, volume 177 of \emph{Proceedings of Machine Learning Research},
  pp.\  927--943, Apr 2022.

\bibitem[Yao et~al.(2021)Yao, Liu, Gong, Han, Niu, and Zhang]{yao2021causalnl}
Yao, Y., Liu, T., Gong, M., Han, B., Niu, G., and Zhang, K.
\newblock Instance-dependent label-noise learning under a structural causal
  model.
\newblock In \emph{Advances in Neural Information Processing Systems},
  volume~34, pp.\  4409--4420, 2021.

\bibitem[Zhang et~al.(2021)Zhang, Zheng, Wu, Goswami, and
  Chen]{prog_noise_iclr2021}
Zhang, Y., Zheng, S., Wu, P., Goswami, M., and Chen, C.
\newblock Learning with feature-dependent label noise: A progressive approach.
\newblock In \emph{9th International Conference on Learning Representations},
  2021.

\bibitem[Zhang \& Sabuncu(2018)Zhang and Sabuncu]{zhang2018lq}
Zhang, Z. and Sabuncu, M.~R.
\newblock Generalized cross entropy loss for training deep neural networks with
  noisy labels.
\newblock In \emph{Proceedings of the 32nd International Conference on Neural
  Information Processing Systems}, pp.\  8792–8802, 2018.

\bibitem[Zhao et~al.(2022)Zhao, Li, Qin, Liu, and Yu]{zhao2022tscsiidn}
Zhao, G., Li, G., Qin, Y., Liu, F., and Yu, Y.
\newblock Centrality and consistency: Two-stage clean samples identification
  for learning with instance-dependent noisy labels.
\newblock In Avidan, S., Brostow, G., Ciss{\'e}, M., Farinella, G.~M., and
  Hassner, T. (eds.), \emph{Computer Vision -- ECCV 2022}, pp.\  21--37, 2022.

\bibitem[Zhu \& Ghahramani(2002)Zhu and Ghahramani]{zhu2002learning}
Zhu, X. and Ghahramani, Z.
\newblock Learning from labeled and unlabeled data with label propagation.
\newblock Technical report, Center for Automated Learning and Discovery,
  Carnegie Mellon University, 2002.

\bibitem[Zhu et~al.(2021)Zhu, Song, and Liu]{zhu2021clusterability}
Zhu, Z., Song, Y., and Liu, Y.
\newblock Clusterability as an alternative to anchor points when learning with
  noisy labels.
\newblock In \emph{Proceedings of the 38th International Conference on Machine
  Learning}, pp.\  12912--12923, 2021.

\end{thebibliography}
\bibliographystyle{icml2024}

\newpage
\appendix
\onecolumn

\section{Selective labeling in the literature}
\label{app:lit_review}

We enumerate domains in which our literature review found instances of selective label problems in the ML methods and applications literature:
\begin{itemize}
    \item \textbf{Healthcare}: Laboratory/diagnostic testing~\cite{chang2022disparate, mullainathan2022diagnosing} and diagnosis~\cite{farahani2020explanatory, shanmugam2021quantifying, balachandar2023domain}
    \item \textbf{Social \& public policy:} Child welfare assessment~\cite{saxena2020childwelfare, kiani2023counterfactual}, urban planning/policy~\cite{kontokosta2021bias, laufer2022end, liu2015classification}, hiring pipelines~\cite{peng2019you, suhr2021fairranking}, student placement~\cite{bergman2021using}, and bias in policing~\cite{rambachan2019bias, pierson2020large}
    \item \textbf{Finance:} Credit repayment~\cite{bjorkegren2020behavior} and financial auditing~\cite{henderson2023integrating}
    \item \textbf{Others/miscellaneous:} Wildlife conservation~\cite{gholami2018adversary}, social media content moderation~\cite{binns2017like}
\end{itemize}

We note that this is not an exhaustive list of all papers in the selective labeling literature or related problem settings. However, this list illustrates the broad applicability and relevance of our problem setting.

\section{Omitted Proofs}
\label{app:proofs}

For convenience, we restate all theorems and propositions here.

\subsection{Theorem 3.1}
\begin{theorem*}[E-step derivation]
    The posterior conditional mean of $y^{(i)} = 1$ given the observed data, $Q(y^{(i)}) \triangleq \mathbb{E}[y^{(i)} =1\mid t^{(i)}, \tilde{y}^{(i)}, a^{(i)}, \*x^{(i)}]$, is equal to
    \begin{align}
    Q(y^{(i)}) &= \begin{cases}
        \tilde{y}^{(i)} & t^{(i)} = 1\\
        \mathbb{E}[y^{(i)} = 1 \mid \*x^{(i)}] & \text{otherwise}
    \end{cases}.
\end{align}
\end{theorem*}

\begin{proof}
We drop superscripts $(\cdot)^{(i)}$ in the proof for clarity.
Denote $Q(y)$ as posterior distribution of $y$ given the observed data, $\mathbb{E}[y =1\mid t, \tilde{y}, a, \*x]$ (\emph{i.e.}, the E-step estimate). First, we can write
\begin{align}
    Q(y) &\triangleq \mathbb{E}[y=1 \mid t, \tilde{y}, a, \*x] =\mathbb{E}[y=1 \mid \tilde{y}, x, t] = P(y = 1 \mid \tilde{y}, x, t) 
\end{align}
for simplicity, where we use the fact that $Y \perp\!\!\!\perp A \mid (T, X)$
and the fact that $\mathbb{E}[y=1 \mid \tilde{y}, x, t]$ is binary. Proceeding, we can write:
\begin{align}
    &= t \cdot P(y=1 \mid \tilde{y}, x, t=1) + (1-t) \cdot P(y =1 \mid \tilde{y}, x, t=0) \\
    &= t \tilde{y} + (1-t) [\tilde{y} P(y=1 \mid \tilde{y}=1, x, t=0) + (1 - \tilde{y} )P(y=1\mid \tilde{y} = 0, x, t=0)] \\ 
    &= t \tilde{y}+ (1-t) (1 - \tilde{y}) P(y=1\mid \tilde{y} = 0, x, t=0) \\ 
    &= t \tilde{y}+ (1-t) (1 - \tilde{y}) \cdot\frac{P( \tilde{y} = 0 \mid y=1, x, t=0)P(y=1 \mid x, t=0)}{P(\tilde{y} = 0 \mid x, t = 0)}\\
    &= t \tilde{y}+ (1-t) (1 - \tilde{y}) P(y=1 \mid x).
\end{align}
The second equality follows since $t=1\implies \tilde{y}=y$. The third equality holds since $P(y=1 \mid \tilde{y} = 1, x, t=0) = 0$ by construction, since $\tilde{y} = yt$ under disparate censorship. The final step follows from three facts: (1) $P(\tilde{y}=0 \mid y=1, x, t=0) = 1$ for all $x$, (2) $P(\tilde{y} = 0 \mid x, t=0) = 1$ for all $x$, and (3) $Y \perp\!\!\!\perp T \mid X$. This is more succinctly rewritten as E-step is:
\begin{equation}
    Q(y) = \begin{cases}
        \tilde{y} & t = 1\\
        \mathbb{E}[y=1 \mid x] & \text{otherwise (i.e., $\tilde{y} = 0 \land t = 0$)}
    \end{cases},
\end{equation}
which is what we wanted to show.
\end{proof}

\begin{remark}
Since $y^{(i)}$ is binary by assumption, this result fully determines the posterior distribution since $\mathbb{E}[y=1 \mid \tilde{y}, x, t] = 1 - \mathbb{E}[y=0 \mid \tilde{y}, x, t]$. 
\end{remark}

\subsection{Theorem 3.2}
\begin{theorem*}[M-step derivation]
Let $P(U, A, X, T, \tilde{Y}; \theta)$ be a model for the joint data distribution parameterized by some arbitrary $\theta \in \Theta$ in some parameter space $\Theta$, which factorizes according to the disparate censorship DAG =(Fig.~\ref{fig:main}).
Let $Q(y) \triangleq \mathbb{E}[y=1 \mid t, \tilde{y}, a, \*x]$  be the posterior expectation that $y = 1$ given the observed data. Then (replacing random variables with their realized counterparts), we have
\begin{align}
    \max_\theta \; \frac{1}{N} \sum_{i=1}^N \log P(u^{(i)}, a^{(i)}, \*x^{(i)}, t^{(i)}, \tilde{y}^{(i)}; \theta) \geq &\max_\theta \frac{1}{N} \sum_{i=1}^N  Q(y^{(i)}) \log P(y^{(i)} \mid \*x^{(i)}; \theta_{Y \mid X}) \nonumber \\
    &\quad + (1 - Q(y^{(i)})) \log (1 - P(y^{(i)} \mid \*x^{(i)}; \theta_{Y \mid X})) \nonumber \\
    &\quad + Q(y^{(i)}) \log P(\tilde{y}^{(i)} \mid y^{(i)}, t^{(i)}; \theta_{\tilde{Y} \mid Y, T}) 
\end{align}
where $\theta = [\theta_{Y \mid X} \; \theta_{\tilde{Y} \mid Y, T}]^\top$.
\end{theorem*}

\begin{proof}
We first construct the evidence-based lower bound (ELBO) of the LHS in the theorem statement. First, for a single example indexed by $i$, we can write:
\begin{align}
    \log P(u^{(i)}, a^{(i)}, \*x^{(i)}, t^{(i)}, \tilde{y}^{(i)}; \theta) &= \log  \sum_{y^{(i)} \in \{0, 1\}} Q(y^{(i)}) \frac{P(u^{(i)}, a^{(i)}, \*x^{(i)}, t^{(i)}, \tilde{y}^{(i)}, y^{(i)}; \theta)}{Q(y^{(i)})}\\
    &\geq  \sum_{y^{(i)} \in \{0, 1\}} Q(y^{(i)}) \log \frac{P(u^{(i)}, a^{(i)}, \*x^{(i)}, t^{(i)}, \tilde{y}^{(i)}, y^{(i)}; \theta)}{Q(y^{(i)})}
\end{align}
via Jensen's inequality. Then, we note that
\begin{align}
    &\underset{\theta}{\max}  \frac{1}{N} \sum_{i=1}^N \sum_{y^{(i)} \in \{0, 1\}} Q(y^{(i)}) \log \frac{P(u^{(i)}, a^{(i)}, \*x^{(i)}, t^{(i)}, \tilde{y}^{(i)}, y^{(i)}; \theta)}{Q(y^{(i)})} \nonumber \\
      =\;& \underset{\theta}{\max} \frac{1}{N} \sum_{i=1}^N  \sum_{y^{(i)} \in \{0, 1\}} Q(y^{(i)}) \log P(u^{(i)}, a^{(i)}, \*x^{(i)}, t^{(i)}, \tilde{y}^{(i)}, y^{(i)}; \theta),
\end{align}
dropping $Q(y^{(i)}) \log Q(y^{(i)})$, which is constant with respect to $\theta$, after expanding the $\log$ term. We can then use the DAG to factorize the joint distribution of all variables (including latent variable $Y$), which is given by
\begin{align}
    P(\tilde{Y}, Y, T, X, A, U) &= P(\tilde{Y} \mid Y, T) P(Y \mid X) P(T \mid X, A)P(X, A, U).
\end{align}
Note that we need only model the first two terms for estimation of $y^{(i)}$. The first two terms do not involve $y^{(i)}$, are not parameterized, and can be dropped from the maximization problem. Hence, we proceed to write 
\begin{equation}
    =\; \underset{\theta}{\max}  \; \frac{1}{N} \sum_{i=1}^N \sum_{y^{(i)} \in \{0, 1\}} Q(y^{(i)}) \log P(y^{(i)} \mid \*x^{(i)}; \theta_{\tilde{Y} \mid X}) P(\tilde{y}^{(i)} \mid t^{(i)}, y^{(i)}; \theta_{\tilde{Y} \mid Y, T})
\end{equation}
where $\theta = [\theta_{Y \mid X} \; \theta_{\tilde{Y} \mid Y, T}]^\top$. This can be rewritten as
\begin{equation}
    =\; \underset{\theta}{\max}  \; \frac{1}{N} \sum_{i=1}^N \sum_{y^{(i)} \in \{0, 1\}} Q(y^{(i)}) \log P(y^{(i)} \mid \*x^{(i)}; \theta_{\tilde{Y} \mid X}) + Q(y^{(i)}) P(\tilde{y}^{(i)} \mid t^{(i)}, y^{(i)}; \theta_{\tilde{Y} \mid Y, T}),
\end{equation}
at which point we note that it is sufficient to show that
\begin{equation}
    (1 - Q(y^{(i)})) \log(1 - P(\tilde{y}^{(i)} \mid y^{(i)}, t^{(i)}; \theta_{\tilde{Y} \mid Y, T})) 
\end{equation}
is constant in $\theta$.
We can rewrite the above as 
\begin{equation}
    (1 - Q(y^{(i)})) [\tilde{y}^{(i)} \log(P(\tilde{y}^{(i)} = 1\mid y^{(i)} = 0, t^{(i)}; \theta_{\tilde{Y} \mid Y, T})) + (1-\tilde{y}^{(i)} )\log(P(\tilde{y}^{(i)} = 0\mid y^{(i)} = 0, t^{(i)}; \theta_{\tilde{Y} \mid Y, T}))].
\end{equation}
First, note that the event $\{\tilde{y}^{(i)} = 1 \mid y^{(i)} = 0\}$ occurs with probability zero by definition (recall $\tilde{y}^{(i)} = y^{(i)} t^{(i)}$). Thus, $P(\tilde{y}^{(i)} = 1\mid y^{(i)} = 0, t^{(i)}; \theta_{\tilde{Y} \mid Y, T}))$ cannot change with respect to $\theta$; we drop it from the maximization problem. 
Similarly, $P(\tilde{y}^{(i)} = 0\mid y^{(i)} = 0, t^{(i)}) = 0$ by definition, so $(1 - Q(y^{(i)})) \log(1 - P(\tilde{y}^{(i)} \mid y^{(i)}, t^{(i)})) = 0$ which is constant as needed, from which the theorem follows.
\end{proof}

\begin{remark}
In the theorem statement, replacing $P(y^{(i)} \mid \*x^{(i)}; \theta_{Y \mid X})$ with $\hat{y}^{(i)}$ and $P(\tilde{y}^{(i)} \mid y^{(i)}, t^{(i)}; \theta_{\tilde{Y} \mid Y, T})$ with $h_\phi(\cdot)$, assuming $y^{(i)}$ and $\tilde{y}^{(i)}$ are binary, and writing the explicit form of negative binary cross-entropy (\emph{e.g.}, $y \log y + (1-y) \log \hat{y}$) recovers the form of the M-step objective seen in Eq.~\ref{eq:mstep_ll}.  Note that the optimization problem flips from a maximization to a minimization due to the relationship between maximizing log-likelihood of binary variable(s) and minimizing cross-entropy loss.
\end{remark}

\subsection{Theorem~\ref{prop:counterbalance}}

\begin{theorem*}[Strength of the causal regularizer in $\hat{t}^{(i)}$]
For an example indexed by $i$, $Q(y^{(i)}) \in [0, 1)$, and $J^{(i)}$ defined as in Eq.~\ref{eq:creg_def}, $R(J^{(i)})$ is monotonically increasing in $\hat{t}^{(i)}$ on $(0, 1]$.\label{thm:causalreg}
\end{theorem*}

\begin{proof}
As a proof outline, we first show the closed-form of $\hat{y}^{(i)}_{\text{OPT}}(Q(y^{(i)}), \hat{t}^{(i)})$ by solving the first-order optimality condition of Eq.~\ref{eq:creg_def}.  Then, we show that $\hat{y}^{(i)}_{\text{OPT}}(Q(y^{(i)}), \hat{t}^{(i)})$ is decreasing in $\hat{t}^{(i)}$, and attains a maximum of $Q(y^{(i)})$ as $\hat{t}^{(i)} \to 0^+$. We conclude by showing that the preceding implies that $R(J^{(i)})$ is monotonically increasing in $\hat{t}^{(i)}$ on $(0, 1]$, as desired.

The first-order optimality condition of Eq.~\ref{eq:creg_def} is
\begin{equation}
    -\left(\frac{Q(y^{(i)})}{\hat{y}^{(i)}} - \frac{1 - Q(y^{(i)})}{1 - \hat{y}^{(i)}}\right) + \frac{\hat{t}^{(i)}Q(y^{(i)})}{1 - \hat{y}^{(i)}\hat{t}^{(i)}} = 0.\label{eq:first_order}
\end{equation}

By assumption (convexity of $\mathcal{L}$), the minimizer is unique. Some algebra yields
\begin{align}
     &-Q(y^{(i)})(1 - \hat{y}^{(i)})(1 - \hat{y}^{(i)} \hat{t}^{(i)})  + (1 - Q(y^{(i)}))\hat{y}^{(i)}(1 - \hat{y}^{(i)} \hat{t}^{(i)}) + \hat{t}^{(i)}Q(y^{(i)})\hat{y}^{(i)}(1 - \hat{y}^{(i)}) &= 0\\
     \iff &(\hat{t}^{(i)} + Q(y^{(i)})\hat{t}^{(i)})\hat{y}^{(i)2} - (2Q(y^{(i)})\hat{t}^{(i)} + 1) \hat{y}^{(i)} + Q(y^{(i)}) &= 0,
\end{align}
from which we can apply the quadratic formula.
Define $B(Q(y^{(i)}), \hat{t}^{(i)}) \triangleq 2Q(y^{(i)})\hat{t}^{(i)} + 1$ and $D(Q(y^{(i)}), \hat{t}^{(i)}) \triangleq (-B(Q(y^{(i)}), \hat{t}^{(i)})) ^2 - 4Q(y^{(i)})(Q(y^{(i)})\hat{t}^{(i)} + \hat{t}^{(i)})$.\footnote{We use letter $B(\cdot)$ because it corresponds to coefficient $b$ in the conventional quadratic formula: $$x = \frac{-b \pm \sqrt{b^2 - 4ac}}{2a}$$ for a quadratic polynomial $ax^2 + bx + c = 0$. We choose letter $D(\cdot)$ since it corresponds to the discriminant.} The quadratic formula yields solutions
\begin{equation}
    \hat{y}^{(i)}_{\text{OPT}}(Q(y^{(i)}), \hat{t}^{(i)}) = \frac{B(Q(y^{(i)}), \hat{t}^{(i)}) \pm \sqrt{D(Q(y^{(i)}, \hat{t}^{(i)}) }}{2(\hat{t}^{(i)} + Q(y^{(i)})\hat{t}^{(i)}) }.\label{eq:quad_result}
\end{equation}
We use the fact that  $\hat{y}^{(i)}$ must be in [0, 1] and the constraints that $\hat{t}^{(i)} \in (0, 1]$ and $Q(y^{(i)}) \in [0, 1)$ to determine which branch of Eq.~\ref{eq:quad_result}
yields real solutions in $[0, 1]$. By Lemma~\ref{lem:d_pos}, $D(Q(y^{(i)}, \hat{t}^{(i)}) \geq 0$, so the solutions are real. 
Then, by Lemma~\ref{lem:oob},
\begin{equation}
    \frac{B(Q(y^{(i)}), \hat{t}^{(i)}) + \sqrt{D(Q(y^{(i)}, \hat{t}^{(i)}) }}{2(\hat{t}^{(i)} + Q(y^{(i)})\hat{t}^{(i)}) } \geq 1,
\end{equation}
eliminating that branch. By elimination, 
\begin{equation}
    \hat{y}^{(i)}_{\text{OPT}}(Q(y^{(i)}), \hat{t}^{(i)}) = \frac{B(Q(y^{(i)}), \hat{t}^{(i)}) - \sqrt{D(Q(y^{(i)}, \hat{t}^{(i)}) }}{2(\hat{t}^{(i)} + Q(y^{(i)})\hat{t}^{(i)}) }.
\end{equation}

Lemma~\ref{lem:in_bounds} verifies that the resulting $\hat{y}^{(i)}_{\text{OPT}}(Q(y^{(i)}), \hat{t}^{(i)})$ are in [0, 1], as needed.
To proceed, it suffices to show that $\hat{y}^{(i)}_{\text{OPT}}(Q(y^{(i)}), \hat{t}^{(i)})$ is decreasing in $\hat{t}^{(i)}$ and attains a maximum of $Q(y^{(i)})$ as $\hat{t}^{(i)} \to 0^+$. 

Applying Lemma~\ref{lem:sign_preserve}, to prove that $\hat{y}^{(i)}_{\text{OPT}}(Q(y^{(i)}), \hat{t}^{(i)})$ decreases in $\hat{t}^{(i)}$, it is sufficient to show
\begin{equation}
    1 - 2\hat{t}^{(i)}Q(y^{(i)})^2 - \sqrt{D(Q(y^{(i)}), \hat{t}^{(i)})}) < 0
\end{equation}
because $1 - 2\hat{t}^{(i)}Q(y^{(i)})^2 - \sqrt{D(Q(y^{(i)}, \hat{t}^{(i)})}$ has the same sign as the derivative of $\hat{y}^{(i)}_{\text{OPT}}(Q(y^{(i)}), \hat{t}^{(i)})$ for the values of $(\hat{t}^{(i)}, Q(y^{(i)}))$ of interest.

For values of $\hat{t}^{(i)} \in (0, 1]$ and $Q(y^{(i)}) \in [0, 1)$ such that $1 - 2\hat{t}^{(i)}Q(y^{(i)})^2 < 0$,  Lemma~\ref{lem:d_pos} yields the desired result. For the remaining values of $(\hat{t}^{(i)}, Q(y^{(i)}))$, we can write 
\begin{align}
    & 1 - 4Q(y^{(i)})^2\hat{t}^{(i)} + 4Q(y^{(i)})^4\hat{t}^{(i)2} < D(Q(y^{(i)}), \hat{t}^{(i)})\\
    \iff & 1 - 4Q(y^{(i)})^2\hat{t}^{(i)} + 4Q(y^{(i)})^4\hat{t}^{(i)2} < 1 - 4Q(y^{(i)})^2 \hat{t}^{(i)} + 4Q(y^{(i)})^2 \hat{t}^{(i)2}\\
    \iff & Q(y^{(i)})^2 < 1 
\end{align}
which holds for all feasible values of $Q(y^{(i)}) \in [0, 1)$.
Lastly, due to the monotonicity of $\hat{y}^{(i)}_{\text{OPT}}(Q(y^{(i)}), \hat{t}^{(i)})$ in $\hat{t}^{(i)}$, the following one-sided limit is the maximum:
\begin{equation}
    \underset{\hat{t}^{(i)} \to 0^+}{\lim} \; \hat{y}^{(i)}_{\text{OPT}}(Q(y^{(i)}), \hat{t}^{(i)}) = \underset{\hat{t}^{(i)}  \in (0, 1]}{\max} \hat{y}^{(i)}_{\text{OPT}}(Q(y^{(i)}), \hat{t}^{(i)}) .
\end{equation}
We want to show that the limit is $Q(y^{(i)})$. Note that, since $\mathcal{L}$ is finite and convex, it is continuous (Corollary 10.1.1,~\cite{rockafellar-1970a}); hence, this limit exists. Since substituting $\hat{t}^{(i)} = 0$ yields the indeterminate form $0/0$, we appeal to L'H\^{o}pital's rule:
\begin{equation}
    \underset{\hat{t}^{(i)} \to 0^+}{\lim} \; \hat{y}^{(i)}_{\text{OPT}}(Q(y^{(i)}), \hat{t}^{(i)}) = \underset{\hat{t}^{(i)} \to 0^+}{\lim} \;  \frac{2Q(y^{(i)}) - \frac{4\hat{t}^{(i)} Q(y^{(i)})^2 - 2Q(y^{(i)})^2}{\sqrt{D(Q(y^{(i)}), \hat{t}^{(i)})}}}{2(Q(y^{(i)}) + 1)} =   \frac{2Q(y^{(i)}) + 2Q(y^{(i)})^2}{2(Q(y^{(i)}) + 1)} = Q(y^{(i)}).
\end{equation}
Note that $\left.\sqrt{D(Q(y^{(i)}), \hat{t}^{(i)})}\;\right|_{\hat{t}^{(i)} = 0} = 1$.
Since $Q(y^{(i)}) - \hat{y}^{(i)}_{\text{OPT}}(Q(y^{(i)}), \hat{t}^{(i)}) > 0$,
\begin{equation}
    R(J^{(i)}) = |Q({y}^{(i)}) - \hat{y}^{(i)}_{\text{OPT}}(Q(y^{(i)}), \hat{t}^{(i)})| = Q({y}^{(i)}) - \hat{y}^{(i)}_{\text{OPT}}(Q(y^{(i)}), \hat{t}^{(i)}).
\end{equation}
Furthermore, since $\hat{y}^{(i)}_{\text{OPT}}(Q(y^{(i)}), \hat{t}^{(i)})$ is decreasing in $\hat{t}^{(i)}$, $R(J^{(i)})$ must increase in $\hat{t}^{(i)}$, from which the theorem follows.
\end{proof}

\begin{remark}
We comment on the potential for DCEM to improve robustness to low overlap. 
    To do so, we analyze the sensitivity of the M-step optimum to extreme $\hat{t}^{(i)}$. While analyzing the asymptotic variance of consistent estimators is a common approach, asymptotic guarantees for DCEM are unclear due to the inherently non-convex (with respect to the parameters) objective function. Thus, we analyze the Lipschitzness of the M-step optimum versus other causal effect estimators. First, note that 
    \begin{equation}
        \frac{d}{d\hat{t}} \hat{y}_{OPT}(Q(y), \hat{t}) = \frac{1 - 2Q(y)^2 \hat{t}^2 - \sqrt{C}}{2(Q(y) + 1)\hat{t}^2 \sqrt{C}}
    \end{equation}
    where $C = 4Q(y)^2 \hat{t}^2 - 4Q(y)^2 \hat{t} + 1$. For all $Q(y^{(i)}) < 1$ and all $\hat{t^{(i)}}$, this derivative is \textit{bounded} (\emph{e.g.}, see Figure~\ref{fig:optim_mstep}), and is $O(1)-$Lipschitz in $\hat{t}^{(i)}$. However, consider the expression for an inverse-propensity-weighted estimator, which sums terms of the form
    \begin{equation}
        \frac{y^{(i)}t^{(i)}}{\hat{t}^{(i)}} -\frac{y^{(i)}(1 - t^{(i)})}{1-\hat{t}^{(i)}}\label{eq:ipw}
    \end{equation}
    to obtain a final estimate. Eq.~\ref{eq:ipw} has $O(\hat{t}^{(i)2})$-Lipschitz terms with respect to $\hat{t}^{(i)}$. Thus, in a Lipschitz sense, DCEM may be less sensitive to extreme propensity scores than causal effect estimators such as IPW.
    
\end{remark}

\begin{corollary}
For an example indexed by $i$, $Q(y^{(i)}) = 1$, and $J^{(i)}$ defined as in Eq.~\ref{eq:creg_def}, $R(J^{(i)})$ is monotonically non-decreasing in $\hat{t}^{(i)}$ on $(0, 1]$.
\end{corollary}

\begin{proof}
The proof is identical to that of Theorem 3.4, except we find that 
\begin{equation}
    \frac{\partial}{\partial\hat{t}^{(i)}}\; \hat{y}^{(i)}_{\text{OPT}}(Q(y^{(i)}), \hat{t}^{(i)}) \leq 0,
\end{equation}
instead of being strictly less than zero, from which the corollary follows.
\end{proof}

\begin{remark}
For intuition, we show a contour plot of $\hat{y}^{(i)}_{\text{OPT}}(Q(y^{(i)}), \hat{t}^{(i)})$ in Fig.~\ref{fig:optim_mstep}. We verify the result in CVXPY.
\end{remark}
\begin{figure}[t]
    \centering
    \includegraphics{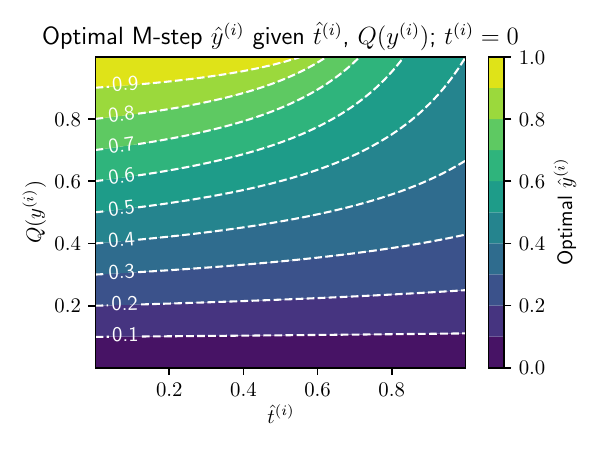}\vspace{-3mm}
    \caption{Contour plot of $\hat{y}^{(i)}_{\text{OPT}}(Q(y^{(i)}), \hat{t}^{(i)})$ with respect to $\hat{t}^{(i)}$ ($x$-axis) and $Q(y^{(i)})$ ($y$-axis). $\hat{y}^{(i)}_{\text{OPT}}(Q(y^{(i)}), \hat{t}^{(i)})$ scales with $Q(y^{(i)})$ but decreases in $\hat{t}^{(i)}$.}
    \label{fig:optim_mstep}
\end{figure}

\subsection{Proposition~\ref{prop:mstep_t1}}

\begin{proposition*}[Minimizer of M-step when $t^{(i)} = 1$]
Suppose that $t^{(i)} = 1$, let $Q(y^{(i)})\triangleq \mathbb{E}[y=1 \mid t, \tilde{y}, a, \*x]$, and let $\hat{y}^{(i)}$ be some estimate of $y^{(i)}$. Use $\mathcal{L}: [0, 1]^2 \to \mathbb{R}_+$ be shorthand for binary cross-entropy loss. Then, the minimization problem
\begin{equation}
    \underset{\hat{y}}{\min}\; \frac{1}{N} \sum_{i=1}^N \mathcal{L}(Q(y^{(i)}), \hat{y}^{(i)}) + Q(y^{(i)})\mathcal{L}(\tilde{y}^{(i)}, \hat{y}^{(i)} \hat{t}^{(i)})
\end{equation}
admits the solution $\hat{y}^{(i)} = y^{(i)}$ for all $i \in \{1, \dots, N\}$.
\end{proposition*}

\begin{proof}
We briefly verify the convexity of the objective, which follows from the convexity of binary cross-entropy loss and the closure of
convexity under addition and positive scalar multiplication ($Q(y^{(i)}) \geq 0$). Thus, any minimizer of the objective is unique.

We proceed by cases. First, suppose that $y^{(i)} = 1$. Substituting the definition of $Q(y^{(i)})$, and using the fact that $t^{(i)} = 1 \implies y^{(i)} = \tilde{y}^{(i)})$, the objective function for a single example becomes
\begin{equation}
    \mathcal{L}(1, \hat{y}^{(i)}) + \mathcal{L}(1, \hat{y}^{(i)} \hat{t}^{(i)}),
\end{equation}
which, by inspection, is maximized for $\hat{y}^{(i)} = 1$. Similarly, for $y^{(i)} = 0$, the objective function for a single example is 
\begin{equation}
    \mathcal{L}(0, \hat{y}^{(i)}) 
\end{equation}
which reduces to binary cross-entropy loss, and $\hat{y}^{(i)} = 0$ minimizes the objective. Combining the two cases, the minimizer of the M-step objective when $t^{(i)} = 1$ is $\hat{y}^{(i)} = y^{(i)}$ as desired.
\end{proof}

\subsection{Causal identifiability}
\label{app:ident}

For completeness, we provide the derivation of the causal identifiability results, though it follows directly from existing results~\cite{imbens2015causal}.

\begin{proposition*}
Suppose that conditional exchangeability, or $\tilde{Y}(t) \perp\!\!\!\perp T \mid X$, holds. Then $\mathbb{E}[Y \mid X] = \mathbb{E}[\tilde{Y} \mid X, do(T=1)]$, which is identifiable as $\mathbb{E}[\tilde{Y} \mid X, T=1]$.
\end{proposition*}

\begin{proof}
    We can write 
    \begin{align}
        \mathbb{E}[Y \mid X] = \mathbb{E}[Y \mid X, T=1] =\mathbb{E}[\tilde{Y} \mid X, T=1]  
        = \mathbb{E}[\tilde{Y} \mid X, do(T=1)]
    \end{align}
    where the first equality is due to $Y \perp\!\!\!\perp T \mid X$, the second equality results from $T = 1 \implies Y = \tilde{Y}$, and the final equality applies conditional exchangeability. Since $\mathbb{E}[Y \mid X] = \mathbb{E}[\tilde{Y} \mid X, do(T=1)] = \mathbb{E}[\tilde{Y} \mid X, T=1]$, the theorem follows.
\end{proof}

\subsection{Lemmas used}

Below are the lemmas and proofs referenced in the preceding theorem and proposition proofs.

\begin{lemma}
Define $B(Q(y^{(i)}), \hat{t}^{(i)}) \triangleq 2Q(y^{(i)})\hat{t}^{(i)} + 1$ and $D(Q(y^{(i)}), \hat{t}^{(i)}) \triangleq (-B(Q(y^{(i)}), \hat{t}^{(i)})) ^2 - 4Q(y^{(i)})(Q(y^{(i)})\hat{t}^{(i)} + \hat{t}^{(i)})$ on $Q(y^{(i)}) \in [0, 1]$ and $\hat{t}^{(i)} \in (0, 1]$. Then, $D(Q(y^{(i)}), \hat{t}^{(i)}) \geq 0$.\label{lem:d_pos}
\end{lemma}

\begin{proof}
Choose any $Q(y^{(i)}) \in [0, 1]$ and $\hat{t}^{(i)} \in (0, 1]$.
We can write:
\begin{align}
    &D(Q(y^{(i)}, \hat{t}^{(i)}) \geq 0 \\
    \iff &B(Q(y^{(i)}), \hat{t}^{(i)})^2 \geq 4Q(y^{(i)})(Q(y^{(i)})\hat{t}^{(i)} + \hat{t}^{(i)})\\
    \iff &4Q(y^{(i)})^2\hat{t}^{(i)2} + 4Q(y^{(i)})\hat{t}^{(i)}+ 1\geq 4Q(y^{(i)})(Q(y^{(i)})\hat{t}^{(i)} + \hat{t}^{(i)})\\
    \iff &4Q(y^{(i)})^2\hat{t}^{(i)2} - 4Q(y^{(i)})^2\hat{t}^{(i)}+ 1\geq 0.
\end{align}
The final LHS is convex (by inspection) in $\hat{t}^{(i)}$, such that 
\begin{align}
    \underset{\hat{t}^{(i)}}{\min} \; 4Q(y^{(i)})^2\hat{t}^{(i)2} - 4Q(y^{(i)})^2\hat{t}^{(i)}+ 1 = 1 - Q(y^{(i)}) \geq 0 
\end{align}
where the minimum is attained at $\hat{t}^{(i)} = \frac{1}{2}$, and concave (by inspection) in $Q(y^{(i)})$, such that it suffices to evaluate the final LHS at $Q(y^{(i)}) \in \{0, 1\}$:\footnote{Recall that, for a concave function $f$, $f(\alpha x + (1-\alpha) y) \geq \alpha f(x) + (1- \alpha) f(y)$ for $\alpha \in [0, 1]$ with equality for $\alpha= 0$ or 1.  Thus, via the extreme value theorem, the minimum of $f$ on $[x, y]$ is achieved at $x$ or $y$.}
\begin{align}
    \left.4Q(y^{(i)})^2\hat{t}^{(i)2} - 4Q(y^{(i)})^2\hat{t}^{(i)}+ 1 \right|_{Q(y^{(i)}) = 0} &= 1 \geq 0\\
    \left.4Q(y^{(i)})^2\hat{t}^{(i)2} - 4Q(y^{(i)})^2\hat{t}^{(i)}+ 1 \right|_{Q(y^{(i)}) = 1} &= 4(\hat{t}^{(i)2} - \hat{t}^{(i)}) + 1 \geq 4\cdot\left(-\frac{1}{4}\right) + 1 = 0,
\end{align}
such that for all other $Q(y^{(i)}) \in (0, 1)$, $4Q(y^{(i)})^2\hat{t}^{(i)2} - 4Q(y^{(i)})^2\hat{t}^{(i)}+ 1 \geq 0$
as needed.
    
\end{proof}

\begin{lemma}
Define $B(Q(y^{(i)}), \hat{t}^{(i)})$ and $D(Q(y^{(i)}, \hat{t}^{(i)})$ as in Lemma~\ref{lem:d_pos}. Then, for $Q(y^{(i)}) \in [0, 1]$ and $\hat{t}^{(i)} \in (0, 1]$, 
    \begin{equation}
    \frac{B(Q(y^{(i)}), \hat{t}^{(i)}) + \sqrt{D(Q(y^{(i)}, \hat{t}^{(i)}) }}{2(\hat{t}^{(i)} + Q(y^{(i)})\hat{t}^{(i)}) } \geq 1.
\end{equation}\label{lem:oob}
\end{lemma}
\begin{proof}
Choose any $Q(y^{(i)}) \in [0, 1]$ and $\hat{t}^{(i)} \in (0, 1]$.
First, we rewrite
\begin{align}
    & \frac{B(Q(y^{(i)}), \hat{t}^{(i)}) + \sqrt{D(Q(y^{(i)}, \hat{t}^{(i)}) }}{2(\hat{t}^{(i)} + Q(y^{(i)})\hat{t}^{(i)}) } \geq 1\\
    \iff & 2Q(y^{(i)})\hat{t}^{(i)} + 1 + \sqrt{D(Q(y^{(i)}, \hat{t}^{(i)}) } \geq 2(\hat{t}^{(i)} + Q(y^{(i)})\hat{t}^{(i)}) \\
    \iff &\sqrt{D(Q(y^{(i)}, \hat{t}^{(i)})} \geq 2\hat{t}^{(i)} - 1.
\end{align}
For $\hat{t}^{(i)} \in (0, \frac{1}{2})$, Lemma~\ref{lem:d_pos} yields the desired conclusion. For $\hat{t}^{(i)}  \in [\frac{1}{2}, 1]$, we can write
\begin{align}
    &\sqrt{D(Q(y^{(i)}, \hat{t}^{(i)})} \geq 2\hat{t}^{(i)} - 1 \\
    \iff & 4Q(y^{(i)2})\hat{t}^{(i)2} - 4Q(y^{(i)}2)\hat{t}^{(i)} + 1 \geq 4\hat{t}^{(i)2} - 4\hat{t}^{(i)} + 1\\
    \iff & Q(y^{(i)2})(\hat{t}^{(i)} - 1) \geq \hat{t}^{(i)} - 1\\
    \iff & Q(y^{(i)2}) \leq 1
\end{align}
which all $Q(y^{(i)2}) \in [0, 1]$ satisfy. This completes the proof.
\end{proof}

\begin{lemma}
Define $B(Q(y^{(i)}), \hat{t}^{(i)})$ and $D(Q(y^{(i)}, \hat{t}^{(i)})$ as in Lemma~\ref{lem:d_pos}. Then, for $Q(y^{(i)}) \in [0, 1]$ and $\hat{t}^{(i)} \in (0, 1]$, 
\begin{equation}
    0 \leq \frac{B(Q(y^{(i)}), \hat{t}^{(i)}) - \sqrt{D(Q(y^{(i)}, \hat{t}^{(i)}) }}{2(\hat{t}^{(i)} + Q(y^{(i)})\hat{t}^{(i)}) } \leq 1.
\end{equation}
    \label{lem:in_bounds}
\end{lemma}

\begin{proof}
Choose any $Q(y^{(i)}) \in [0, 1]$ and $\hat{t}^{(i)} \in (0, 1]$.
Equivalently, we can show
\begin{equation}
    0 \leq B(Q(y^{(i)}), \hat{t}^{(i)}) - \sqrt{D(Q(y^{(i)}, \hat{t}^{(i)}) } \leq 2(\hat{t}^{(i)} + Q(y^{(i)})\hat{t}^{(i)}).
\end{equation}
For the first inequality, note that
\begin{align}
     \sqrt{D(Q(y^{(i)}, \hat{t}^{(i)}) } &= \sqrt{(-B(Q(y^{(i)}), \hat{t}^{(i)}))^2 - 4Q(y^{(i)})(Q(y^{(i)})\hat{t}^{(i)} + \hat{t}^{(i)}) } \leq \sqrt{(-B(Q(y^{(i)}), \hat{t}^{(i)})) ^2}\\
     &= |B(Q(y^{(i)}), \hat{t}^{(i)}))| = B(Q(y^{(i)}), \hat{t}^{(i)}))
\end{align}
which rearranges to $0 \leq B(Q(y^{(i)}), \hat{t}^{(i)}) - \sqrt{D(Q(y^{(i)}, \hat{t}^{(i)}) }$ as desired. For the second inequality, note that
\begin{align}
    &B(Q(y^{(i)}), \hat{t}^{(i)}) - \sqrt{D(Q(y^{(i)}, \hat{t}^{(i)}) } \leq 2(\hat{t}^{(i)} + Q(y^{(i)})\hat{t}^{(i)})\\
    \iff &1 - \sqrt{D(Q(y^{(i)}, \hat{t}^{(i)}) } \leq 2\hat{t}^{(i)}  \\
    \iff & 1 - 2\hat{t}^{(i)} \leq \sqrt{D(Q(y^{(i)}, \hat{t}^{(i)})}.
\end{align}
For $\hat{t}^{(i)} \in [\frac{1}{2}, 1]$, Lemma~\ref{lem:d_pos} yields the desired conclusion. For $\hat{t}^{(i)} \in (0, \frac{1}{2})$, the proof proceeds similarly to Lemma~\ref{lem:oob}:
\begin{align}
     &\sqrt{D(Q(y^{(i)}, \hat{t}^{(i)})} \geq 1 - 2\hat{t}^{(i)}  \\
    \iff & 4Q(y^{(i)2})\hat{t}^{(i)2} - 4Q(y^{(i)})^2\hat{t}^{(i)} + 1 \geq 4\hat{t}^{(i)2} - 4\hat{t}^{(i)} + 1\\
    \iff & Q(y^{(i)2})(\hat{t}^{(i)} - 1) \geq \hat{t}^{(i)} - 1\\
    \iff & Q(y^{(i)2}) \leq 1
\end{align}
which all $Q(y^{(i)2}) \in [0, 1]$ satisfy. This completes the proof.
\end{proof}

\begin{lemma}
Define $B(Q(y^{(i)}), \hat{t}^{(i)})$ and $D(Q(y^{(i)}, \hat{t}^{(i)})$ as in Lemma~\ref{lem:d_pos}, and define $\hat{y}^{(i)}_{\text{OPT}}(Q(y^{(i)}), \hat{t}^{(i)})$ as in Definition~\ref{eq:creg_def}. Then, 
\begin{equation}
    \text{sign}\left(\frac{\partial}{\partial \hat{t}^{(i)})}\;\hat{y}^{(i)}_{\text{OPT}}(Q(y^{(i)}), \hat{t}^{(i)})\right) = \text{sign}\left(1 - 2\hat{t}^{(i)}Q(y^{(i)})^2 - \sqrt{D(Q(y^{(i)}), \hat{t}^{(i)})})\right)
\end{equation}
where $\text{sign}(x): \mathbb{R} \to \{-1, 0, 1\}$ is the sign function:
\begin{equation}
    \text{sign}(x) \triangleq\begin{cases}
    -1 & x < 0\\
    0 & x = 0\\
    1 & x > 0
    \end{cases}.
\end{equation}

\label{lem:sign_preserve}
\end{lemma}
\begin{proof}
The proof is largely algebraic simplification based on sign-preserving operations. 
Taking derivatives: 
\begin{equation}
    \frac{\partial}{\partial\hat{t}^{(i)}}\; \hat{y}^{(i)}_{\text{OPT}}(Q(y^{(i)}), \hat{t}^{(i)}) = \frac{Q(y^{(i)}) - \frac{2\hat{t}^{(i)}Q(y^{(i)})^2 - Q(y^{(i)})^2}{\sqrt{D(Q(y^{(i)}), \hat{t}^{(i)})}}}{\hat{t}^{(i)} + Q(y^{(i)})\hat{t}^{(i)}} - \frac{\left(B(Q(y^{(i)}), \hat{t}^{(i)}) - \sqrt{D(Q(y^{(i)}), \hat{t}^{(i)})})\right)(Q(y^{(i)}) + 1)}{2(\hat{t}^{(i)} + Q(y^{(i)})\hat{t}^{(i)}) ^2}
\end{equation}
via the quotient rule of derivatives and cancelling terms. We can apply sign-preserving operations, namely, positive scalar multiplication, canceling additive zeroes, and commuting additive terms, as follows:
\begin{align}
    \frac{\partial}{\partial\hat{t}^{(i)}}\; \hat{y}^{(i)}_{\text{OPT}}(Q(y^{(i)}), \hat{t}^{(i)}) &\propto (\hat{t}^{(i)} + Q(y^{(i)})\hat{t}^{(i)})\left(2Q(y^{(i)}) - \frac{4\hat{t}^{(i)} Q(y^{(i)})^2 - 2Q(y^{(i)})^2}{\sqrt{D(Q(y^{(i)}), \hat{t}^{(i)})}}\right) \nonumber \\
    &\quad - \left(B(Q(y^{(i)}), \hat{t}^{(i)}) - \sqrt{D(Q(y^{(i)}), \hat{t}^{(i)})})\right)(Q(y^{(i)}) + 1)\\
    &\propto \hat{t}^{(i)}\left(2Q(y^{(i)}) - \frac{4\hat{t}^{(i)} Q(y^{(i)})^2 - 2Q(y^{(i)})^2}{\sqrt{D(Q(y^{(i)}), \hat{t}^{(i)})}}\right) - B(Q(y^{(i)}), \hat{t}^{(i)}) + \sqrt{D(Q(y^{(i)}), \hat{t}^{(i)})}) \\
    &=  \left(\frac{2\hat{t}^{(i)}Q(y^{(i)})^2 - 4\hat{t}^{(i)2} Q(y^{(i)})^2 }{\sqrt{D(Q(y^{(i)}), \hat{t}^{(i)})}}\right) - 1 + \sqrt{D(Q(y^{(i)}), \hat{t}^{(i)})})\\
    &\propto 2\hat{t}^{(i)}Q(y^{(i)})^2 - 4\hat{t}^{(i)2}  -  \sqrt{D(Q(y^{(i)}), \hat{t}^{(i)})}) + D(Q(y^{(i)}), \hat{t}^{(i)})\\
    &= 1 - 2\hat{t}^{(i)}Q(y^{(i)})^2 - \sqrt{D(Q(y^{(i)}), \hat{t}^{(i)})})
\end{align}
which completes the proof.
\end{proof}

\subsection{Definition~\ref{def:causalreg}: causal regularization strength}

We expand on our definition of causal regularization strength here. Conventionally, regularization strength is operationalized in terms of a regularization parameter $\lambda \in \mathbb{R}_+$, given a loss  $\ell(\theta)$ and a regularizer $r(\theta)$ (\emph{e.g.}, $r(\theta) = \lVert \theta \rVert_2^2$ for some objective  $J(\theta)$ of the form
\begin{equation}
    J(\theta) \triangleq \ell(\theta) + \lambda r(\theta).\label{eq:rrm}
\end{equation}
Eq.~\ref{eq:rrm} is an instance of \emph{regularized risk minimization}~\cite{shalev2014understanding}. It is also identical to \emph{linear scalarization}, a technique for characterizing tradeoffs in multi-objective optimization.
The equivalence between regularized risk minimization and linear scalarization simply reflects that regularization can impose tradeoffs in optimizing $J(\theta)$ between minimizing $\ell(\theta)$ versus $r(\theta)$. Regularized risk minimization treats $r(\theta)$ as a ``penalty'' term, while linear scalarization treats $r(\theta)$ as merely another objective. As $\lambda$ increases, the tradeoff increasingly favors $r(\theta)$, and vice versa. 

Now, consider our M-step objective example-wise:
\begin{equation}
    \mathcal{L}(Q(y^{(i)}), \hat{y}^{(i)}) + Q(y^{(i)}) \mathcal{L}(\tilde{y}^{(i)}, \hat{y}^{(i)}\hat{t}^{(i)}).
\end{equation}
The M-step objective can similarly be interpreted as variation of a regularized risk minimziation problem, where $\ell(\theta) = \mathcal{L}(Q(y^{(i)}, \hat{y}^{(i)})$, and $Q(y^{(i)}) \mathcal{L}(\tilde{y}^{(i)}, \hat{y}^{(i)}, \hat{t}^{(i)}) = r(\theta), \lambda = 1$. However, $\hat{t}^{(i)}$ is a constant that can affect regularization strength, but is not a multiplier like $\lambda$.
The purpose of our result is to characterize the impact of $\hat{t}^{(i)}$ on the tradeoff between the two terms of the M-step objective.

Thus, motivated by the tradeoff/multi-objective perspective of regularization, we define regularization strength in terms of a tradeoff between optimizing $\mathcal{L}(Q(y^{(i)}, \hat{y}^{(i)})$ and $Q(y^{(i)}) \mathcal{L}(\tilde{y}^{(i)}, \hat{y}^{(i)} \hat{t}^{(i)})$. We observe that
\begin{align}
    Q(y^{(i)}) &= \underset{\hat{y}^{(i)}}{\arg\min}\;\mathcal{L}(Q(y^{(i)}), \hat{y}^{(i)}) 
\end{align}
and define causal regularization strength as the absolute distance between $Q(y^{(i)})$, the minimizer of $\mathcal{L}(Q(y^{(i)}), \hat{y}^{(i)})$, and the optimum of the example-wise M-step objective. 

\begin{definition}[Causal regularization strength]
Given an example indexed by $i$, and a finite loss function $\mathcal{L}: [0, 1]^2 \to \mathbb{R}$ convex in $\hat{y}^{(i)}$ on [0, 1] for all $i$, define 
\begin{equation}
     \hat{y}^{(i)}_{\text{OPT}}(Q(y^{(i)}), \hat{t}^{(i)}) \triangleq \underset{\hat{y}^{(i)}}{\arg\min}\; J^{(i)}(\hat{y}^{(i)}, \dots) \triangleq \underset{\hat{y}^{(i)}}{\arg\min}\;\mathcal{L}(Q(y^{(i)}), \hat{y}^{(i)}) + Q(y^{(i)}) \mathcal{L}(\tilde{y}^{(i)}, \hat{y}^{(i)}\hat{t}^{(i)}).\label{eq:creg_def}
\end{equation}
The causal regularization strength of objective $J^{(i)}$ is defined as $R(J^{(i)}) = |Q({y}^{(i)}) - \hat{y}^{(i)}_{\text{OPT}}(Q(y^{(i)}), \hat{t}^{(i)})|$.
\end{definition}

Intuitively, we define causal regularization strength in terms of the absolute distance between the optimum of each term of the M-step objective, which captures some notion of a tradeoff between the two terms. 
Note that this definition does not relate to convergence to $\hat{y}^{(i)}_{\text{OPT}}(Q(y^{(i)}), \hat{t}^{(i)})$; we are largely interested in how much the solution to $\min\; \mathcal{L}(Q(y^{(i)}), \hat{y}^{(i)})$ shifts after adding the causal regularization term.

\section{Additional experimental setup}
\label{app:setup}

For both settings, we set random seeds to 42 to facilitate reproducibility.

\subsection{Additional details for fully synthetic dataset}
\label{app:sim_details}
We choose $s_Y$ as follows:
\begin{align*}
    S_Y(x) &= (s_{Y1} \circ s_{Y2})(\*x) \\
    s_{Y1}(\*x) &= x_1 - \frac{1}{4}\sin(8 \pi x_0 + \psi) \quad  s_{Y2}(\*x) = \mathbf{R}_{\pi / 6}\*x + 0.5
\end{align*}
where $\psi$ is a simulation parameter, and $\mathbf{R}_{\pi / 6}$ is a 2D rotation matrix. Intuitively, $s_Y(\*x)$ rotates and translates $\*x$, then applies a sinewave-based function that yields a similarly rotated, sinewave-shaped decision boundary.

We choose $s_T$ as follows: 
\begin{equation*}
    s_T(\*x^{(i)}, a^{(i)}) = \mathbf{1}^\top \mathbf{x}^{(i)} - \tau_{a^{(i)}}
\end{equation*}
where $\tau_a$ is a simulation parameter.
For demonstration, we set $c_y$ such that $P(Y = 1) = 0.25$ to allow for sufficiently-sized performance gaps across groups to emerge.\footnote{Empirically, at extreme values of $P(Y = 1)$, we found artificially small performance gaps. This is because model errors tend to concentrate near the true decision boundary, which lies in the tails of the covariate distributions defining $X \mid A = a$. In those tail regions, the difference between the densities $X \mid A = a$ across values of $A$ is smaller in our two-Gaussian simulation design.} As a sensitivity analysis, we also replicate all experiments on fully synthetic data across $\psi \in [0, \pi / 6, \pi / 3, \dots, 11\pi / 6]$, representing the ``phase'' of the decision boundary.

\paragraph{Computing simulation parameters.} We discuss how we find simulation parameters for each value of $q_y, q_t$, and $k$. Given $q_y$ and $P(A = 0) = P(A = 1) = 0.5$, we have: 
\begin{align*}
   q_y &= \frac{P(Y = 1 \mid A = 0)}{P(Y = 1 \mid A = 1)} \\
   P(Y = 1) = \frac{1}{4} &= \frac{P(Y = 1 \mid A = 0) + P(Y = 1 \mid A = 1)}{2}
\end{align*}
which yields, by substitution,
\begin{align*}
    \frac{q_y}{2(q_y + 1)} &= P(Y = 1 \mid A = 0) \\
    \frac{1}{2(q_y + 1)} &= P(Y = 1 \mid A = 1),
\end{align*}
from which we use a binary search algorithm (bisection) evaluated using simulated versions of $X \mid A = a$ with the current estimate of the mean $\mu_a$ to solve for the requisite values of $\mu_a$. Given values of $\mu_a$, we can then solve for $\tau_a$ using $q_t$ and $k$ identically:
\begin{align*}
   q_t &= \frac{P(T = 1 \mid A = 0)}{P(T = 1 \mid A = 1)} \\
   P(T = 1) = k P(Y = 1) = \frac{k}{4} &= \frac{P(T = 1 \mid A = 0) + P(T = 1 \mid A = 1)}{2}
\end{align*}
which yields, again by substitution,
\begin{align*}
    \frac{q_tk}{2(q_t+ 1)} &= P(T = 1 \mid A = 0) \\
    \frac{k}{2(q_t + 1)} &= P(T = 1 \mid A = 1),
\end{align*}
and we can again use binary search to solve for $\tau_a$. 

\subsection{Additional details for pseudo-synthetic sepsis risk-stratification task}
\label{app:sepsis_details}

\begin{figure}[t]
    \centering
    \includegraphics[width=\linewidth]{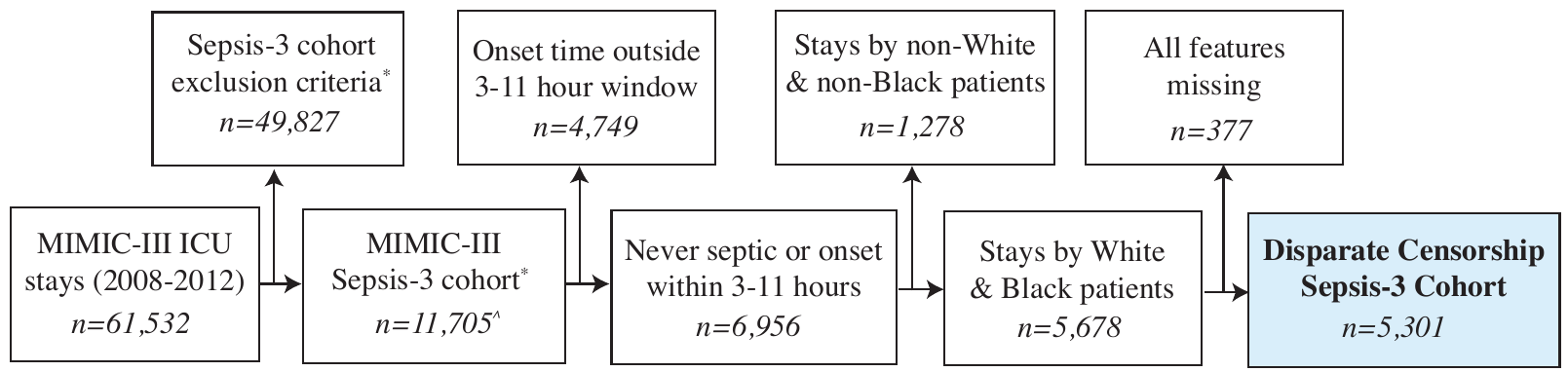}\vspace{-7mm}
    \caption{Cohort diagram for our Sepsis-3 cohort ($N=5301$). *: Further details are provided in~\cite{johnson2018comparative}. \textasciicircum: Our cohort size differs slightly from that reported in~\cite{johnson2018comparative} due to an apparent pre-processing error in defining Sepsis-3~\cite{singer2016third}; our reported cohort size applies the relevant correction.}
    \label{fig:enter-label}
\end{figure}

\paragraph{Cohort description.} Our cohort follows from the MIMIC-III Sepsis-3 cohort~\cite{johnson2018comparative}. Their cohort exclusion criteria are publicly available.\footnote{\url{https://github.com/alistairewj/sepsis3-mimic}} We corrected an apparent Sepsis-3 definition bug that erroneously labeled individuals with suspicion of infection if they received a blood culture at any time after an antibiotic before re-running their pipeline. In contrast, the Sepsis-3~\cite{singer2016third} definition requires the blood culture to occur \emph{within 24 hours} of the antibiotic time for suspicion of infection.\footnote{There are multiple ``paths'' for meeting the criteria for suspicion of infection; for a full enumeration, see Table 2 of~\cite{singer2016third}.}
In practice, this stricter condition affects $<1\%$ of rows in their original cohort: their cohort size is $N=11,791$, while ours is $N=11,705$.

\paragraph{Feature extraction.} Following the Risk of Sepsis model~\citep{delahanty2019development}, we extract the following 13 summary statistics over the initial 3-hour observation period: 
\begin{enumerate}
    \item Maximum lactic acid measurement, 
    \item first shock index times age (years),
    \item last shock index times age (years), 
    \item maximum white blood cell count, 
    \item change in lactic acid (last - first), 
    \item maximum neutrophil count, 
    \item maximum blood glucose, 
    \item maximum blood urea nitrogen, 
    \item maximum respiratory rate,
    \item last albumin measurement,
    \item minimum systolic blood pressure,
    \item maximum creatinine, and 
    \item maximum body temperature (Fahrenheit).
\end{enumerate}
The shock index is defined as the ratio of heart rate (beats per minute) and systolic blood pressure. Missing features are replaced with -9999 following the original manuscript.

\paragraph{Testing decision boundary.} We define $s_T$ as 
\begin{equation}
    \beta \cdot \frac{RR_{\max} - 22}{RR_\sigma}  + (1 - \beta) \cdot \frac{SBP_{\min} - 100}{SBP_\sigma}  - \tau_a,
\end{equation}
where $RR_{\max}$ and $SBP_{\min}$ are maximum respiratory rate and minimum systolic blood pressure, respectively, and $RR_\sigma, SBP_\sigma$ are their corresponding standard deviations on the training split ($RR_\sigma = 9.8, SBP_\sigma = 21.8$). The parameter $\beta$ allows us to examine different testing decisions. Thus, we replicate all experiments over $\beta \in \{0, 0.1, \dots, 1\}$.

\section{Hyperparameters \& additional model details}
\label{app:hyperparam}

All hyperparameters were selected using a validation set of examples. Hyperparameters for the sepsis simulation task were chosen such that all approaches attained similar performance when using $y$. We reimplement all existing methods following the original papers, using the code repository as a reference if applicable. We set random seeds to 42 for all models (used for initialization), unless otherwise noted.

\subsection{Default hyperparameters}

\paragraph{Fully synthetic.} 
All models use a two-layer neural network with layer sizes $(64, 64)$, trained for 1000 epochs via Adam~\cite{kingma2014adam} with learning rate $10^{-3}$ and no weight decay unless specified. 
EM approaches are trained up to 50 iterations with early stopping on validation loss (patience 3) and warm starts (initialized with solution from the previous iteration).

\paragraph{Sepsis classification.} 
All predictors are three-layer neural networks with sizes $(128, 128, 16)$ trained for 10000 epochs using Adam with learning rate $10^{-5}$ and weight decay $10^{-3}$. 
The DCEM propensity model ($g_\zeta$) is trained for 20000 epochs with learning rate $10^{-5}$ and early stopping with patience 1000, and the DCEM model ($f_\theta$) uses learning rate and weight decay $5 \times 10^{-7}$ and $10^{-6}$, respectively.
\subsection{Simulation study}

\paragraph{Peer loss \& group peer loss:} Both peer loss methods depend on a hyperparameter $\alpha$, for which the optimal value depends on $y$. To show the peer loss methods in the best light, we manually calculate the optimal value for usage in training.

\paragraph{ITE corrected model (DragonNet):} Our estimand of interest is the conditional average treatment effect (CATE) of the sensitive attribute $A$ on testing $T$, which is identifiable via
\begin{equation}
    \text{CATE}_{A \to T}(\*x) \triangleq \mathbb{E}[T(1) \mid X=\*x] - \mathbb{E}[T(0) \mid X=\*x] = \mathbb{E}[T \mid X=\*x, A=1 ] - \mathbb{E}[T \mid X=\*x, A=0 ] 
\end{equation}
under assumptions of consistency ($T(a) = T$) and conditional exchangeability ($T(a) \perp\!\!\!\perp A \mid X$). We then apply the CATE as a correction factor to the default model:
\begin{equation}
    \hat{y} \triangleq \hat{\tilde{y}}  - \text{CATE}_{A \to T}(\*x);  
\end{equation}
\emph{i.e.}, counterbalancing disparate censorship by ``subracting out'' the labeling bias. Note that this is an alternative to the counterbalancing approach of DCEM.
We train and conduct inference with targeted regularization. 

\paragraph{Truncated LQ:} We searched across $k \in \{0.1, 0.2, \dots, 1\}$ and $q \in \{0.1, 0.2, \dots, 1\}$ (using the notation of the original paper), using $k = 0.1, q = 0.1$ for the final results.

\paragraph{SELF:} 
We were unable to obtain convergence with Adam, so we used SGD with learning rate 0.01, momentum 0.9, noise parameter 0.05 (for input augmentation), consistency regularization parameter 1, and weight decay $1\times 10^{-6}$ as used for one of the experiments in the original paper. Weight decay was selected from $\{0, 10^{-6}, 10^{-5}, 10^{-4}, 10^{-3}, 10^{-2}\}$. The ensembling/mean teacher parameters were chosen from $\{0.9, 0.99, 0.999\}$. The noise was chosen from $\{0, 0.005, 0.01, 0.05, 0.1\}$. The regularization parameter was chosen from $\{1, 5, 10, 50\}$. SELF proceeds for a maximum of 50 iterations with patience 1 with respect to validation AUC. We set ensembling momentum to 0.9 and the mean teacher moving average parameter to 0.9. 
To show SELF in the best light, we prevented SELF from filtering tested positive individuals.

\paragraph{DivideMix:} We use 20 warmup epochs, with $\alpha=4$ as the Beta parameter for the MixMatch step, $T = 0.5$, $\lambda_u = 50$, $\lambda_r = 1$, and $\tau = 0.5$, and weight decay $5 \times 10^{-4}$.
We also experimented with preventing DivideMix from filtering tested positive individuals, but DivideMix was unstable in both settings. Ultimately, we did not prevent DivideMix from filtering tested positive individuals.

\paragraph{EM-based methods (SAREM, DCEM):} We tested SAREM and DCEM with and without the usage of warm starts in the M-step.

\subsection{Sepsis risk-stratification}

For all baselines, the setup matches the fully synthetic setting except as specified below.

\paragraph{DCEM:}
The learning rates under consideration were $[10^{-7}, 5 \times 10^{-7}, 10^{-6}, 5 \times 10^{-6}, 10^{-5}, 10^{-4}, 10^{-3}]$. The weight decay was selected from $[0, 10^{-6}, 2 \times 10^{-6}]$.

\paragraph{SELF:} For the sepsis classification experiments, weused SGD and set the learning rate to $10^{-8}$, the highest learning rate tested that did not result in NaN loss. We tested learning rates of the form $\{10^{-d}, 5 \times 10^{-d}\}$ for $d \in \{2, 3, 4, 5, 6, 7, 8\}$.

\section{Additional empirical results and discussion}

\subsection{Full results}
\label{app:full_results}

Here, we report empirical results for all baselines and settings.
Due to the large number of empirical settings tested (simulation: 224, sepsis classification: 45), we include a representative subset of the figures, and report the raw numbers used for these results \emph{and} results not shown in the Appendix via CSV files in the code appendix.

For the simulated task, we show results for $k \in [1/3, 3]$, $q_t \in [0.5, 2]$, and $q_y = 0.5$. Empirically, changing $q_y$ did not affect the general trends, but amplified/dampened the scale. Increasing $q_t$ beyond the selected range has similar impacts. For lower values of $k$, all methods perform poorly. 
For the sepsis classification task, we show results for $k \in [1/3, 5]$ and $q_t = 1.5$. 

\paragraph{Summary of results.} We summarize when our method (DCEM) empirically performed the best, when it performed similarly to baselines, and when it underperformed baselines. 

\paragraph{DCEM is best where...}
\begin{itemize}
    \item (Both metrics) The higher-prevalance group is undertested ($q_y < 1$ but $q_t > 1$) \emph{and}
    \item (Both metrics) testing rates are sufficiently high ($k \geq 0.5$).
\end{itemize}
\paragraph{DCEM is similar to baselines when...}
\begin{itemize}
    \item (Both metrics) Testing rates are moderately low ($1/3 \leq k \leq 1/2$), or sufficiently high that it is easier to extrapolate from labeled data ($k \geq 3$).
\end{itemize}

\paragraph{DCEM underperforms baselines when...}
\begin{itemize}
    \item (ROC gap only) when testing rates are low ($k \leq 1/2$) \textit{and} 
    \item (ROC gap only) the testing disparity aligns with the prevalence disparity (\emph{e.g.}, $q_t$ and $q_y < 1$ such that learning to predict $\tilde{y}$ preserves ranking in $y$), \emph{or} 
    \item (both metrics) testing rates are extremely low ($k < 1/3$).
\end{itemize}

The strongest alternatives to DCEM in our experiments were SELF (both datasets, bias mitigation), DragonNet (sepsis only, both metrics), and the tested-only model (simulation only, discriminative performance).

\paragraph{Index of figures.} We provide here a list of all result figures in the Appendix, indexed by problem parameters $k$ (overall testing rate multiplier), $q_t$ (testing disparity), and $q_y$ (prevalence disparity; simulation only).

\paragraph{Fully-synthetic data}
\begin{itemize}
    \item Figure~\ref{fig:qy0.5_k0.3_qt0.5}: $q_y = 1/2, k=1/3, q_t=1/2$
    \item Figure~\ref{fig:qy0.5_k0.3_qt1}: $q_y = 1/2, k=1/3, q_t=1$
    \item Figure~\ref{fig:qy0.5_k0.3_qt2}: $q_y = 1/2, k=1/3, q_t=2$
    \item Figure~\ref{fig:qy0.5_k0.5_qt0.5}: $q_y = 1/2, k=1/2, q_t=1/2$
    \item Figure~\ref{fig:qy0.5_k0.5_qt1}: $q_y = 1/2, k=1/2, q_t=1$
    \item Figure~\ref{fig:qy0.5_k0.5_qt2}: $q_y = 1/2, k=1/2, q_t=2$
    \item Figure~\ref{fig:qy0.5_k1_qt0.5}: $q_y = 1/2, k=1, q_t=1/2$
    \item Figure~\ref{fig:qy0.5_k1_qt1}: $q_y = 1/2, k=1, q_t=1$
    \item Figure~\ref{fig:qy0.5_k1_qt2}: $q_y = 1/2, k=1, q_t=2$
    \item Figure~\ref{fig:qy0.5_k2_qt0.5}: $q_y = 1/2, k=2, q_t=1/2$
    \item Figure~\ref{fig:qy0.5_k2_qt1}: $q_y = 1/2, k=2, q_t=1$
    \item Figure~\ref{fig:qy0.5_k2_qt2}: $q_y = 1/2, k=2, q_t=2$
    \item Figure~\ref{fig:qy0.5_k3_qt0.5}: $q_y = 1/2, k=3, q_t=1/2$
    \item Figure~\ref{fig:qy0.5_k3_qt1}: $q_y = 1/2, k=3, q_t=1$
    \item Figure~\ref{fig:qy0.5_k3_qt2}: $q_y = 1/2, k=3, q_t=2$
\end{itemize}

\paragraph{Sepsis classification}
\begin{itemize}
    \item Figure~\ref{fig:sepsis_k0.25}: $k=1/4, q_t=3/2$
    \item Figure~\ref{fig:sepsis_k0.3}: $k=1/3, q_t=3/2$
    \item Figure~\ref{fig:sepsis_k0.5}: $k=1/2, q_t=3/2$
    \item Figure~\ref{fig:sepsis_k1}: $k=1, q_t=3/2$
    \item Figure~\ref{fig:sepsis_k2}: $k=2, q_t=3/2$
    \item Figure~\ref{fig:sepsis_k3}: $k=3, q_t=3/2$
    \item Figure~\ref{fig:sepsis_k4}: $k=4, q_t=3/2$
    \item Figure~\ref{fig:sepsis_k5}: $k=5, q_t=3/2$
\end{itemize}

\subsection{DCEM ablation study}
\label{app:ablation}

To understand how DCEM design choices impact performance, we conduct an ablation study of repeated iterations and causal regularization:
\begin{itemize}
    \item \textbf{Imputation-only}: This approach trains a model on the tested-only (labeled) examples, imputes pseudo-labels for the remaining, then trains a model on both the pseudo-labeled and labeled data. This is equivalent to a single EM iteration without causal regularization.
    \item \textbf{No causal regularization}: This approach runs multiple EM iterations, but without causal regularization.
\end{itemize}

The results (Table~\ref{tab:ablation}) suggest that both repeated iterations and causal regularization are essential to the bias mitigation and discriminative capabilities of DCEM. 
The imputation-only approach fails due to low overlap between the tested vs. untested regions. Consequently, the imputed outcomes could be arbitrarily inaccurate.
If we keep imputing and retraining (without causal regularization), we recover a form of pseudo-labeling~\citep{lee2013pseudo}. The empirical improvement in performance suggests that repeated supervision from reliably labeled examples helps improve discriminative performance. However, this approach does not adjust for labeling bias (e.g., by using  $A$), and indeed the ROC gap does not improve.
Incorporating causal regularization recovers the DCEM M-step. Adding causal regularization guarantees that DCEM locally maximizes log-likelihood, and allows it to mitigate labeling bias by incorporating $A$ into a propensity score-like term (causal regularization; see Theorem~\ref{thm:causalreg}).

\begin{table}[t]
    \centering
    \caption{Sensitivity analysis of DCEM components with respect to AUC and ROC gap (min, max across $s_T$ in parentheses) for $q_y = 0.5, k=1, q_t=2$. Maximum (minimum) median AUC (ROC gap) in bold.}
    \begin{tabular}{c|c|c}
        Method & $\uparrow$ AUC & $\downarrow$ ROC gap \\
        \midrule
         Imputation-only & .676 [.644, .715] & .063 [.036, .086]\\
         No causal regularization & .767 [.733, .813] & .056 [.016, .086]\\
         DCEM (ours) & .\textbf{791 [.763, .820]} & \textbf{.031 [.019, .072]}  \\
         \bottomrule
    \end{tabular}

    \label{tab:ablation}
\end{table}

\begin{table}
    \centering
    \caption{Sensitivity analysis of causal effect estimators for estimating $P(Y \mid X)$ compared to DCEM with respect to  AUC and ROC gap (min, max across $s_T$ in parentheses) for $q_y = 0.5, k=1, q_t=2$. Maximum (minimum) median AUC (ROC gap) in bold.}
    \begin{tabular}{c|c|c}
        Method &  $\uparrow$ AUC & $\downarrow$ ROCGap \\
        \midrule
        Tested-only & .808 [.623, .876] & .052 [.020, .093] \\
        Tested-only + group &  .764 [.675, .863] & .078 [.025, .278] \\
        IPW & \textbf{.829 [.598, .874]}& .048 [.020, .104]\\
        DR-Learner & .643 [.558, .769]& .117 [.080, .216]\\
        DCEM (ours) & .791 [.763, .820] & \textbf{.031 [.019, .072]}\\
        \bottomrule
    \end{tabular}
    
    \label{tab:causal_sens}
\end{table}

\begin{figure}[t]
    \centering
    \includegraphics[width=\linewidth]{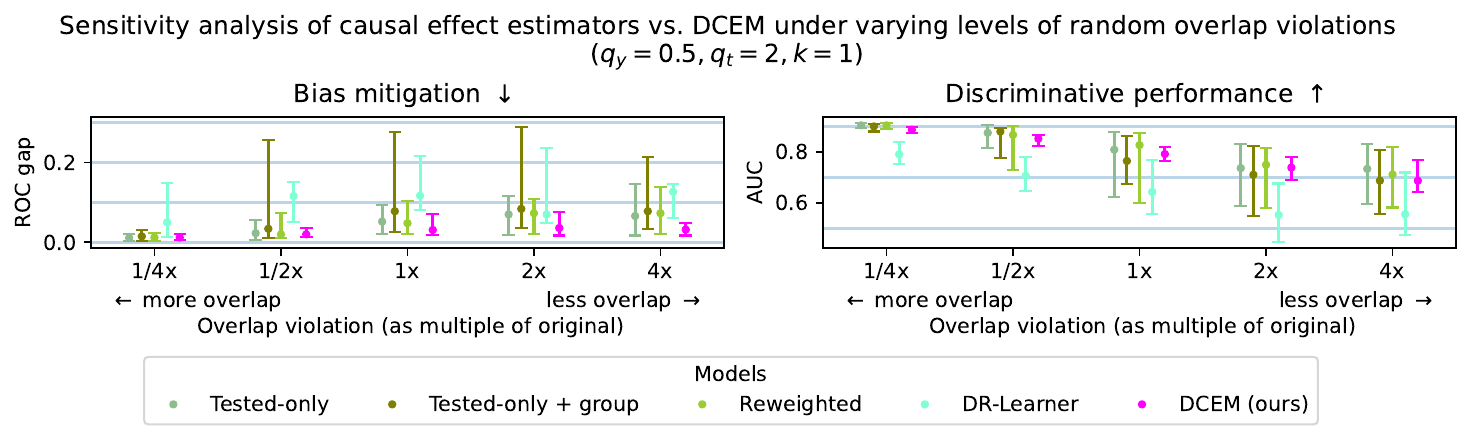}
    \caption{Sensitivity analysis of causal effect estimators with respect to AUC and ROC gap (min, max across $s_T$ in parentheses) for $q_y = 0.5, k=1, q_t=2$ across varying levels of overlap. Causally-motivated methods are shown in green, while DCEM is shown in magenta. Empirically, DCEM improves robustness to overlap violations.}
    \label{fig:overlap_sens}
\end{figure}

\subsection{Sensitivity analysis of causally-motivated approaches}
\label{app:overlap}

Here, we conduct a sensitivity analysis of causally-motivated approaches under disparate censorship. The causally-motivated approaches are theoretically consistent estimators of $P(Y\mid X)$, which we can interpret at the conditional average treatment effect of testing ($T$) on the observed outcome ($\tilde{Y}$; see Appendix~\ref{app:ident}). We examine the following causally-motivated estimators:
\begin{itemize}
    \item \textbf{Tested-only}: training models on tested individuals only, using $X$ as covariates,
    \item \textbf{Tested-only + group}: training models on tested individuals only, using $X$ and $A$ as covariates,
    \item \textbf{Inverse propensity weighting (IPW)}: an IPW-based~\citep{rosenbaum1983central} version of the tested-only approach, and
    \item \textbf{Doubly-robust estimator (DR-Learner)}: a doubly-robust estimator of $P(Y\mid X)$~\citep{kennedy2023towards}.
\end{itemize}

Models are evaluated for $q_y = 0.5, k=1, q_t=2$ (\emph{i.e.}, same setting as Fig.~\ref{fig:main-result}). 
Under disparate censorship, low overlap is common due to the ``sharpness'' of the testing boundary. To validate this hypothesis, we also evaluate causal effect estimators versus DCEM at varying levels of overlap ($1/4$x, $1/2$x, $1$x, $2$x, and $4$x of the original setting). Overlap is controlled by the coefficient inside the sigmoid for generating $t^{(i)}$ (\emph{i.e.}, 30 in the original experiments).\footnote{Recall that $t^{(i)}$ is generated as a Bernoulli random variable with parameters of the form $\sigma(ax + b)$.} 
For the DR-learner, we trimmed propensity scores (threshold: 0.05) to obtain estimates that were in [0, 1] (the possible values of $P(Y \mid X)$).

\paragraph{DCEM has better bias mitigation capabilities than causal approaches, and a tighter range of discriminative performance.} Table~\ref{tab:causal_sens} shows that, empirically, DCEM exhibits lower variance under overlap violations than causally-motivated approaches. Notably, DCEM achieves the lowest median ROC gap, and maintains competitive (but not necessarily best) median AUC. Causally-motivated methods generally have good median discriminative performance, but poor bias mitigation properties. Furthermore, the wide performance ranges of causally-motivated approaches may be unacceptable for safety-critical/high-stakes domains. We note that the DR-learner may underperform in this setting due if the propensity score trimming introduces sufficient bias: recall that, although double-robustness only requires one correctly-specified model, the asymptotic properties may still depend on the asymptotics of each model (\emph{e.g.}, as shown in~\citep{wager2020stats}).

\paragraph{DCEM is empirically more robust to overlap violations.} Figure~\ref{fig:overlap_sens} shows that, empirically, as overlap violations increase, DCEM degrades more slowly than causally-motivated approaches in terms of both bias mitigation and discriminative performance. Furthermore, DCEM maintains similarly tight performance ranges across levels of overlap, while the performance ranges of causal approaches widens as overlap violations increase. At low overlap, causally-motivated approaches have similarly tight performance ranges as  DCEM.

\subsection{Sensitivity analysis of softmax temperature scaling}
\label{app:sensitivity_t}

 We can further tune the smoothness of $\hat{t}^{(i)}$ via the softmax temperature $\tau$ of the binary classifier for $\hat{t}^{(i)}$:
\begin{equation}
    \hat{t}^{(i)} :=  \frac{\exp(z_t / \tau)}{\exp(z_t / \tau) + \exp(z_{1-t} / \tau)}
\end{equation}
where $z_t$ is the logit outputted by $g_\zeta$ for each $t \in \{0, 1\}$.
Lower values of $\tau$ sharpen $\hat{t}^{(i)}$ towards $\{0, 1\}$, while larger values smooth $\hat{t}^{(i)}$ toward $\frac{1}{2}$. Note that $\tau = 1$ recovers the standard softmax function. Thus, adjusting $\tau$ allows us to control the smoothness of the $\tilde{y}^{(i)} = t^{(i)} y^{(i)}$ constraint.

\begin{table}[]
    \centering
        \caption{Sensitivity analysis of softmax temperature scaling parameter ($T$) with respect to DCEM AUC and ROC gap (min, max across $s_T$ in parentheses) for $q_y=0.5, k=1, q_t=2$. Maximum (minimum) median AUC (ROC gap) in bold.}\vspace{3mm}
    \begin{tabular}{c|c|c}
        $\tau$ & $\uparrow$ AUC & $\downarrow$ ROC gap \\
        \midrule
        0.01 & .778 [.737, .815] & .051 [.020, .104] \\
        0.1 & .791 [.762, .818] & \textbf{.025 [.014, .057]} \\
        1 (default) & .791 [.763, .820] & .031 [.019, .072] \\
        10 & \textbf{.800 [.730, .858]} & .051 [.021, .096] \\
        100 & .762 [.667, .835] & .071 [.032, .097]\\
        \bottomrule
    \end{tabular}

    \label{tab:ablation_softmax}
\end{table}
\begin{figure}[t]
    \centering
    \includegraphics[width=\linewidth]{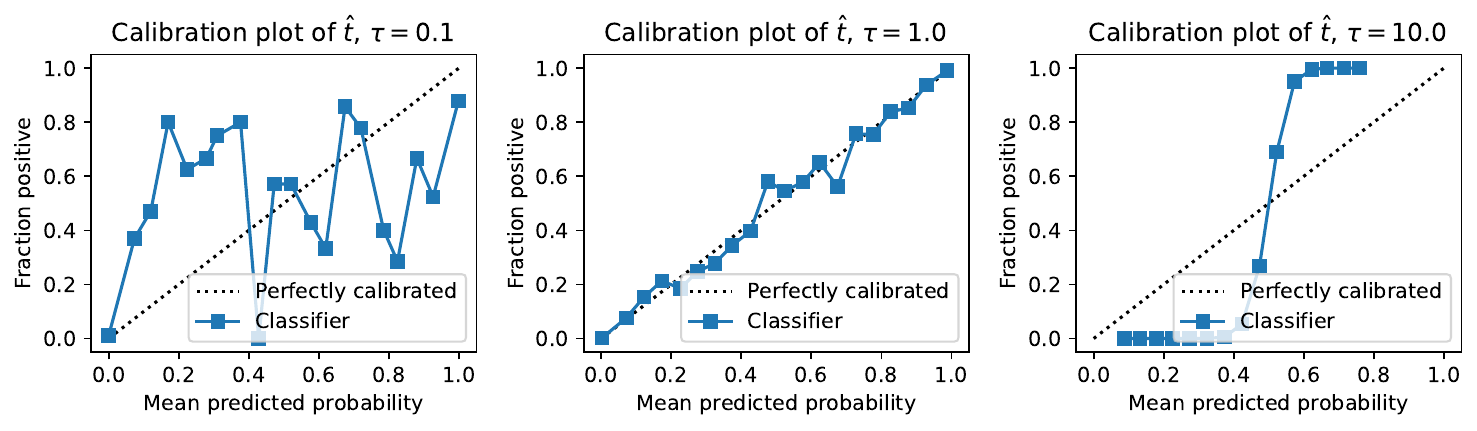}\vspace{-7mm}
    \caption{Calibration plot of $\hat{t}$ for (from left to right) $\tau \in \{0.1, 1, 10\}$. While $\hat{t}$ is well-calibrated for $\tau = 1$, changing $\tau$ in either direction ($<1 $ vs. $>1$) induces miscalibration error.}
    \label{fig:calibration}
\end{figure}

Table~\ref{tab:ablation_softmax} shows full results (median AUC and ROC gap, plus minima and maxima across $s_Y$) for DCEM across various values of temperature scaling parameter $\tau$. Empirically, our results suggest that temperature scaling does not significantly change the AUC, and may trade off with bias mitigation since $\hat{t}^{(i)}$ may no longer be calibrated.
Furthermore, even though median AUC improves in one case ($\tau=10$), the range of AUC is much larger (0.057 vs. 0.128), and $\tau=1$ still yields the maximum empirical worst-case AUC (0.763).

Values of $\tau$ away from 1 tend to yield larger ROC gaps. We find that $\hat{t}^{(i)}$ is well-calibrated for $\tau = 1$, but not so for values of $\tau$ (Figure~\ref{fig:calibration}). Since $\hat{t}^{(i)}$ is critical to counterbalancing disparate censorship, miscalibration error in $\hat{t}^{(i)}$ could result in larger ROC gaps by reducing the effectiveness/correctness of the causal regularization term. Thus, we opt to maintain $\tau = 1$.

\subsection{Sensitivity analysis of E-step initialization}
\label{app:sensitivity_init}

We compare random initialization to using a tested-only model as initialization (the final approach). Empirically, Table~\ref{tab:init} shows trivial changes to performance when using a model trained on labeled data to initialize the E-step. This suggests that DCEM is able to overcome poor initialization in the settings studied; \emph{i.e.}, the gains from tested-only initialization may be marginal, if nonzero.

\begin{table}[t]
    \centering
    \caption{Sensitivity analysis of E-step initialization with respect to DCEM AUC and ROC gap (min, max across $s_Y$ in parentheses) for $q_y = 0.5, k=1, q_t=2$. Maximum (minimum) median AUC (ROC gap) in bold.}\vspace{3mm}
     \begin{tabular}{c|c|c}
         Initialization scheme & AUC & ROC gap \\
         \midrule
         random & .787 [.768, .822] & \textbf{.031 [.011, .060]} \\
         tested-only & \textbf{.791 [.763, .820]} & \textbf{.031 [.019, .072]} \\
         \bottomrule
    \end{tabular}
    \label{tab:init}
\end{table}

\subsection{Tradeoffs between bias mitigation and discriminative performance: SELF}
\label{app:self_tradeoff}
\begin{figure}[t]
    \centering
    \includegraphics[width=\linewidth]{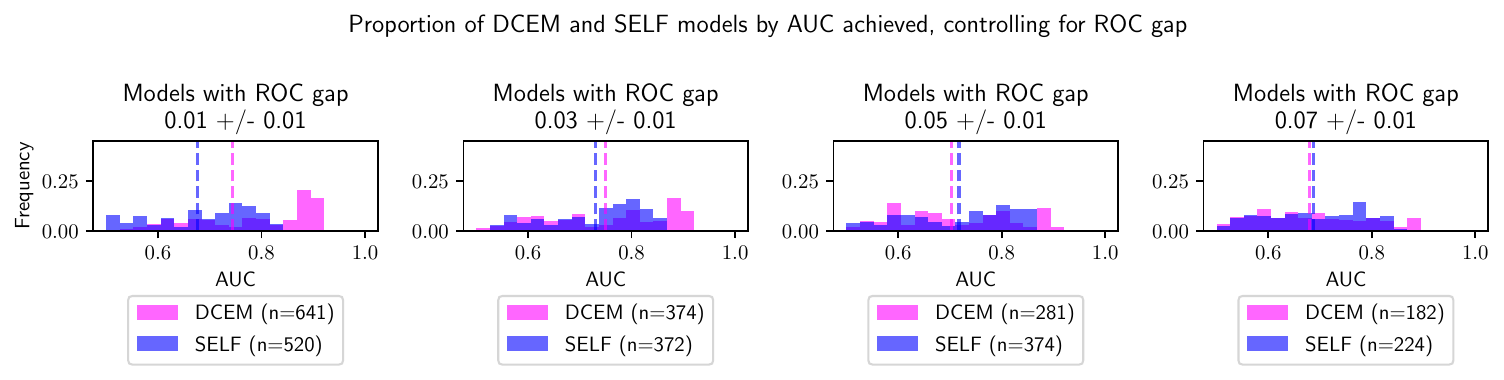}\vspace{-7mm}
    \caption{Relative frequencies of AUC for DCEM vs. SELF at similar ROC gaps, pooling models across all $k, q_y ,q_t$ tested. Dashed lines = mean AUC by model. DCEM improves AUC among models with similar ROC gaps when the ROC gap is below 0.04.}
    \label{fig:self-tradeoff}
\end{figure}

We compare instances of DCEM to SELF, controlling for ROC gap. We find that DCEM optimizes discriminative performance more effectively than SELF.
Fig.~\ref{fig:self-tradeoff} shows a histogram of AUC for SELF and DCEM models with similar ROC gaps across $q_t, q_y, k$ and $s_Y$, increasing in ROC gap to the right. For models with ROC gaps $< 0.04$ (Fig.~\ref{fig:self-tradeoff}, 1st and 2nd from left), DCEM improves AUC compared to instances of SELF with similar ROC gaps. At larger ROC gaps, DCEM and SELF obtain similar AUCs (Fig.~\ref{fig:self-tradeoff}, 1st and 2nd from right). Similarly to the comparison with tested-only  models, the results suggest that DCEM is not simply trading improved bias mitigation for performance, but is also able to optimize discriminative performance. Since SELF is a filtering approach that does not account for the causal structure of disparate censorship, its estimates of label bias are likely skewed. In contrast, DCEM explicitly uses the causal structure of disparate censorship to counterbalance label bias.

\subsection{Sepsis classification and robustness to shifts in labeling decisions}
\label{app:sepsis_robust}
\begin{figure*}[t]
    \centering
    \includegraphics[width=\linewidth]{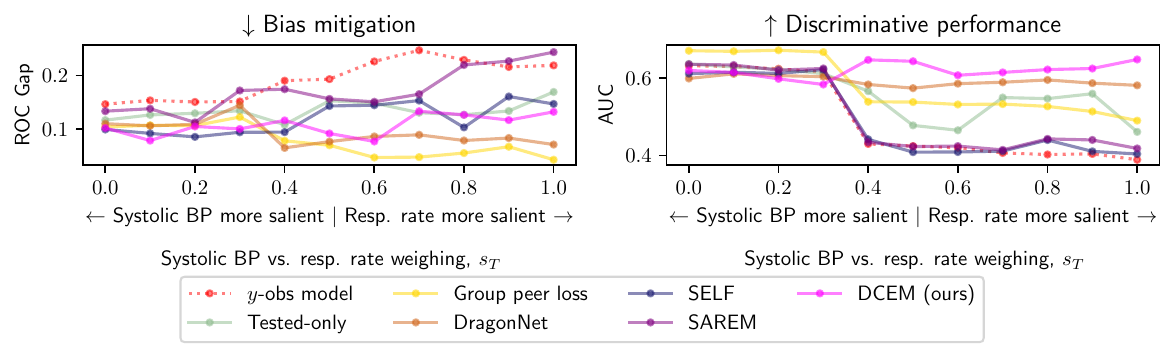}\vspace{-7mm}
    \caption{ROC gaps (left) and AUC (right) of selected models on sepsis classification task as weighting of systolic blood pressure (BP) and respiratory (resp.) rate ($s_T$) for testing changes (``0.0''/left: consider systolic BP only; ``1.0''/right: consider resp. rate only) at $q_t=1.5, k=4$. If a feature is ``more salient,'' it is weighted higher than the other in the testing decision function $s_T$.}
    \label{fig:boundaries}
\end{figure*}

Fig.~\ref{fig:boundaries} shows the performance of DCEM vs. models with bimodial behavior across different $s_T$, indexed by different feature weightings in $s_T$. 
Our results suggest that the baselines require specific $s_T$ to perform above random. The baselines catastrophically underperform (AUC below 0.5) otherwise. 
Trends are analogous for the ROC gap.

Specifically, baseline performance improves when one feature is more heavily weighted than the other in the labeling decision ($x$-axis near 0 or 1). However, when both features feature in labeling decisions ($x$-axis near 0.5), the baselines catastrophically fail, while DCEM performance stays high.
As seen in Fig.~\ref{fig:sepsis}, DCEM AUC and ROC gap also exhibit less variation across the different $s_T$.

Determining which $s_T$ is appropriate is a clinical problem that requires domain expertise, and we make no claims as to the clinical appropriateness of $s_T$. Thus, ML practitioners should not assume that their data will be representative of any particular decision-making pattern. 
DCEM is an alternative approach that is more robust than baselines to shifts in $s_T$, and thus warrants consideration when narrow assumptions about labeling biases are undesirable.

\section{Computing Infrastructure}

\paragraph{Hardware.} We parallelize experiments across 4 A6000 GPUs and 256 AMD CPU cores (4x AMD EPYC 7763 64-Core processors), though the memory requirements of each model are under 2GB of VRAM. 

\paragraph{Software.} All experiments are run on a distribution of Ubuntu 20.04.5 LTS (Focal Fossa) with Python 3.9.16 managed by conda 23.3.1. We use Pytorch 1.13.1 with CUDA 11.6 for all experiments~\cite{paszke2019pytorch}, with scikit-learn 1.2.2~\cite{scikit-learn}, scipy 1.10.1~\cite{2020SciPy-NMeth}, numpy 1.25.0~\cite{harris2020array} and pandas 1.5.3~\cite{reback2020pandas} for data processing/analysis. Matplotlib 3.7.1 was used to generate figures. Additionally, torch\_ema 0.3 was used in our implementation of SELF. For the simulation study, we use a modified version of the official disparate censorship repository at \url{https://github.com/MLD3/disparate_censorship}~\cite{chang2022disparate}, which is included with our code repository.

\section{Code}
Code will be released at the MLD3 Github repository at \url{https://github.com/MLD3/DCEM}. We redact the data-processing code for the sepsis task only where necessary to ensure compliance with the terms of use for MIMIC-III~\cite{johnson2016mimic}. 

\begin{figure}[t]
    \centering
    \includegraphics[width=\linewidth]{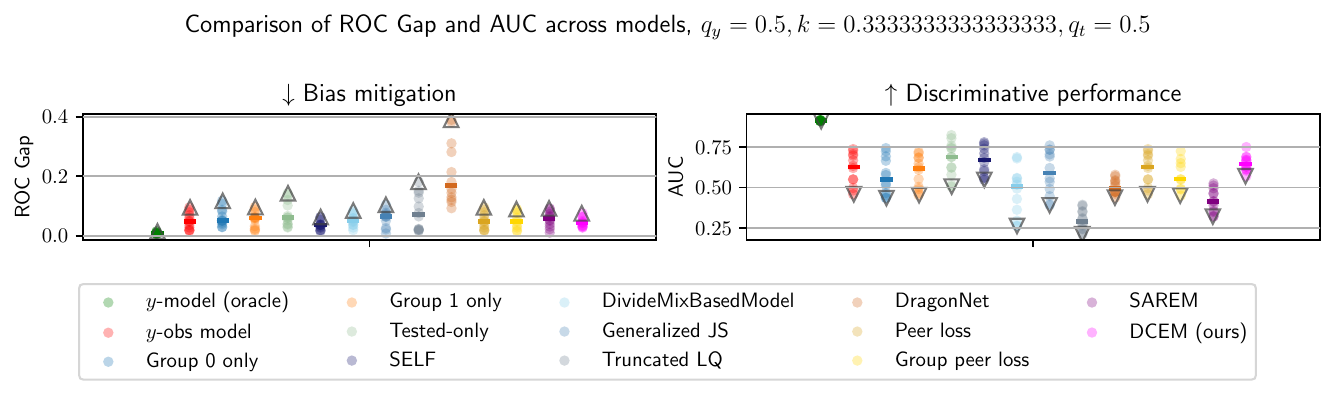}
    \vspace{-5mm}
    \caption{ROC gap (left) and AUC (right) of baselines on simulated data at $q_y = 1/2, k=1/3, q_t=1/2$. ``-'': median, ``$\bigtriangleup$'': worst-case ROC gap, ``$\bigtriangledown$'': worst-case AUC.}
    \label{fig:qy0.5_k0.3_qt0.5}
\end{figure}
\begin{figure}[t]
    \centering
    \includegraphics[width=\linewidth]{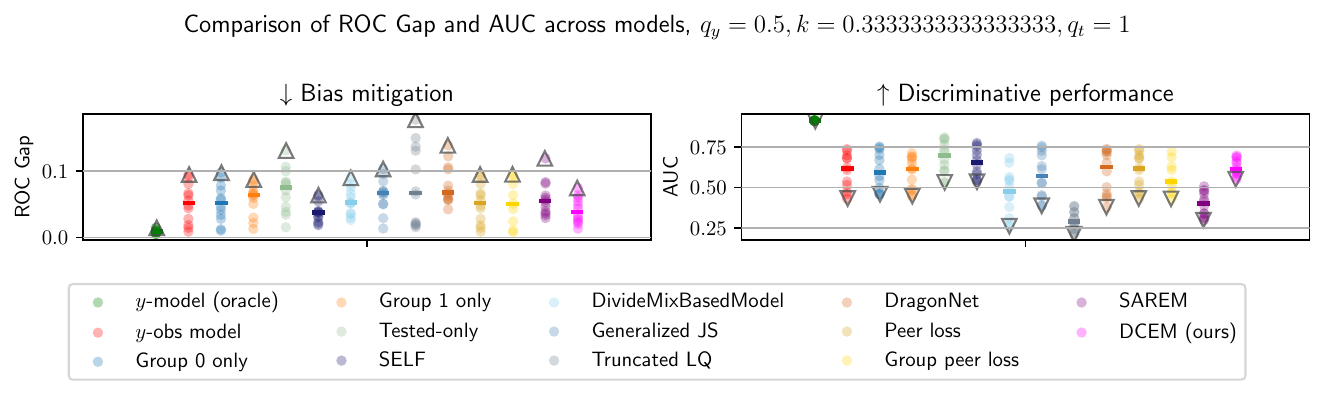}
    \vspace{-5mm}
    \caption{ROC gap (left) and AUC (right) of baselines on simulated data at $q_y = 1/2, k=1/3, q_t=1$. ``-'': median, ``$\bigtriangleup$'': worst-case ROC gap, ``$\bigtriangledown$'': worst-case AUC.}
    \label{fig:qy0.5_k0.3_qt1}
\end{figure}
\begin{figure}[t]
    \centering
    \includegraphics[width=\linewidth]{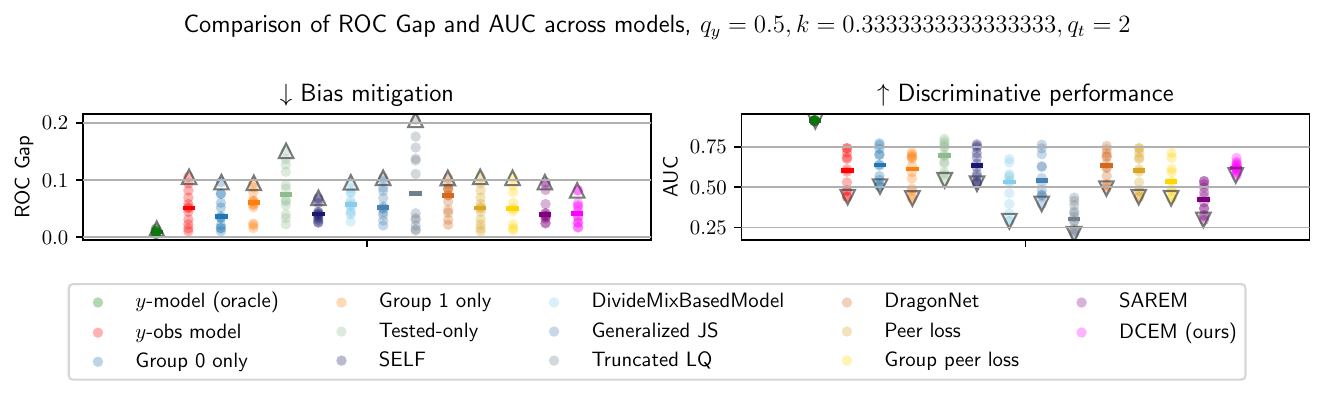}
    \vspace{-5mm}
    \caption{ROC gap (left) and AUC (right) of baselines on simulated data at $q_y = 1/2, k=1/3, q_t=2$. ``-'': median, ``$\bigtriangleup$'': worst-case ROC gap, ``$\bigtriangledown$'': worst-case AUC.}
    \label{fig:qy0.5_k0.3_qt2}
\end{figure}
\begin{figure}[t]
    \centering
    \includegraphics[width=\linewidth]{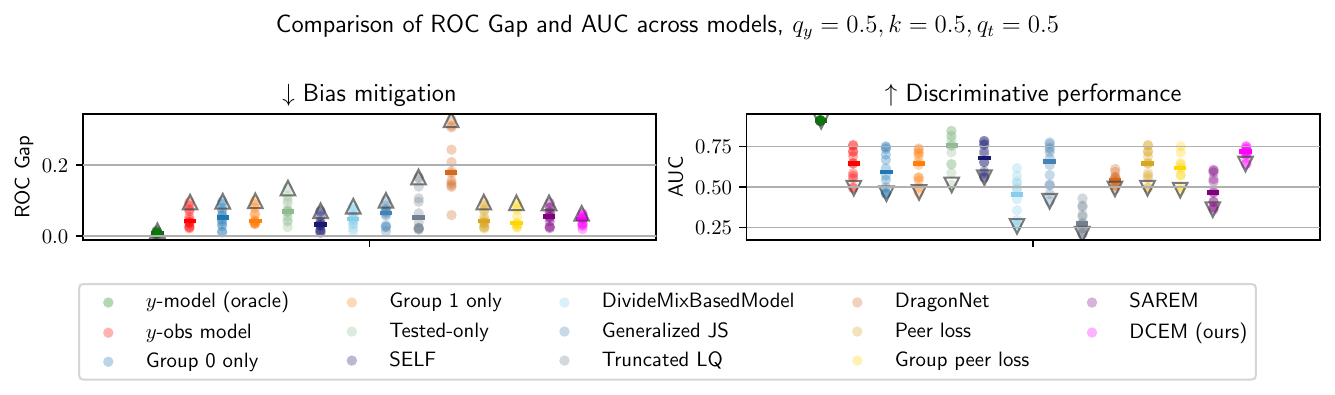}
    \vspace{-5mm}
    \caption{ROC gap (left) and AUC (right) of baselines on simulated data at $q_y = 1/2, k=1/2, q_t=1/2$. ``-'': median, ``$\bigtriangleup$'': worst-case ROC gap, ``$\bigtriangledown$'': worst-case AUC.}
    \label{fig:qy0.5_k0.5_qt0.5}
\end{figure}
\begin{figure}[t]
    \centering
    \includegraphics[width=\linewidth]{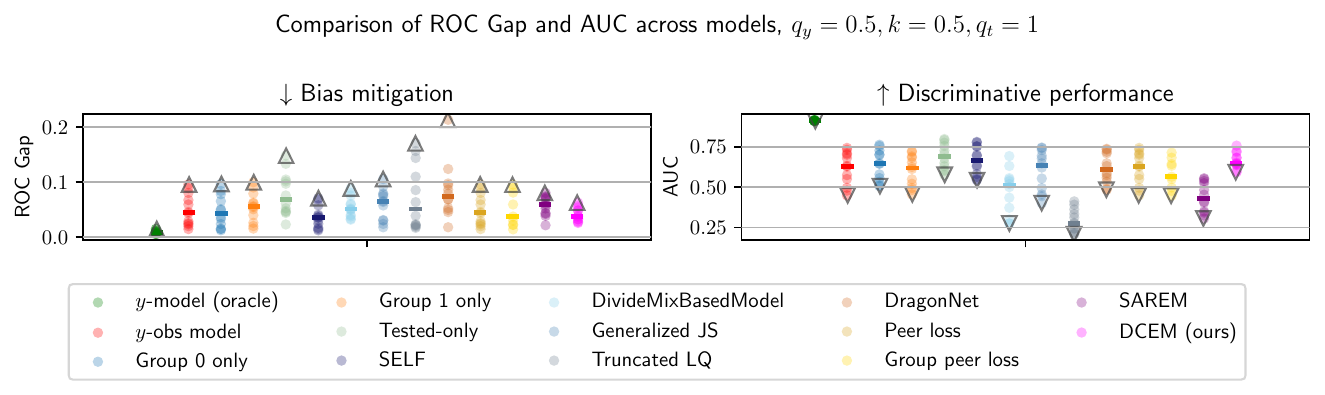}
    \vspace{-5mm}
    \caption{ROC gap (left) and AUC (right) of baselines on simulated data at $q_y = 1/2, k=1/2, q_t=1$. ``-'': median, ``$\bigtriangleup$'': worst-case ROC gap, ``$\bigtriangledown$'': worst-case AUC.}
    \label{fig:qy0.5_k0.5_qt1}
\end{figure}
\begin{figure}[t]
    \centering
    \includegraphics[width=\linewidth]{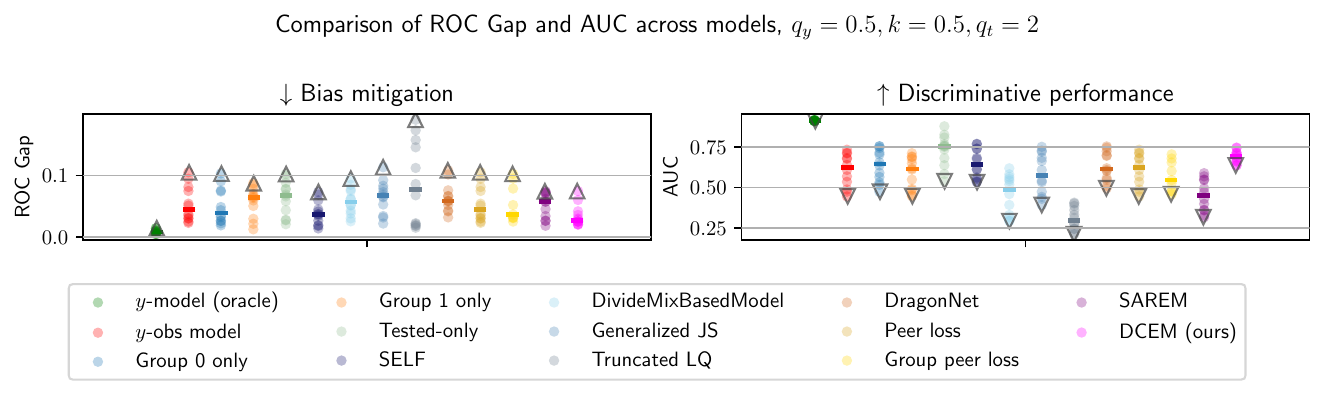}
    \vspace{-5mm}
    \caption{ROC gap (left) and AUC (right) of baselines on simulated data at $q_y = 1/2, k=1/2, q_t=2$. ``-'': median, ``$\bigtriangleup$'': worst-case ROC gap, ``$\bigtriangledown$'': worst-case AUC.}
    \label{fig:qy0.5_k0.5_qt2}
\end{figure}
\begin{figure}[t]
    \centering
    \includegraphics[width=\linewidth]{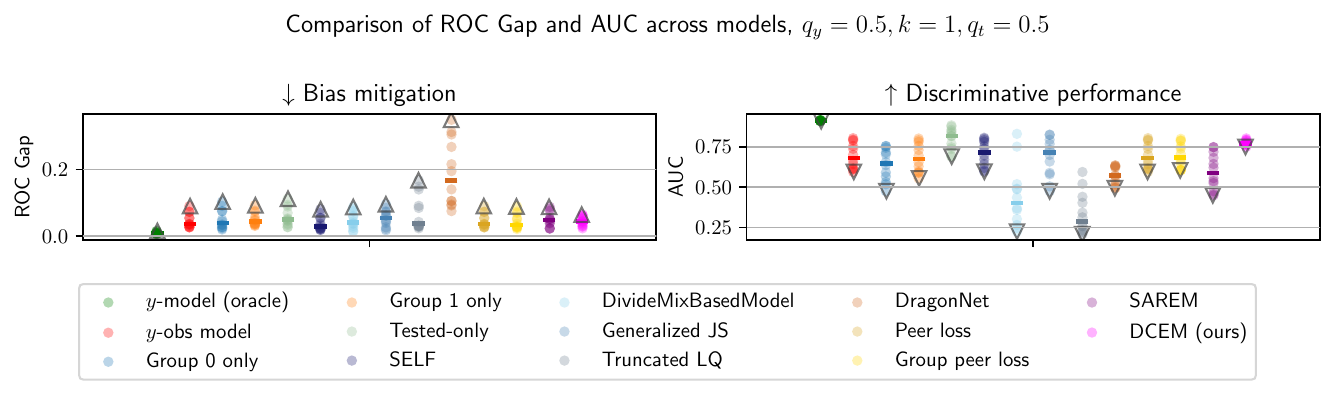}
    \vspace{-5mm}
    \caption{ROC gap (left) and AUC (right) of baselines on simulated data at $q_y = 1/2, k=1, q_t=1/2$. ``-'': median, ``$\bigtriangleup$'': worst-case ROC gap, ``$\bigtriangledown$'': worst-case AUC.}
    \label{fig:qy0.5_k1_qt0.5}
\end{figure}
\begin{figure}[t]
    \centering
    \includegraphics[width=\linewidth]{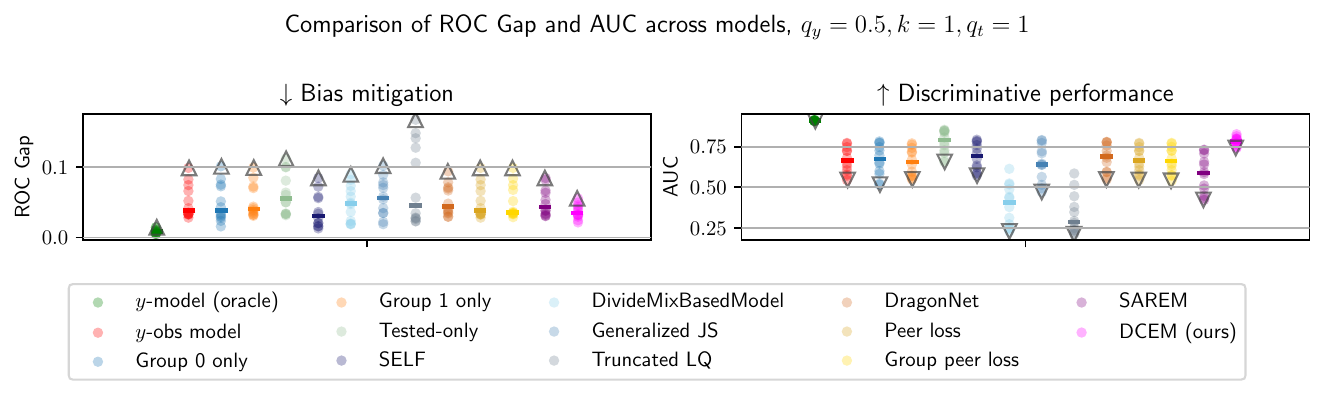}
    \vspace{-5mm}
    \caption{ROC gap (left) and AUC (right) of baselines on simulated data at $q_y = 1/2, k=1, q_t=1$. ``-'': median, ``$\bigtriangleup$'': worst-case ROC gap, ``$\bigtriangledown$'': worst-case AUC.}
    \label{fig:qy0.5_k1_qt1}
\end{figure}
\begin{figure}[t]
    \centering
    \includegraphics[width=\linewidth]{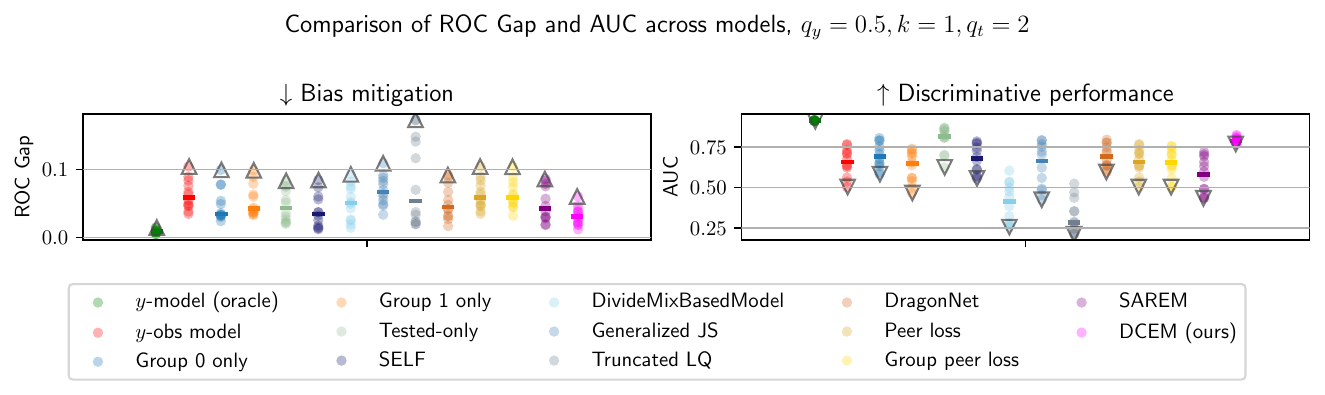}
    \vspace{-5mm}
    \caption{ROC gap (left) and AUC (right) of baselines on simulated data at $q_y = 1/2, k=1, q_t=2$. ``-'': median, ``$\bigtriangleup$'': worst-case ROC gap, ``$\bigtriangledown$'': worst-case AUC.}
    \label{fig:qy0.5_k1_qt2}
\end{figure}
\begin{figure}[t]
    \centering
    \includegraphics[width=\linewidth]{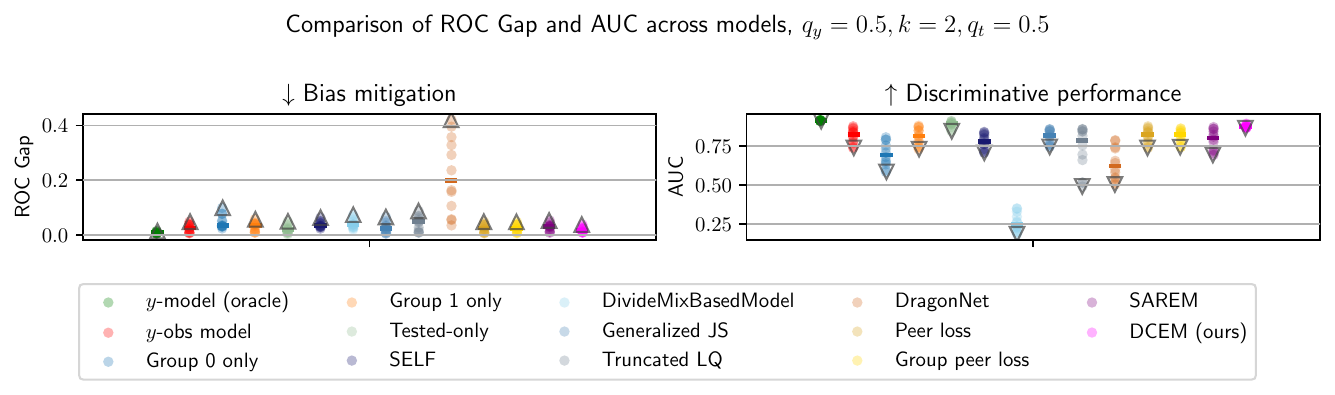}
    \vspace{-5mm}
    \caption{ROC gap (left) and AUC (right) of baselines on simulated data at $q_y = 1/2, k=2, q_t=1/2$. ``-'': median, ``$\bigtriangleup$'': worst-case ROC gap, ``$\bigtriangledown$'': worst-case AUC.}
    \label{fig:qy0.5_k2_qt0.5}
\end{figure}
\begin{figure}[t]
    \centering
    \includegraphics[width=\linewidth]{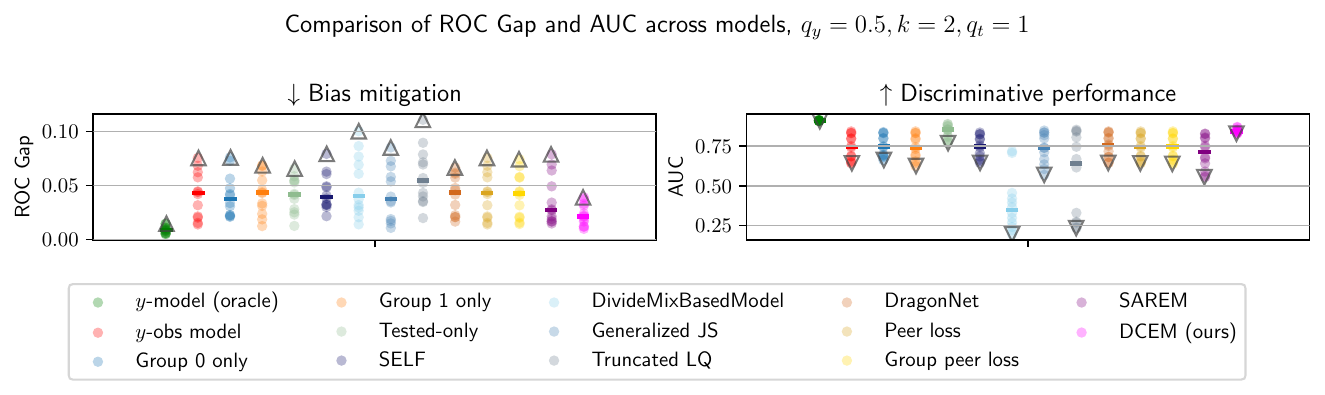}
    \vspace{-5mm}
    \caption{ROC gap (left) and AUC (right) of baselines on simulated data at $q_y = 1/2, k=2, q_t=1$. ``-'': median, ``$\bigtriangleup$'': worst-case ROC gap, ``$\bigtriangledown$'': worst-case AUC.}
    \label{fig:qy0.5_k2_qt1}
\end{figure}
\begin{figure}[t]
    \centering
    \includegraphics[width=\linewidth]{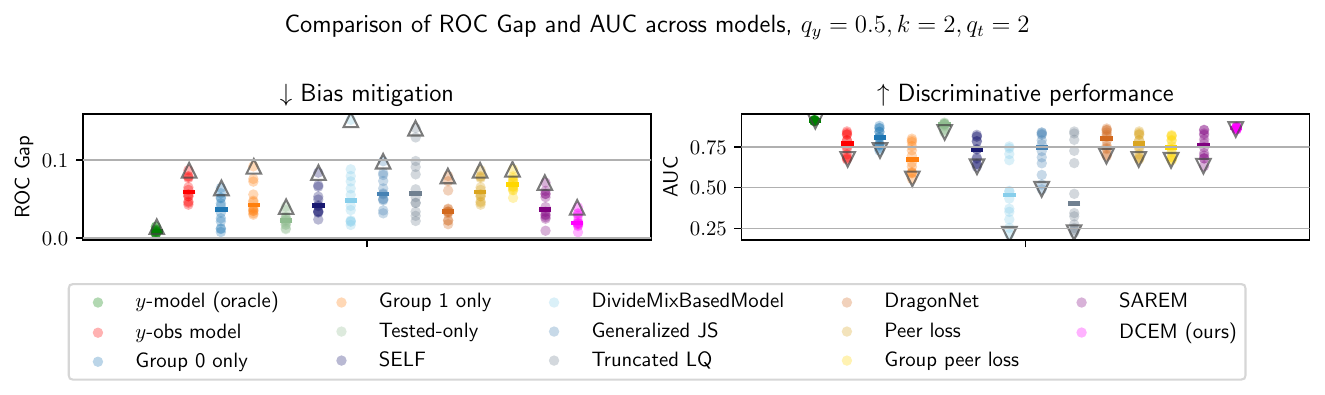}
    \vspace{-5mm}
    \caption{ROC gap (left) and AUC (right) of baselines on simulated data at $q_y = 1/2, k=2, q_t=2$. ``-'': median, ``$\bigtriangleup$'': worst-case ROC gap, ``$\bigtriangledown$'': worst-case AUC.}
    \label{fig:qy0.5_k2_qt2}
\end{figure}
\begin{figure}[t]
    \centering
    \includegraphics[width=\linewidth]{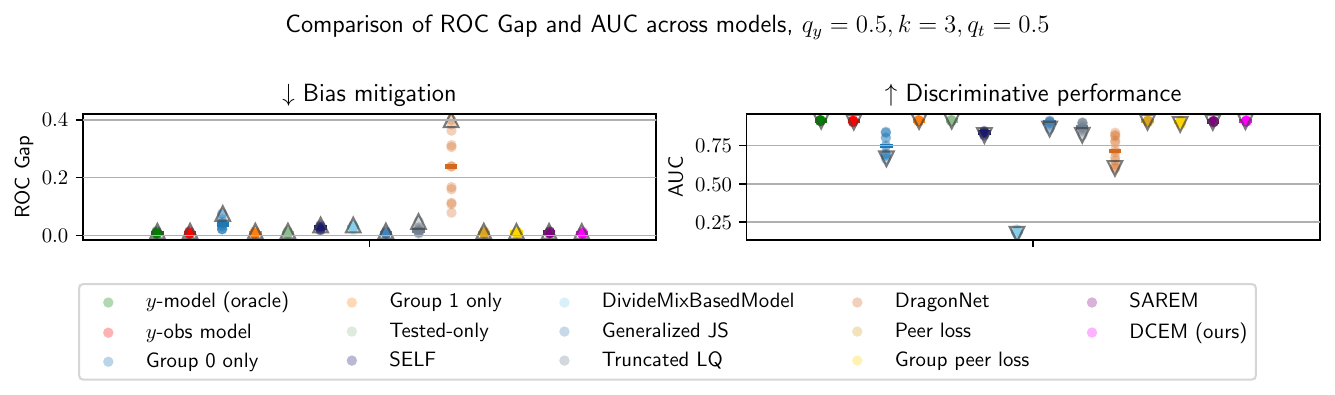}
    \vspace{-5mm}
    \caption{ROC gap (left) and AUC (right) of baselines on simulated data at $q_y = 1/2, k=3, q_t=1/2$. ``-'': median, ``$\bigtriangleup$'': worst-case ROC gap, ``$\bigtriangledown$'': worst-case AUC.}
    \label{fig:qy0.5_k3_qt0.5}
\end{figure}
\begin{figure}[t]
    \centering
    \includegraphics[width=\linewidth]{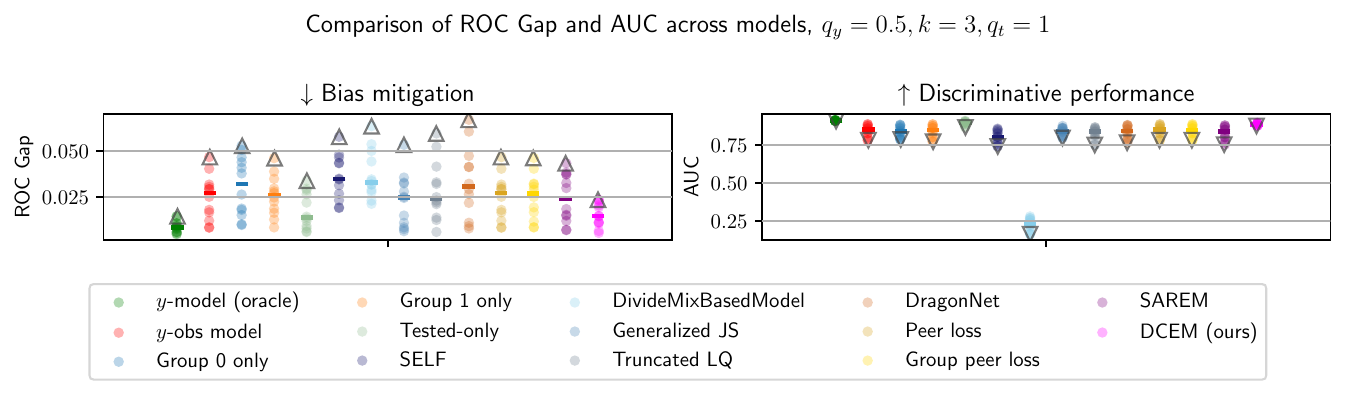}
    \vspace{-5mm}
    \caption{ROC gap (left) and AUC (right) of baselines on simulated data at $q_y = 1/2, k=3, q_t=1$. ``-'': median, ``$\bigtriangleup$'': worst-case ROC gap, ``$\bigtriangledown$'': worst-case AUC.}
    \label{fig:qy0.5_k3_qt1}
\end{figure}
\begin{figure}[t]
    \centering
    \includegraphics[width=\linewidth]{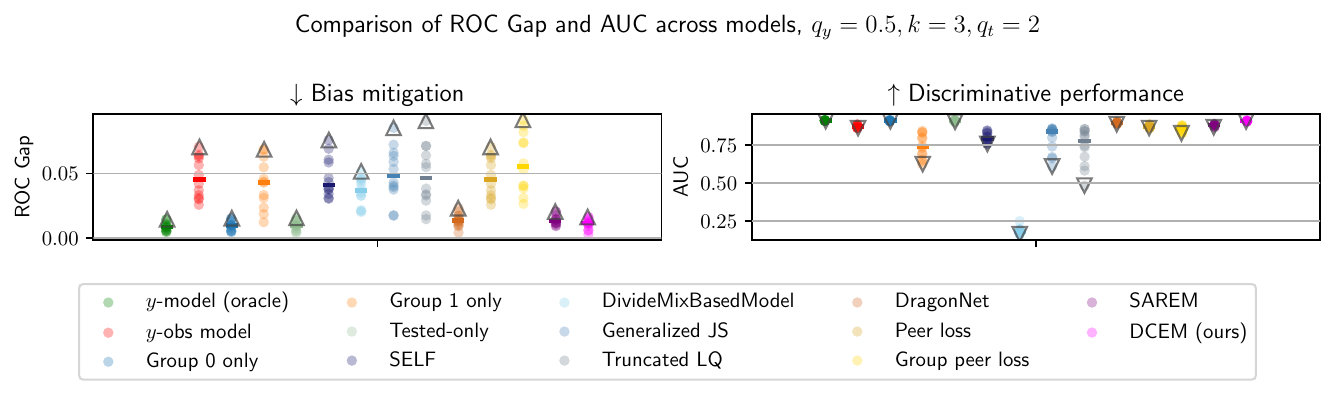}
    \vspace{-5mm}
    \caption{ROC gap (left) and AUC (right) of baselines on simulated data at $q_y = 1/2, k=3, q_t=2$. ``-'': median, ``$\bigtriangleup$'': worst-case ROC gap, ``$\bigtriangledown$'': worst-case AUC.}
    \label{fig:qy0.5_k3_qt2}
\end{figure}

\begin{figure}[t]
    \centering
    \includegraphics[width=\linewidth]{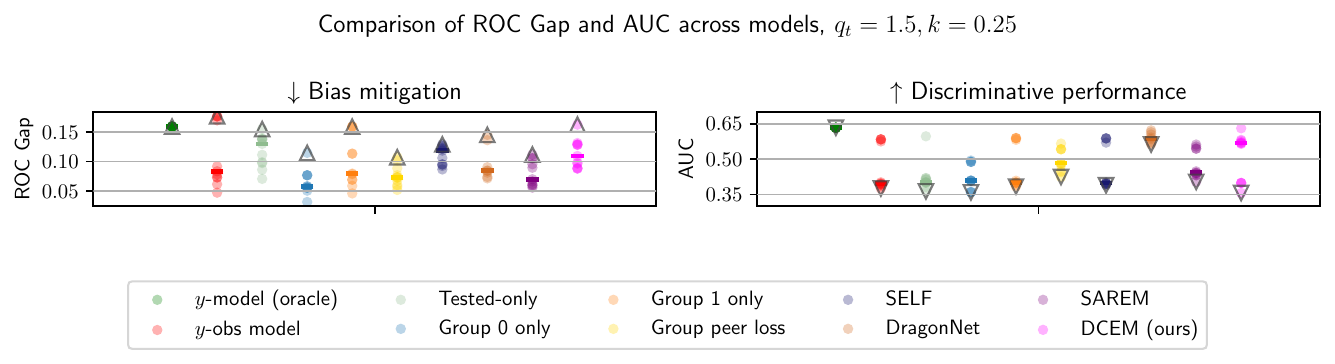}\vspace{-7mm}
    \caption{ROC gap (left) and AUC (right) of baselines on sepsis classification at $k=1/4, q_t=1.5$. ``-'': median, ``$\bigtriangleup$'': worst-case ROC gap, ``$\bigtriangledown$'': worst-case AUC.}\vspace{-3mm}
    \label{fig:sepsis_k0.25}
\end{figure}
\begin{figure}[t]
    \centering
    \includegraphics[width=\linewidth]{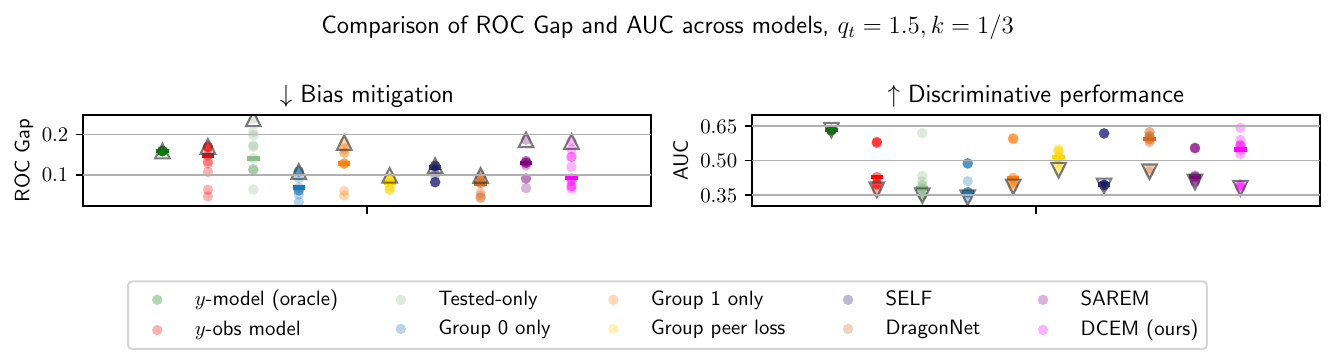}\vspace{-7mm}
    \caption{ROC gap (left) and AUC (right) of baselines on sepsis classification at $k=1/3, q_t=1.5$. ``-'': median, ``$\bigtriangleup$'': worst-case ROC gap, ``$\bigtriangledown$'': worst-case AUC.}\vspace{-3mm}
    \label{fig:sepsis_k0.3}
\end{figure}
\begin{figure}[t]
    \centering
    \includegraphics[width=\linewidth]{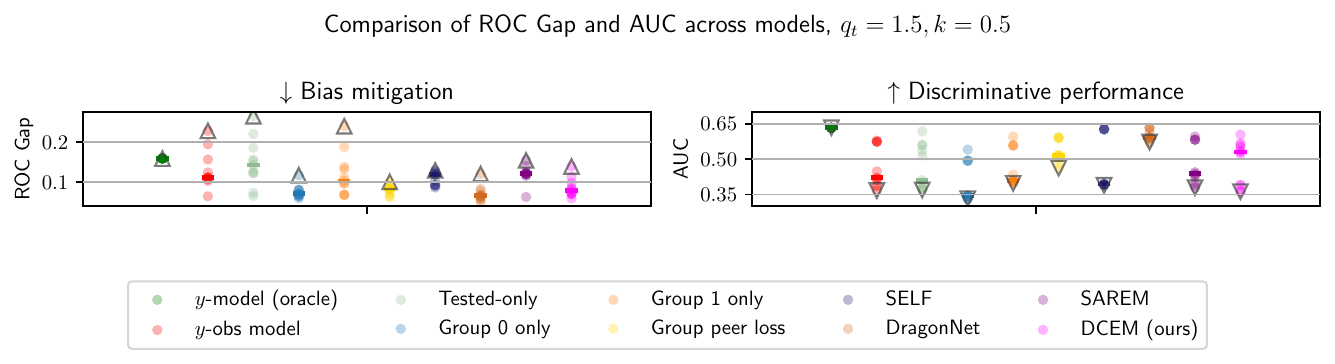}\vspace{-7mm}
    \caption{ROC gap (left) and AUC (right) of baselines on sepsis classification at $k=1/2, q_t=1.5$. ``-'': median, ``$\bigtriangleup$'': worst-case ROC gap, ``$\bigtriangledown$'': worst-case AUC.}\vspace{-3mm}
    \label{fig:sepsis_k0.5}
\end{figure}
\begin{figure}[t]
    \centering
    \includegraphics[width=\linewidth]{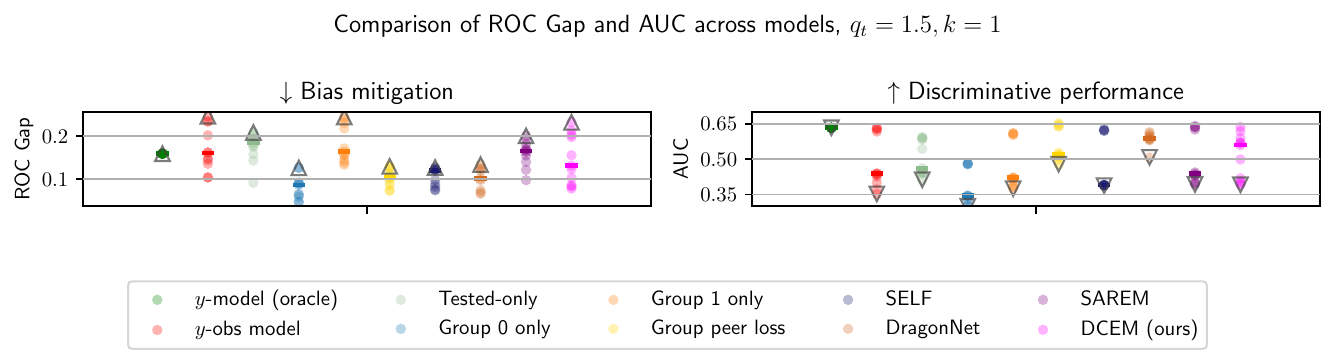}\vspace{-7mm}
    \caption{ROC gap (left) and AUC (right) of baselines on sepsis classification at $k=1, q_t=1.5$. ``-'': median, ``$\bigtriangleup$'': worst-case ROC gap, ``$\bigtriangledown$'': worst-case AUC.}\vspace{-3mm}
    \label{fig:sepsis_k1}
\end{figure}
\begin{figure}[t]
    \centering
    \includegraphics[width=\linewidth]{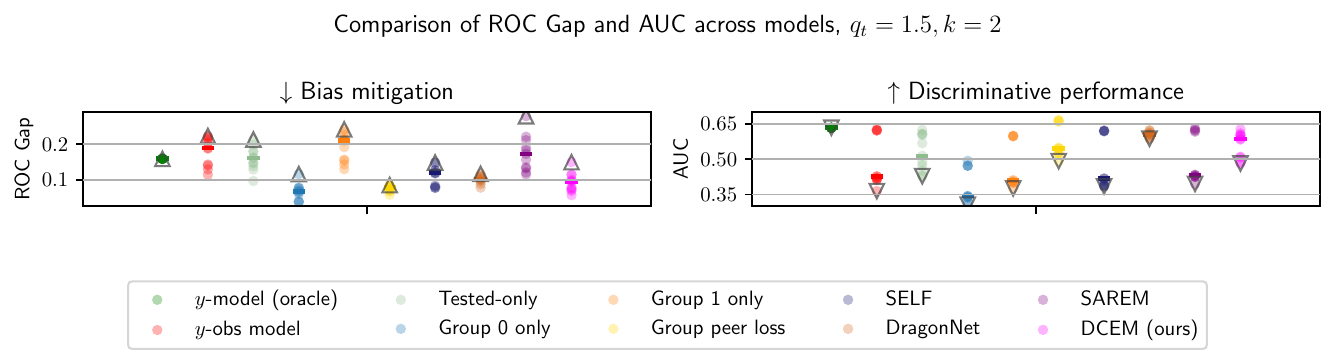}\vspace{-7mm}
    \caption{ROC gap (left) and AUC (right) of baselines on sepsis classification at $k=2, q_t=1.5$. ``-'': median, ``$\bigtriangleup$'': worst-case ROC gap, ``$\bigtriangledown$'': worst-case AUC.}\vspace{-3mm}
    \label{fig:sepsis_k2}
\end{figure}
\begin{figure}[t]
    \centering
    \includegraphics[width=\linewidth]{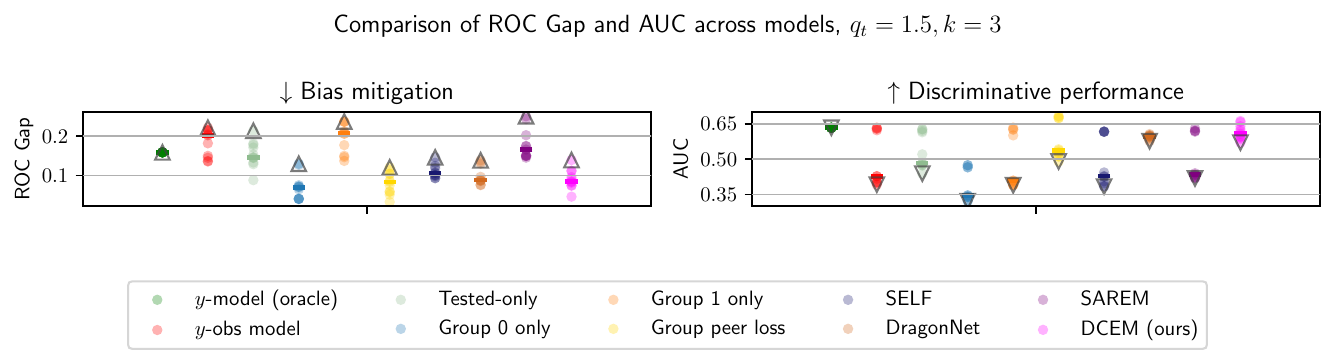}\vspace{-7mm}
    \caption{ROC gap (left) and AUC (right) of baselines on sepsis classification at $k=3, q_t=1.5$. ``-'': median, ``$\bigtriangleup$'': worst-case ROC gap, ``$\bigtriangledown$'': worst-case AUC.}\vspace{-3mm}
    \label{fig:sepsis_k3}
\end{figure}
\begin{figure}[t]
    \centering
    \includegraphics[width=\linewidth]{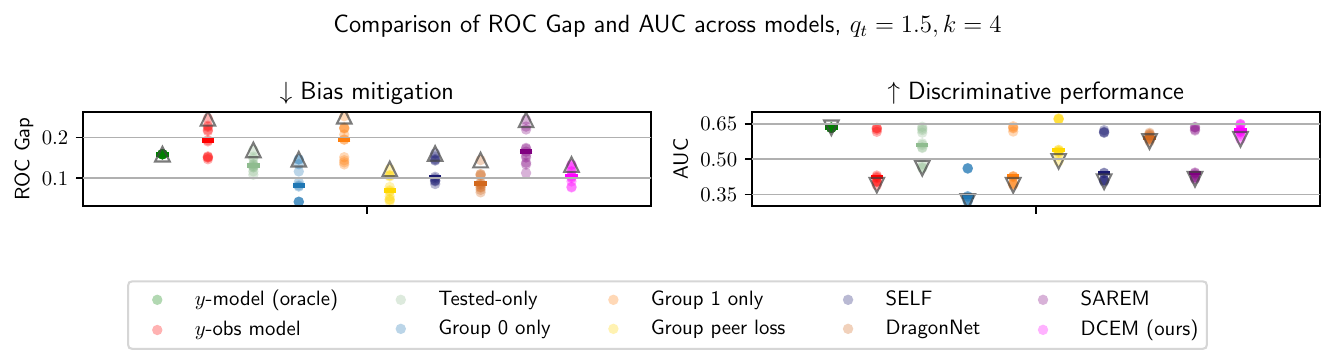}\vspace{-7mm}
    \caption{ROC gap (left) and AUC (right) of baselines on sepsis classification at $k=4, q_t=1.5$. ``-'': median, ``$\bigtriangleup$'': worst-case ROC gap, ``$\bigtriangledown$'': worst-case AUC.}\vspace{-3mm}
    \label{fig:sepsis_k4}
\end{figure}
\begin{figure}[t]
    \centering
    \includegraphics[width=\linewidth]{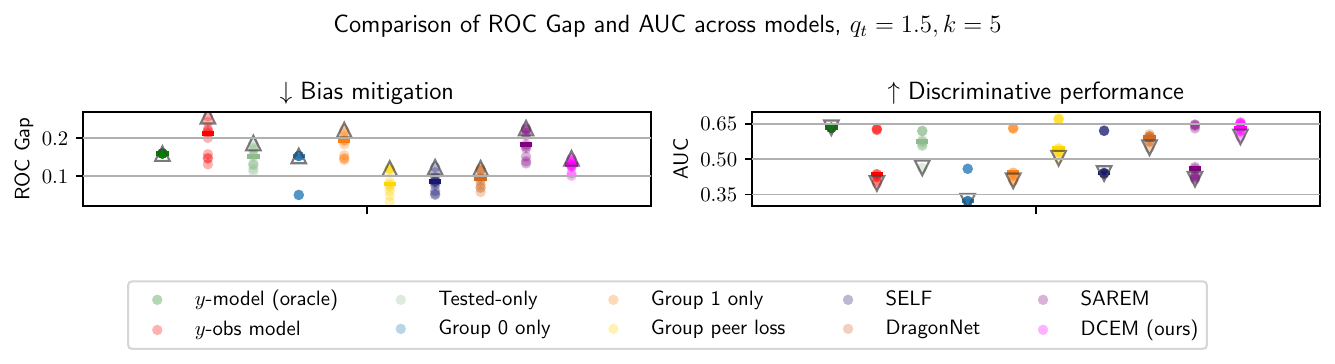}\vspace{-7mm}
    \caption{ROC gap (left) and AUC (right) of baselines on sepsis classification at $k=5, q_t=1.5$. ``-'': median, ``$\bigtriangleup$'': worst-case ROC gap, ``$\bigtriangledown$'': worst-case AUC.}
    \label{fig:sepsis_k5}
\end{figure}

\end{document}